\documentclass[11pt]{article}

\usepackage{amsmath,amssymb}

\usepackage{times}
\usepackage{bm}
\usepackage{natbib}
\usepackage{algorithm}
\usepackage{amssymb}
\usepackage{mathrsfs}
\usepackage{graphicx}
\usepackage[flushleft]{threeparttable}
\usepackage{enumitem} 

\usepackage{microtype}
\usepackage{subfigure}
\usepackage{mathtools}
\usepackage[disable,textsize=tiny]{todonotes}

\usepackage{titletoc}

\usepackage{nexais}

\usepackage[colorlinks,
linkcolor=metablue,
anchorcolor=metablue,
citecolor=metablue,
urlcolor=metablue
]{hyperref}

\usepackage{txfonts}

\usepackage{geometry}

\providecommand{\pmifi}{\ensuremath{\mathrm{PMI}(\mathrm{FI};\mathrm{gap}\mid\mathrm{II})}}
\providecommand{\pminoise}{\ensuremath{\mathrm{PMI}(\mathrm{noise};\mathrm{gap}\mid\mathrm{II})}}
\providecommand{\dRtwo}{\ensuremath{\Delta \mathrm{R}^2}}
\providecommand{\drfi}{\ensuremath{\dRtwo(\mathrm{FI})}}
\providecommand{\drnoise}{\ensuremath{\dRtwo(\mathrm{noise})}}

\providecommand{\independent}{\perp\kern-5pt \perp}

\begin{document}

\title{Feature Bagging Provides Stability}

\author[1,3]{Yuheng Ma}
\author[2,3]{Qiang Sun}

\affiliation[1]{School of Statistics, East China Normal University}
\affiliation[2]{University of Toronto}
\affiliation[3]{MBZUAI}

\abstract{
We study feature bagging through the lens of algorithmic stability. Feature bagging is an ensemble strategy that aggregates base learners trained on randomly subsampled feature subsets, possibly in a data-dependent manner. We introduce feature instability (FI), the feature-axis analogue of instance instability (II), which measures sensitivity to removing a single feature. Smaller values of II or FI correspond to stronger stability, and our experiments show that FI captures generalization-relevant information complementary to II. Within this framework, we analyze feature bagging in both a parametric linear model and a model-free setting inspired by recursive feature subsampling in random forests. In both settings, we establish formal guarantees showing that feature bagging improves the relevant stability relative to its non-bagged counterpart, with larger improvements under more aggressive subsampling. We further show that a modest number of bagging rounds is sufficient to approach the infinite-bagging stability level.
}

\metadata[Keywords]{stability, bagging, feature instability, random forests, random forward selection.}

\date{\today}

\maketitle


\startcontents[main]
\section*{\contentsname}
\printcontents[main]{}{1}{}

\section{Introduction}\label{sec:introduction}

Stability is a fundamental prerequisite for trustworthy machine learning \citep{murdoch2019definitions, yu2020veridical, xing2021algorithmic, zhou2023toward}, and is closely connected to generalization \citep{bousquet2002stability}, uncertainty quantification \citep{wang2023stability}, and model selection \citep{meinshausen2010stability, nogueira2018stability}.
For a predictor $f$ trained on $n$ samples in dimension $d$, an informal way to capture instance-wise instability is to measure how much its prediction changes when one training instance is removed:
\begin{align}\label{equ:informal-instance-stability}
\frac{1}{n} \sum_{i = 1}^n \left( f(\bx)- f^{-i, : }(\bx)\right)^2
\end{align}
where $f^{-i,:}$ denotes the predictor obtained after removing the $i$-th training instance.
Smaller values indicate that the algorithm is more stable to single-instance perturbations. We use this quantity only as motivation; Definition~\ref{def:stability-instance} formalizes \(\phi\)-instance stability and names the associated instance-instability measure. Regularization \citep{bousquet2002stability} and bagging \citep{soloff2024bagging} are two standard approaches for improving instance stability.


Feature bagging has been recognized as a potential driver of ensemble success \citep{breiman2001random, sutton2006reducing, lejeune2020implicit}, but the mechanism behind its empirical effectiveness remains debated.
One promising explanation is that feature bagging acts as a form of regularization \citep{mentch2020randomization, curth2024random}.
This viewpoint makes stability a natural lens for studying feature bagging.
Instance stability is a natural starting point because it measures sensitivity to changes in the training instances. However, it does not directly capture sensitivity to changes in the feature set, precisely the axis on which feature bagging operates; see Section~\ref{sec:exp_for_feature_stability_meaning} for empirical evidence that feature-axis perturbations carry generalization-relevant information that is complementary to instance stability. These considerations motivate our central question:
\begin{quote}
\centering
\textit{Does feature bagging improve stability, and if so, how?}
\end{quote}
Answering this question requires a feature-axis analogue of classical instance  instability.
We therefore introduce \emph{feature instability} (FI), the leave-one-feature-out analogue of \eqref{equ:informal-instance-stability}:
\begin{align}\label{equ:informal-feature-stability}
    \frac{1}{d} \sum_{j = 1}^d \left( f(\bx)- f^{:, -j }(\bx)\right)^2,
\end{align}
where $f^{:,-j}$ denotes the predictor obtained after removing the $j$-th feature.
This informal quantity is the feature-axis analogue of instance instability; the formal definitions of feature-instability and $\phi$-stability are given later in Definition~\ref{def:stability-feature}.

We answer this question through the following contributions:
\begin{enumerate}
\item We introduce FI as the feature-axis analogue of classical II.
This gives a two-axis view of stability tailored to feature bagging: II measures sensitivity to removing training instances, whereas FI measures sensitivity to removing features.
We further show empirically that FI carries generalization-relevant information that is complementary to II, motivating FI as a distinct object of study.

\item In linear regression, we give an exact asymptotic characterization of how bagging affects II and FI.
The analysis explains how ensemble averaging improves stability and how the stability gains depend on the sampling of both instances and features.
It also reveals that feature-level perturbations have their own scaling behavior, governed by the ambient dimension, which is invisible from II alone.

\item Beyond parametric models, we develop a model-free theory for FI under feature bagging, with particular emphasis on recursive feature subsampling.
This theory applies to algorithms such as random forward selection and random forests, where features are resampled repeatedly and data-dependently.
It shows that more aggressive feature subsampling yields stronger feature-stability guarantees, identifies regimes in which feature bagging is provably more stable than its non-bagged counterpart, and establishes that a number of bagging rounds proportional to the dimension is enough to approach the infinite-bagging stability level.
\end{enumerate}

The rest of the paper proceeds as follows. Section~\ref{sec:basic-concept} formalizes II, FI, and bagging. Section~\ref{sec:OLS} analyzes the linear-model setting, while Section~\ref{sec:model-free-feature-stability} develops the model-free theory for recursive feature subsampling. Proofs, derivations, and additional experiments are deferred to the appendix.


\subsection{Related Works}
\label{sec:related-work}

\paragraph{Stability.}
Algorithmic stability formalizes the idea that a learned predictor should not
change too much after a small change to the training data.  Classical
formulations usually modify the instance axis, either by removing one sample or
by replacing it with an independent copy
\citep{bousquet2002stability, elisseeff2005stability, kutin2002almost,
mukherjee2006learning, liu2017algorithmic}.  These notions are central to
stability-based generalization analysis
\citep{bousquet2002stability, feldman2018generalization, feldman2019high,
bousquet2020sharper}.  A closely related line studies averaging stability,
which averages the sensitivity over all single-sample changes
\citep{shalev2010learnability, lei2020fine, soloff2024bagging}.  Recent work has
used this viewpoint to analyze why bagging can stabilize learning algorithms
\citep{adrian2024stabilizing, soloff2024building, soloff2024bagging,
soloff2024stability, liang2025assumption}.

Our definitions follow the removal-based convention.  For instance stability,
we compare the predictor trained on \(\cD\) with the predictor trained on
\(\cD^{-i,:}\).  Replacement-based variants can be related to the removal-based
definition by a triangle-inequality argument
\citep{bousquet2002stability, soloff2024bagging}.  The same issue appears on
the feature axis.  We use feature removal as the primary convention because
feature replacement requires specifying how the replacement coordinate is
generated, for example unconditionally or conditionally on
\((Y_i,\bx_i^{-j})\).

\paragraph{Bagging.}
Instance bagging is widely used across tree-based models
\citep{breiman1996bagging, geurts2006extremely}, nearest-neighbor methods
\citep{biau2010rate, cai2025bagged}, and neural-network ensembles
\citep{hansen1990neural, perrone1995networks, lakshminarayanan2017simple,
zaidi2021neural}.  Feature bagging is most prominent in tree-based methods,
especially random forests, where each split is chosen after subsampling the
candidate features \citep{breiman1996bagging, breiman2001random,
geurts2006extremely}.  Recent high-dimensional linear-model analyses provide a
more controlled view of feature subsampling and ensembling.  Sketched ridgeless
regression studies how downsampling changes interpolation behavior
\citep{chen2023sketched}; related reference-panel estimators analyze the exact
risk of regularization induced by an auxiliary feature panel
\citep{su2024exact}; and bagged linear interpolators show how averaging
sketched estimators controls interpolation-driven variance
\citep{wu2025ensemble}.  These works connect feature subsampling to prediction
risk.  We build on this literature by asking the corresponding stability
question: whether feature bagging also stabilizes predictions along the feature
axis.

\section{Basic Concepts}\label{sec:basic-concept}

\begin{figure}[!t]
\vskip -0.05in
\centering
\subfigure[Instance subsampling.]{
\begin{minipage}{0.17\linewidth}
\centering
\includegraphics[width=\textwidth]{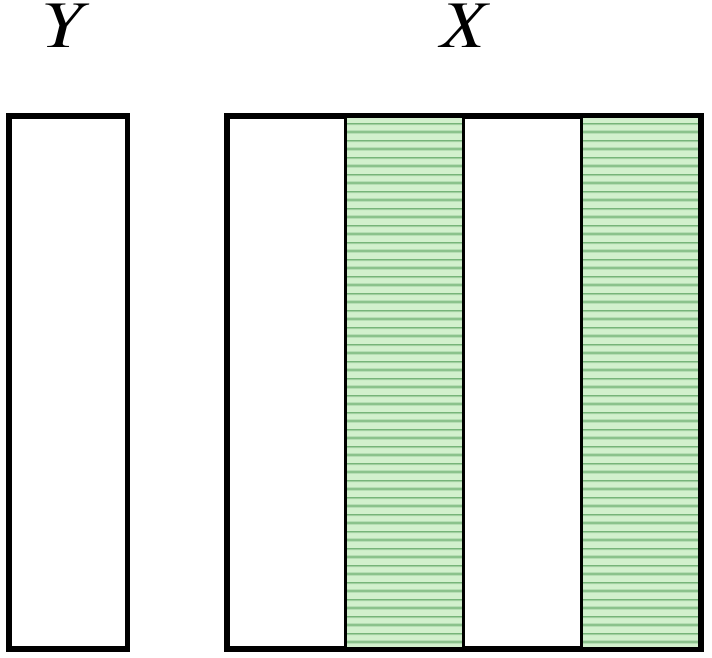}
\end{minipage}
\label{fig:illustration_subsample_1}
}
\subfigure[Feature subsampling.]{
\begin{minipage}{0.17\linewidth}
\centering
\includegraphics[width=\textwidth]{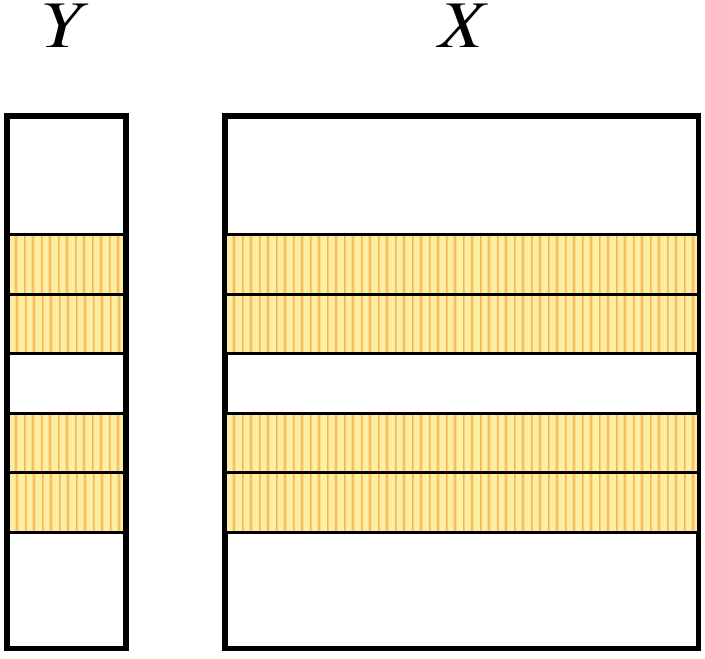}
\end{minipage}
\label{fig:illustration_subsample_2}
}
\subfigure[Instance and feature subsampling.]{
\begin{minipage}{0.17\linewidth}
\centering
\includegraphics[width=\textwidth]{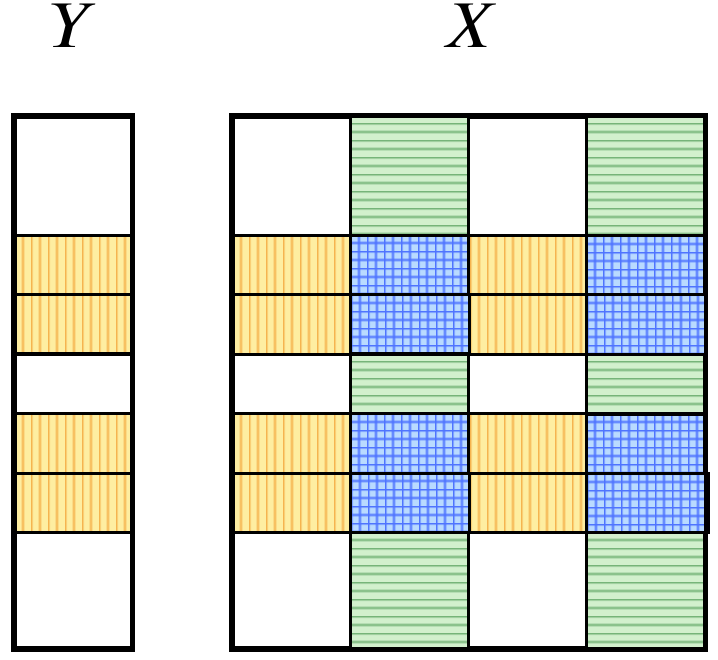}
\end{minipage}
\label{fig:illustration_subsample_3}
}
\subfigure[Bagging scheme in random forests.]{
\begin{minipage}{0.36\linewidth}
\centering
\includegraphics[width=1\linewidth]{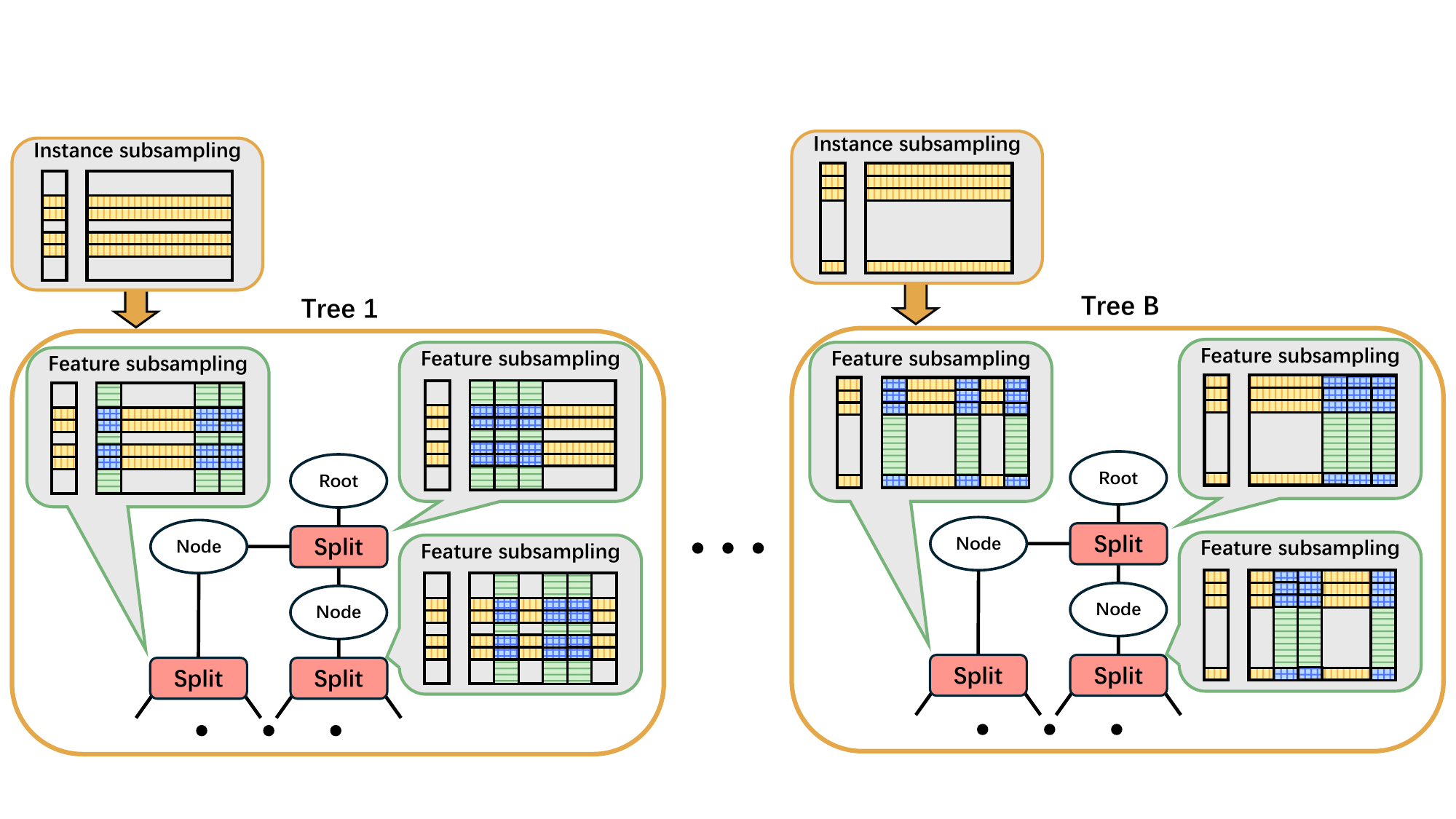}
\end{minipage}
\label{fig:illustration_subsample_4}
\label{fig:illustration}
}
\vskip -0.05in
\caption{
Illustration of subsampling and bagging schemes.
In the random forest example (d), all splits within a single tree share the
same instance subsample (yellow lines), while each split uses an independently
drawn feature subsample (green columns).
}
\vskip -0.05in
\label{fig:illustration_subsample}
\end{figure}

\subsection{Stability}\label{sec:stability_intro}

We begin by fixing notation for the two perturbation axes considered in this
paper. Let $\cD = (\bX, \by)$ be a dataset with $n$ instances and $d$ features.
The data matrix $\bX \in \RR^{n \times d}$ has rows
\(\bx_i\in\cX=\RR^d\), and the labels are collected in
\(\by\in\cY^n\).
For integer $K$, let $[K]=\{1,\dots,K\}$.
For $i\in[n]$, denote by $\cD^{(i),:}$ and $\cD^{-i,:}$ the datasets obtained by
replacing or removing the $i$-th sample from $\cD$, with corresponding data
matrices $\bX^{(i),:}\in\RR^{n\times d}$ and $\bX^{-i,:}\in\RR^{(n-1)\times d}$.
Similarly, for $j\in[d]$, let $\bX^{:, (j)}$ and $\bX^{:,-j}$ denote the data
matrices with the $j$-th feature replaced or removed, and write
$\cD^{:, (j)}$ and $\cD^{:,-j}$ for the associated datasets.
Throughout the paper, our primary convention is feature or instance removal.
Replacement-based variants are also common in the stability literature, and
we discuss their relationship to our removal-based convention in
Section~\ref{sec:related-work}.

Following the standard algorithmic-stability setup, we let a learning algorithm
\(\cA\) map a dataset and a random seed to a predictor
\(f=\cA(\cD,\xi):\cX\to\cY\), where \(\xi\) collects all sources of algorithmic
randomness. Conditional on \(\xi\), the algorithm is deterministic; for example,
\(\xi\) may encode random initialization, data shuffling, or subsampling.
Classical instance stability quantifies the sensitivity of \(\cA\) to
instance-wise perturbations of \(\cD\). We recall this notion first, and defer a
discussion of related formulations to Section~\ref{sec:related-work}.

\begin{definition}\label{def:stability-instance}
An algorithm $\cA$ is $\phi$-instance stable if
\begin{align}\label{equ:stability-instance}
\mathrm{II}:= \frac{1}{n} \sum_{i = 1}^n
\mathbb{E} \left[\left( f(\bx)- f^{-i,  : }(\bx)\right)^2\right] \leq \phi^2,
\end{align}
where $f=\cA(\cD;\xi)$ and $f^{-i,:}=\cA(\cD^{-i,:};\xi)$ are evaluated using
the same algorithmic randomness $\xi$. 
\end{definition}

The expectation is interpreted according
to the stability convention being used. In the uniform convention, the
expectation is over algorithmic randomness and the bound holds for fixed
$\cD$ and $\bx$. In distributional variants, the expectation may also average
over the randomness in $\cD$ and $\bx$. With some abuse of notation, we call the left-hand side of
\eqref{equ:stability-instance} the instance instability of $\cA$, with smaller
values indicating greater instance stability.
The uniform convention is the standard instance-stability convention in the
algorithmic-stability literature, while the distributional convention is weaker
and is used in some of our average-case results, particularly in
Section~\ref{sec:OLS}. The average over \(i\) makes the criterion invariant to
permutations of the dataset, even when \(\cA\) itself is not
permutation-invariant. If \(\cA\) is permutation-invariant, then
\eqref{equ:stability-instance} reduces to
\(\mathbb{E}[(f(\bx)-f^{-i,:}(\bx))^2]\le \phi^2\) for any \(i\in[n]\), under
the same expectation convention.

To study feature bagging, we need the analogous notion along the feature axis.
Rather than modifying training instances, we remove one feature from the input
representation and measure the resulting change in prediction.

\begin{definition}\label{def:stability-feature}
An algorithm $\cA$ is $\phi$-feature stable if
\begin{align}\label{equ:stability-feature}
\mathrm{FI}:=\frac{1}{d} \sum_{j = 1}^d
\mathbb{E} \left[\left( f(\bx)- f^{:,  -j }(\bx)\right)^2\right] \leq \phi^2,
\end{align}
where $f=\cA(\cD ; \xi)$ and
$f^{:,  -j }=\cA\left(\cD^{:,  -j } ; \xi\right)$ are evaluated using the same
algorithmic randomness. By convention, we regard $f^{:,-j}$ as a predictor on
the original feature space that ignores the $j$-th feature. 
\end{definition}

The expectation
follows the same convention as in
Definition~\ref{def:stability-instance}. We call the left-hand side of
\eqref{equ:stability-feature} the feature instability of $\cA$, with smaller
values indicating greater feature stability.
The preceding definition fixes a shared seed across the full-data and
feature-removed runs. We next allow these two runs to use
coupled randomness. This is useful when removing a feature changes the
admissible resampling space, as in the linear-model analysis of
Section~\ref{sec:OLS}.

\begin{definition}[Coupled feature stability]\label{def:coupled-feature-stability}
For a dataset \(\cD\), let \(\Xi(\cD)\) denote the seed space used to run
\(\cA\) on \(\cD\). For each feature \(j\), let \(\Gamma_j(\cD)\) be a coupling
on \(\Xi(\cD)\times\Xi(\cD^{:,-j})\), that is, a joint law for the seeds
\((\xi,\xi^j)\) used in the full-data and feature-removed runs.
An algorithm \(\cA\) is \(\phi\)-feature stable under the coupling family
\(\Gamma=\{\Gamma_j(\cD)\}_{j=1}^d\) if
\begin{align}
\frac{1}{d}\sum_{j=1}^d
\mathbb{E}
\left[
\left(
\cA(\cD;\xi)(\bx)
-
\cA(\cD^{:,-j};\xi^j)(\bx)
\right)^2
\right]
\leq \phi^2,
\label{equ:coupled-feature-stability}
\end{align}
where, in the \(j\)-th summand, \((\xi,\xi^j)\sim\Gamma_j(\cD)\), and any
additional averaging follows the same convention as above. The left-hand side is
the coupled feature instability under \(\Gamma\).
\end{definition}

Definition~\ref{def:stability-feature} is recovered as the shared-randomness
special case of Definition~\ref{def:coupled-feature-stability}: after
identifying the two seed spaces, \(\Gamma_j(\cD)\) is the diagonal coupling
\(\xi^j=\xi\). Unless otherwise stated, FI refers to this shared-randomness
convention. Section~\ref{sec:OLS} is an exception. In the linear-model analysis,
the reduced-feature run shares the unaffected sketching randomness but redraws
the feature sketch from the reduced feature set, corresponding to a
affected-axis resampling coupling with \(\xi^j\ne\xi\). In contrast,
Section~\ref{sec:model-free-feature-stability} returns to the
shared-randomness coupling \(\xi^j=\xi\), because its goal is a conditional
algorithmic stability guarantee.

The FI criterion in \eqref{equ:stability-feature} differs from the II criterion
in \eqref{equ:stability-instance} in two important respects. First, FI perturbs
only the input matrix and leaves the response vector \(\by\) unchanged. Second,
the average over features is essential: features are generally non-exchangeable,
and learning algorithms are typically not permutation-invariant along the
feature dimension.

\subsection{Why Do We Care about Feature Stability?}
\label{sec:exp_for_feature_stability_meaning}


Feature stability is motivated by several practical and methodological concerns.
First, relevant features may be absent, unavailable, or deliberately excluded;
feature inclusion can therefore depend on preprocessing decisions, acquisition
constraints, and analyst judgment
\citep{jeng2024weak, ma2025locally, shen2025can, poudel2025multimodal}. The FI
measure quantifies how sensitive an algorithm is to this feature-level
variation. Second, feature entries may themselves be perturbed because of
corruption \citep{mcwilliams2014fast}, missingness
\citep{little2019statistical, chen2023handling}, or measurement error
\citep{hou2026nonparametric}. Algorithms with stronger feature stability are
therefore more robust to such perturbations.


Beyond these practical considerations, we empirically examine whether FI carries
information about generalization, paralleling a central motivation for the
literature on instance stability
\citep{bousquet2002stability, hardt2016train}. In a controlled synthetic
regression benchmark, specifically the MARSadd random-forest
benchmark \citep{friedman1991multivariate, mentch2020randomization,curth2024random},
where the signal combines a weak component distributed across many features
with a concentrated nonlinear component, we vary model complexity and
subsampling behavior in random forest models, and measure how FI and II account
for variation in the generalization gap. For each configuration of
\texttt{max\_features}, \texttt{max\_samples}, and \texttt{max\_depth}, we
compute the train-test MSE gap together with II and FI by retraining after
removing one instance or one feature. We then quantify whether FI explains
generalization variation beyond II using linear \(R^2\), mutual information,
and residual-based partial mutual information, with Gaussian and permuted-noise
controls. Section~\ref{app:pmi_fs_generalization} provides the detailed
experiments and analysis. The results indicate that FI captures
generalization-relevant variation that is not explained by II alone. Across
Tables~\ref{tab:pmi-synthetic-full}--\ref{tab:dr2-real-full}, the conditional
contribution of FI after accounting for II remains positive in all displayed
synthetic and real-data settings. The noise-control results in
Tables~\ref{tab:pmi-noise-synthetic-full}--\ref{tab:dr2-noise-real-full}
further show that replacing FI with unrelated Gaussian or permuted noise does
not replicate this signal. These findings indicate that FI and II provide
complementary measures for capturing generalization-relevant variation.





\subsection{Instance and Feature Bagging}\label{sec:subsampling}

We now formalize the bagging schemes analyzed in the paper. For both instances
and features, we use sub-bagging: each base learner is trained on a uniformly
sampled fixed-size subset without replacement. Other bagging schemes can lead to
similar theoretical conclusions \citep{soloff2024bagging}, but fixed-size
subsampling is especially natural for features. Classical bagging
\citep{breiman1996bagging} and Poissonized bagging \citep{oza2001online} sample
with replacement, which can select the same feature multiple times and create
non-identifiability issues. Bernoulli sub-bagging
\citep{harrington2003online, wu2025ensemble} produces random subset sizes, which
can complicate models that require a fixed input dimension, such as neural
networks.

\noindent\textbf{Instance bagging.}
We first consider instance bagging, the classical bagging setting. In each
sub-bagging round, \(m\) of the \(n\) instances are sampled uniformly without
replacement. Let \(p=m/n\) denote the instance-subsampling ratio, and write
$
\bmu^{(b)} = (\bmu_1^{(b)}, \dots, \bmu_m^{(b)})^{\top}
$
for the indices selected in the \(b\)-th subsample.
For any algorithm $\cA_0$, we define its evaluation on the $b$-th subsample as
$
\cA(\cD, \xi^{(b)}) := \cA_0(\cD^{\bmu^{(b)}, :}),
$
where
\(\cD^{\bmu^{(b)}, :}
= \{(\bx_{\bmu^{(b)}_1}, y_{\bmu^{(b)}_1}), \dots,
(\bx_{\bmu^{(b)}_m}, y_{\bmu^{(b)}_m})\}\). See
Figure~\ref{fig:illustration_subsample_1} for an illustration. This construction
defines a randomized algorithm \(\cA\), with randomness induced by instance
subsampling. Let \(f^{(b)}:=\cA(\cD,\xi^{(b)})\) be the predictor trained in the
\(b\)-th round. The bagged predictor averages \(B\) such predictors:
\begin{align}
\label{equ:bagged-estimator}
f^B(\bx) := \frac{1}{B} \sum_{b=1}^B f^{(b)}(\bx).
\end{align}
We refer to \(B\) as the number of bagging rounds. The infinite-bagging limit is
the corresponding expectation over the subsampling randomness:
\begin{align}
\label{equ:infinite-bagged-estimator}
\cA_\infty(\cD)(\cdot) := \mathbb{E}_{\xi^{(b)}}[f^{(b)}(\cdot)] = \mathbb{E}_{\xi^{(b)}}[\cA(\cD, \xi^{(b)})(\cdot)].
\end{align}
We write its output as \(f^\infty\), so
\(f^\infty(\bx)=\cA_\infty(\cD)(\bx)
=\mathbb{E}_{\xi^{(b)}}[f^{(b)}(\bx)]\). This predictor is the large-\(B\)
limit of \(f^B\).

\noindent\textbf{Feature bagging.}
Feature bagging is defined analogously.
Let $s \leq d$ denote the number of selected features and $q = s/d$ the subsampling ratio.
In each round, a subset $\bnu^{(b)} = (\bnu^{(b)}_1,\dots,\bnu^{(b)}_s)$ is drawn
uniformly at random without replacement.
The corresponding randomized algorithm is
$\cA(\cD, \xi^{(b)}) := \cA_0(\cD^{:, \bnu^{(b)}})$.
Here, $\cD^{:, \bnu^{(b)}}$ denotes the dataset restricted to the selected
features and is illustrated in Figure~\ref{fig:illustration_subsample_2}.
The finite and infinite bagged predictors are defined as in
\eqref{equ:bagged-estimator} and \eqref{equ:infinite-bagged-estimator}.
At prediction time, the same test point is evaluated by all base learners, but
the \(b\)-th learner only receives the coordinates \(\bx^{\bnu^{(b)}}\) that
match its sampled feature subset.  The final prediction is the average of these
base-learner predictions.  This differs from instance bagging, where every base
learner receives the full test point.

\noindent\textbf{General bagging.} Instance and feature bagging can be combined
by jointly subsampling both axes, yielding an ensemble whose randomness arises
from both sources; see Figure~\ref{fig:illustration_subsample_3}.
Section~\ref{sec:OLS} explores their interaction.
In random forests, feature bagging is implemented recursively rather than once
per tree: as Figure~\ref{fig:illustration} illustrates, all splits in a tree
share the same instance subsample, while feature subsampling is performed
independently at each node. Section~\ref{sec:model-free-feature-stability}
analyzes this recursive feature-subsampling strategy.


\section{Stability under Model Assumptions} \label{sec:OLS}

This section studies the effect of bagging on the II and FI measures in linear regression,
where explicit model assumptions allow sharp comparisons.


\subsection{Basic Settings}\label{sec:linear-basic-setting}

\noindent\textbf{Generating Distribution.}
We consider the standard linear regression model \citep{lejeune2020implicit, chen2023sketched}:
\begin{align} \label{equ:linear-model}
y_i = \bx_i^{\top} \bbeta^* + \varepsilon_i,
\end{align}
where $\bx_i \in \mathbb{R}^d$ is the feature vector, $y_i \in \mathbb{R}$ is the response, $\varepsilon_i \sim \mathcal{N}(0, \sigma^2)$ is Gaussian noise, and ${(\bx_i, \varepsilon_i)}_{i=1}^n$ are i.i.d.\ copies of $(\bx, \varepsilon)$.
Additionally, we assume that $\bx \sim \cN(\zero, \bI)$, $\bbeta^* \sim \mathcal{N}(\zero, \bI/d)$, and that $\bx$, $\bbeta^*$, and the noise $\varepsilon$ are mutually independent. The Gaussian assumptions on $\bx$ and $\varepsilon$ can be generally relaxed to bounded moment conditions under which the same theoretical results hold, a phenomenon known as universality \citep{hastie2022surprises, chen2023sketched, wu2025ensemble}. To avoid introducing technical overhead, we adopt the Gaussian setting.
Let $\bX = (\bx_1, \ldots, \bx_n)^{\top}$ denote the feature matrix, $\by = (y_1, \ldots, y_n)^{\top}$ the response vector, and $\bvarepsilon = (\varepsilon_1, \ldots, \varepsilon_n)^{\top}$ the noise vector. We work under the proportional asymptotic regime, where $d, n \to \infty$ with $d/n \to \gamma$ for some constant $\gamma$. This regime has been widely adopted in recent studies analyzing the exact risk behavior of linear models; see, e.g., \citet{hastie2022surprises, chen2023sketched, wu2025ensemble}.

\begin{table}[t!]
\centering
\caption{Closed-form asymptotic limits for the four \(\Delta_{\ell}\)-terms in
Theorem~\ref{thm:stabilitylinear-decomp-subscript}.
}
\label{tab:linear-stability-limits}
{\setlength{\tabcolsep}{11pt}
\renewcommand{\arraystretch}{1.50}
\makebox[\textwidth][c]{\resizebox{0.96\textwidth}{!}{%
\begin{tabular}{@{}lcccc@{}}
\toprule
\textbf{Setting and Case} &
\(\Delta_\ell^{V=}\) &
\(n\Delta_\ell^{V\neq}\) &
\(\Delta_\ell^{B=}\) &
\(n\Delta_\ell^{B\neq}\) \\
\midrule
II ($\ell = -i$),\,\, \(\gamma q<p\) &
\(\frac{2\gamma q(1-p)}{(p-\gamma q)(1-\gamma q)}\) &
\(\frac{\gamma q^2}{(1-\gamma q^2)^2}\) &
\(\frac{2\gamma q(1-q)(1-p)}{(p-\gamma q)(1-\gamma q)}\) &
\(\frac{\gamma q^2(1-q)^2}{(1-\gamma q^2)^2}\) \\
II ($\ell = -i$),\,\, \(\gamma q>p\) &
\(\frac{2\gamma q(p-p^2)}{(\gamma q-p)(\gamma q-p^2)}\) &
\(\frac{\gamma p^2}{(\gamma-p^2)^2}\) &
\(\frac{2p(1-p)(\gamma^2q-2\gamma pq+p^2)}{\gamma(\gamma q-p)(\gamma q-p^2)}\) &
\(\frac{p^2(\gamma-p)^2}{\gamma(\gamma-p^2)^2}\) \\
FI ($\ell = -j$), \(\gamma q<p\) &
\(\frac{2\gamma pq(1-q)}{(p-\gamma q)(p-\gamma q^2)}\) &
\(\frac{q^2}{(1-\gamma q^2)^2}\) &
\(\frac{2pq(1-q)(p+\gamma-2\gamma q)}{(p-\gamma q)(p-\gamma q^2)}\) &
\(\frac{q^2(1+\gamma-2\gamma q)}{\gamma(1-\gamma q^2)^2}\) \\
FI ($\ell = -j$), \(\gamma q>p\) &
\(\frac{2\gamma p(1-q)}{(\gamma q-p)(\gamma-p)}\) &
\(\frac{p^2}{(\gamma-p^2)^2}\) &
\(\frac{2p(1-q)}{\gamma q-p}\) &
\(\frac{p^2(1-2p+\gamma)}{(\gamma-p^2)^2}\) \\
\bottomrule
\end{tabular}}}}
\end{table}

\noindent\textbf{Bagged least square estimator.}
Following \citet{wu2025ensemble}, we reformulate the bagged least square estimator as an average of sketched least square estimators.
 As described in Section~\ref{sec:subsampling}, let $\bmu^{(b)}$ and $\bnu^{(b)}$ denote the instance and feature indices selected in the $b$-th subsample. Define the sketching matrices $\bU_b \in \mathbb{R}^{n \times m}$ and $\bV_b \in \mathbb{R}^{d \times s}$ as the corresponding selection matrices: $\bU_b$ has ones at entries $(\bmu^{(b)}_i, i)$ for $i \in [m]$, and zeros elsewhere; similarly, $\bV_b$ has ones at entries $(\bnu^{(b)}_j, j)$ for $j \in [s]$.
The matrix $\bU_b^{\top} \bX \bV_b \in \mathbb{R}^{m \times s}$ extracts a submatrix of $\bX$ formed by the sampled rows and columns.

For each subsample, we compute the sketched minimum-norm least squares estimator
\[
\bbeta^{(b)} = \bV_b \left( \bU_b^{\top} \bX \bV_b \right)^{\dagger} \bU_b^{\top} \by,
\]
where $(\cdot)^{\dagger}$ denotes the Moore–Penrose pseudoinverse.
The final bagged estimator is obtained by averaging over $B$ such subsamples $\bbeta := \frac{1}{B} \sum_{b=1}^B \bbeta^{(b)}$.

\noindent\textbf{Affected-axis resampling coupling.}
This section uses an affected-axis resampling coupling for the sketching
randomness: after an instance or feature is removed, we redraw the sketching
matrix on the perturbed axis from the reduced index set, while keeping the
sketching matrix on the other axis fixed. Thus the reduced run is neither a
shared-seed run nor a fully independent rerun.
After removing the $i$-th instance from the dataset, we independently resample a sketching matrix $\bU_b^{\prime}$ from the reduced index set $[n]\setminus\{i\}$.
The resulting bagged estimator after instance removal is
\begin{align}\label{equ:instance-removed-ols}
\bbeta^{-i, :} := \frac{1}{B} \sum_{b=1}^B \bV_b \left( \bU_b^{\prime\top} \bX \bV_b \right)^{\dagger} \bU_b^{\prime\top} \by.
\end{align}
Similarly, removing the $j$-th feature corresponds to resampling $\bV_b^{\prime}$ from $[d]\setminus\{j\}$, yielding
\begin{align}\label{equ:feature-removed-ols}
\bbeta^{:, -j} := \frac{1}{B} \sum_{b=1}^B \bV_b^{\prime} \left( \bU_b^{\top} \bX \bV_b^{\prime} \right)^{\dagger} \bU_b^{\top} \by.
\end{align}


\subsection{Stability Analysis}

Under this affected-axis resampling coupling, we study the expected instability
measures
\(\frac{1}{n}\sum_{i=1}^n
\mathbb{E}_{\bx,\bbeta,\bbeta^{-i,:}}
[(\bbeta^{\top}\bx-(\bbeta^{-i,:})^{\top}\bx)^2]\) for II and
\(\frac{1}{d}\sum_{j=1}^d
\mathbb{E}_{\bx,\bbeta,\bbeta^{:,-j}}
[(\bbeta^{\top}\bx-(\bbeta^{:,-j})^{\top}\bx)^2]\) for FI.
Since the features are isotropic, taking expectation over $\bx$ reduces both
quantities to
\begin{align} \label{equ:expected-stability}
\text{Instability} = \mathbb{E} \left[\left\|\bbeta- \bbeta^{\ell} \right\|_2^2\right],
\quad \ell \in \{-i,-j\},
\end{align}
where the subscript of  \(\EE_{\bbeta,\bbeta^\ell}\) is omitted for brevity, \(\ell=-i\) denotes $i$-th instance removal, and \(\ell=-j\) denotes  $j$-th  feature removal.
We slightly abuse notation by writing \(\bbeta^\ell\) for either
\(\bbeta^{-i,:}\) or \(\bbeta^{:,-j}\). Rather than upper bounding this
instability, we characterize the exact limiting value of
\eqref{equ:expected-stability} under proportional asymptotics. Our first result
gives a finite-\(B\) decomposition and the corresponding asymptotic limits.



\begin{theorem}\label{thm:stabilitylinear-decomp-subscript}
Assume the Gaussian linear model in Section~\ref{sec:linear-basic-setting}.
Then, for $\ell\in\{-i,-j\}$,
\begin{align}\label{eq:finiteB-decomp-subscript}
 \mathbb{E}\!\left[\|\bbeta-\bbeta^{\ell}\|_2^2\right] =
\sigma^2\!\left(\frac{\Delta_{\ell}^{V=}}{B} + \frac{B-1}{B}\Delta_{\ell}^{V\neq}\right)
 \quad +  \frac{\Delta_{\ell}^{B=}}{B} +
\frac{B-1}{B}\Delta_{\ell}^{B\neq}.
\end{align}
Moreover, as $n,d\to\infty$ with $d/n\to\gamma$,
\begin{align}\label{eq:finiteB-exact-term}
\Delta_{\ell}^{V=},  \Delta_{\ell}^{B=}\to\Theta(1),
\text{ and }
 \Delta_{\ell}^{V\neq}, \Delta_{\ell}^{B\neq}\to \Theta(\frac{1}{n}),
\end{align}
where $\Theta(\cdot)$ denotes asymptotic equality up to constants.
The exact asymptotic limits of $\Delta_{\ell}^{V=}$, $n\Delta_{\ell}^{V\neq}$,
$\Delta_{\ell}^{B=}$, and $n\Delta_{\ell}^{B\neq}$ are given in
Table~\ref{tab:linear-stability-limits} for instance removal and feature
removal.
\end{theorem}

The proof of Theorem~\ref{thm:stabilitylinear-decomp-subscript} is given in
Appendix~\ref{app:linear-stability-proofs}. The theorem decomposes the expected
instance and feature instability of the bagged least-squares estimator into
noise-variance terms and signal-bias terms. We use the superscripts \(V\) and
\(B\) for these two contributions, respectively; the superscript \(B\) denotes
bias and should not be confused with the number of bagging rounds. The \(=\)
terms compare paired estimators built from the same subsampling realization and
receive weight \(1/B\), whereas the \(\neq\) terms compare different submodels
and receive weight \((B-1)/B\).

This decomposition separates the effects of subsampling and bagging. When
\(B=1\), the estimator is a single subsampled minimum-norm least-squares fit,
so only the same-submodel terms remain. These terms inherit the interpolation
behavior of an individual subsampled estimator and become large near the
effective interpolation threshold \(\gamma q=p\), as shown in
Figure~\ref{fig:stability_basic_subsampling}.
Both
II and FI have a double-descent-shaped instability curve, with a peak near
\(\gamma q=p\), where the subsampled design matrix is nearly singular. Changing
the subsampling ratios \(p\) and \(q\) moves this threshold through the
effective aspect ratio \(\gamma q/p\), but does not average away the
same-submodel instability. This is the main contrast with the bagged curves in
Figure~\ref{fig:stability_basic_bagged}.

Increasing \(B\) downweights these
same-submodel terms and transfers mass to the cross-submodel terms
\(\Delta_{\ell}^{V\neq}\) and \(\Delta_{\ell}^{B\neq}\). In the large-\(B\)
limit, the cross terms dominate and are of order \(1/n\), yielding the dashed
curves in Figure~\ref{fig:stability_basic_bagged}. Thus subsampling changes the
effective interpolation threshold and the limiting constants, while bagging
averages away the instability carried by individual subsampled estimators.
Classical stability-based generalization bounds
\citep{bousquet2002stability, elisseeff2005stability, mukherjee2006learning}
do not directly apply to linear regression because the output is unbounded.
Nonetheless, the generalization-error curves observed in
\citep{lejeune2020implicit, wu2025ensemble} show a related qualitative pattern:
subsampling shifts the interpolation threshold, while averaging over subsampled
estimators suppresses the associated peak.

\begin{figure}[!t]
\centering
\subfigure[II with both subsampling]{
\begin{minipage}{0.48\linewidth}
\centering
\includegraphics[width=\textwidth]{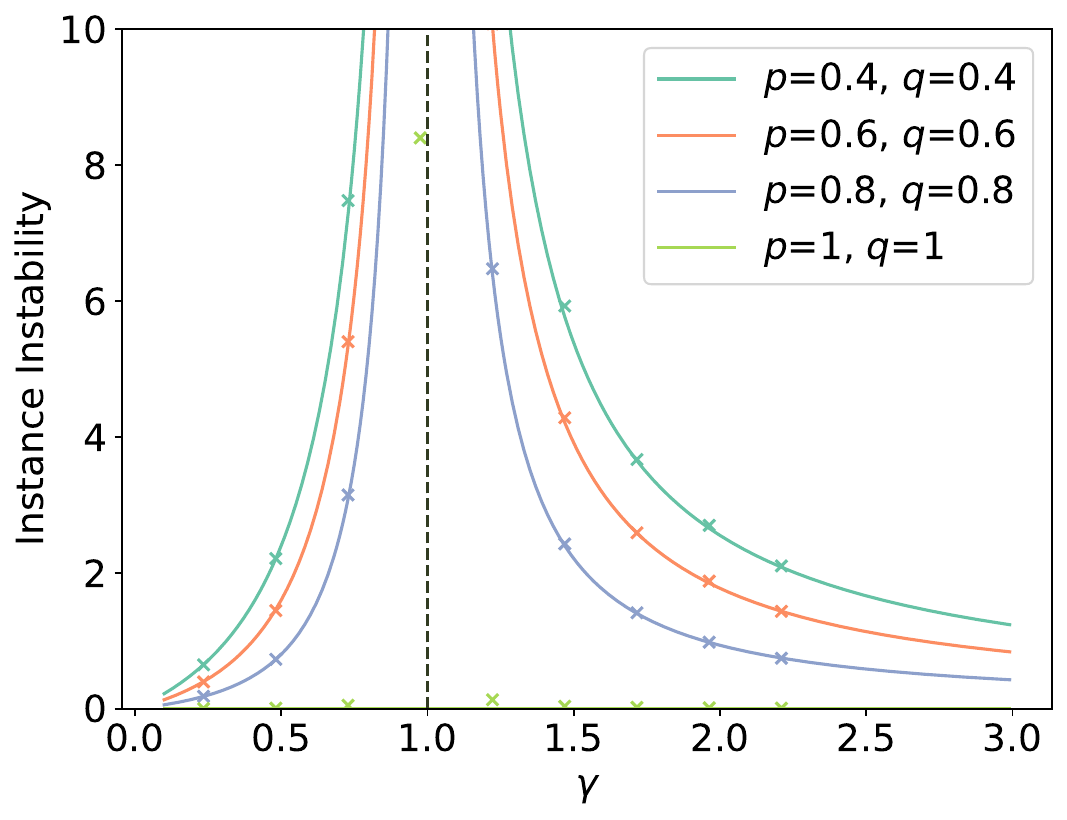}
\end{minipage}
\label{fig:instance_stability_basic_nonbagged_aligned}
}
\subfigure[II with instance subsampling]{
\begin{minipage}{0.48\linewidth}
\centering
\includegraphics[width=\textwidth]{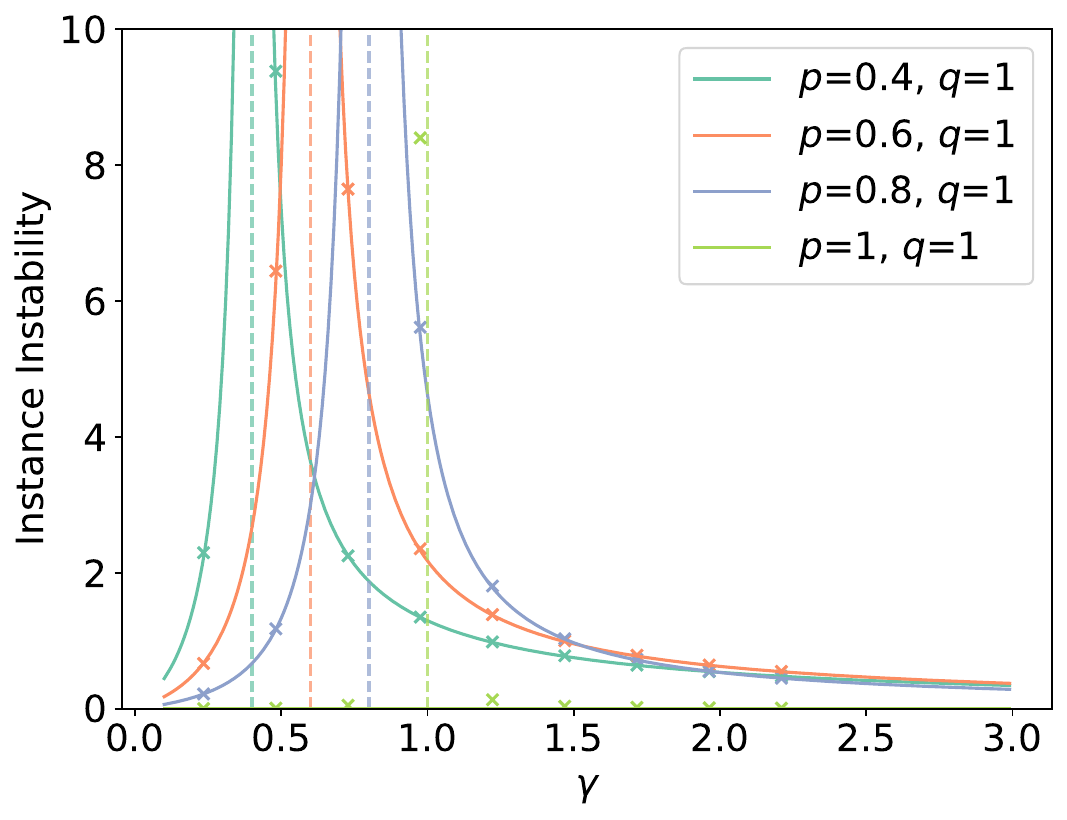}
\end{minipage}
\label{fig:instance_stability_basic_fixed_q_1_nonbagged}
}
\subfigure[FI with both subsampling]{
\begin{minipage}{0.48\linewidth}
\centering
\includegraphics[width=\textwidth]{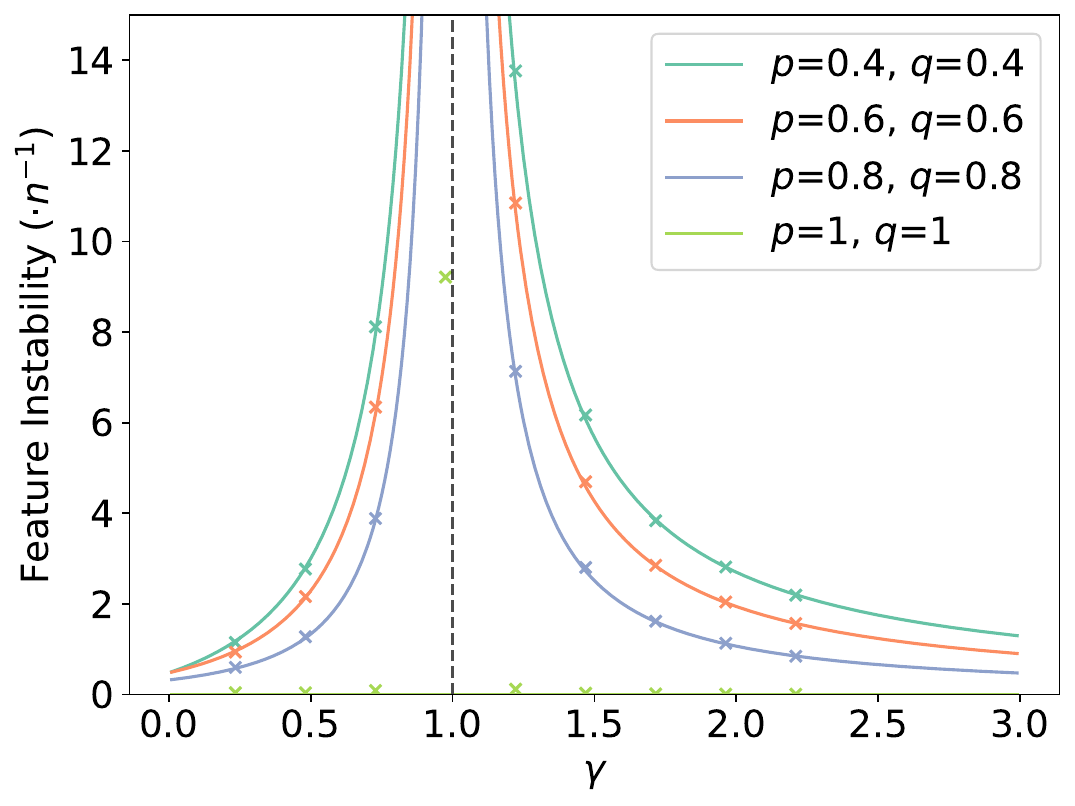}
\end{minipage}
\label{fig:feature_stability_basic_aligned_nonbagged}
}
\subfigure[FI with feature subsampling]{
\begin{minipage}{0.48\linewidth}
\centering
\includegraphics[width=\textwidth]{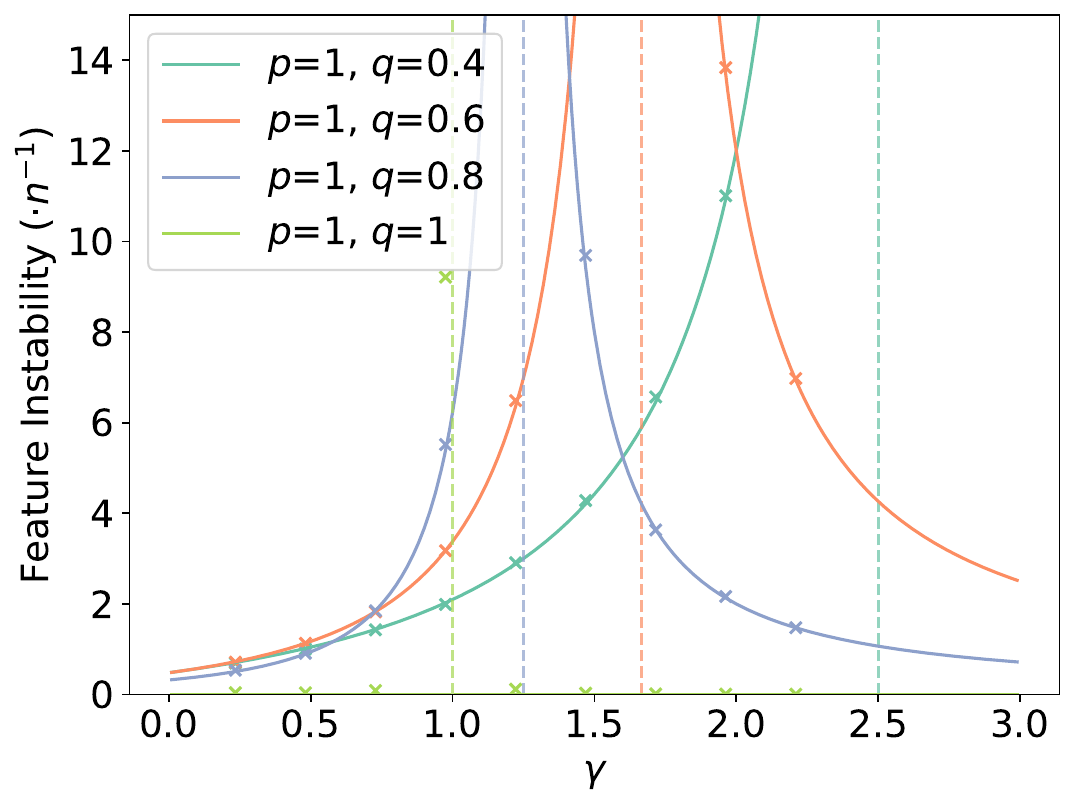}
\end{minipage}
\label{fig:feature_stability_basic_fixed_p_1_nonbagged}
}
\caption{
Instability of the single subsampled minimum-norm OLS estimator as a function
of the aspect ratio \(\gamma\). Colors indicate subsampling ratios \((p,q)\);
solid lines denote theoretical predictions and crosses denote empirical
averages over 300 realizations with \(\sigma=1\), \(n=400\), and
\(d\in\{100,200,\ldots,900\}\). Vertical dashed lines indicate the
interpolation threshold \(\gamma q=p\).
}
\label{fig:stability_basic_subsampling}
\end{figure}

The closed forms further show that feature removal has a different natural
scale from instance removal. In the large-\(B\) limit, most cross terms are
reported on the \(1/n\) scale, but the feature-removal bias term also reflects
the number of available coordinates. In the underparameterized regime
\(\gamma q<p\),
\[
n\Delta_{-j}^{B\neq}
\to
\frac{q^2 (1+\gamma - 2\gamma q)}{\gamma (1-\gamma q^2)^2}
\sim \frac{q^2}{\gamma}
\quad (\gamma\downarrow0).
\]
Since \(d=\gamma n\), the unscaled term behaves as
\(\Delta_{-j}^{B\neq}\asymp q^2/d\). Thus this component of FI is controlled
by the feature dimension, not only by the sample size: removing one feature is
more consequential when \(d\) is small relative to \(n\), while the same
feature-specific contribution vanishes as \(\gamma\to\infty\).

Finally, Theorem~\ref{thm:stabilitylinear-decomp-subscript} quantifies the
finite-\(B\) averaging gain, as illustrated in Figure~\ref{fig:stability_B}.
For small \(B\), the \(1/B\)-weighted same-submodel terms
\(\Delta_{\ell}^{V=}\) and \(\Delta_{\ell}^{B=}\) remain visible. As \(B\)
increases, the instability approaches its order-\(1/n\) large-\(B\) limit; away
from interpolation, taking \(B\) of order \(n\) is already sufficient to make
further bagging rounds have diminishing returns. Near the interpolation
boundary, however, the same-submodel constants can be large, so larger
ensembles may still be needed to flatten the finite-\(B\) peak. The limiting
formulas also show that \(p\) and \(q\) do not enter symmetrically: in the
underparameterized regime \(\gamma q<p\), the leading large-\(B\) constants are
primarily controlled by \(q\), whereas in the overparameterized regime
\(\gamma q>p\), they are primarily controlled by \(p\).

\begin{figure}[!t]
\vskip -0.1in
\centering
\begin{minipage}{0.32\textwidth}
\centering
\subfigure[II with both subsampling]{
\includegraphics[width=0.95\linewidth]{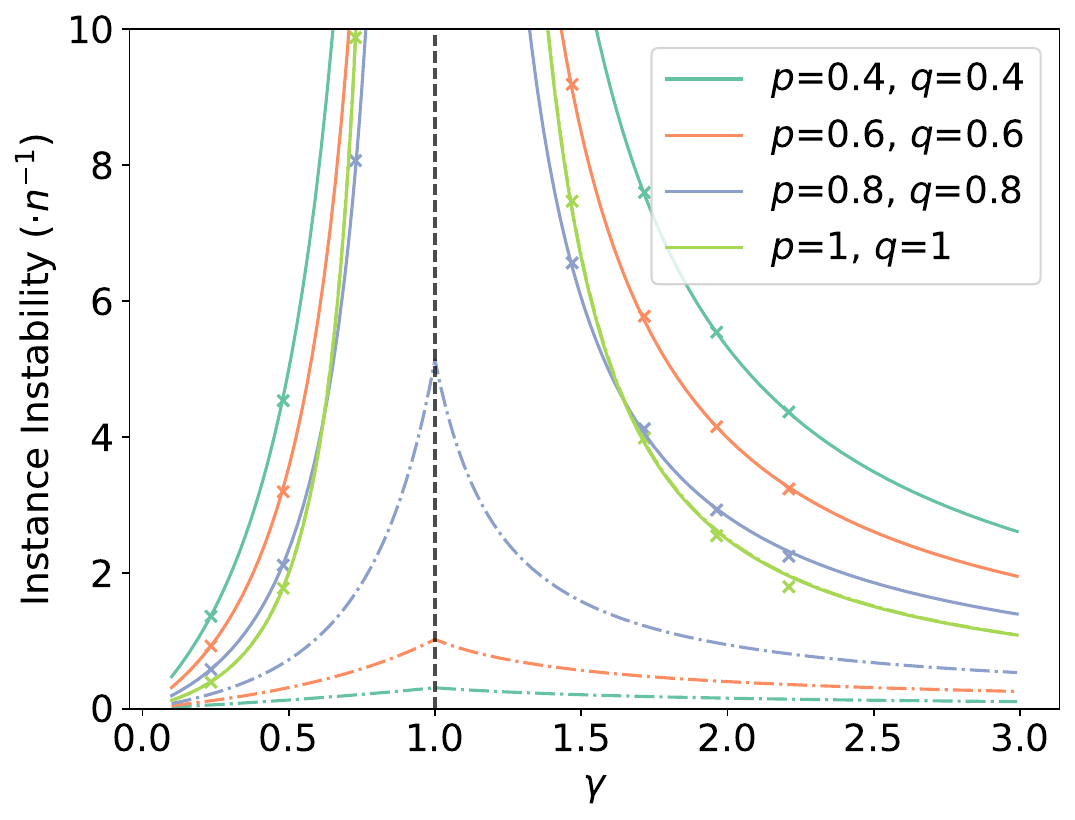}
\label{fig:instance_stability_basic_aligned}
}
\par\vspace{-0.04in}
\subfigure[II with instance subsampling]{
\includegraphics[width=0.95\linewidth]{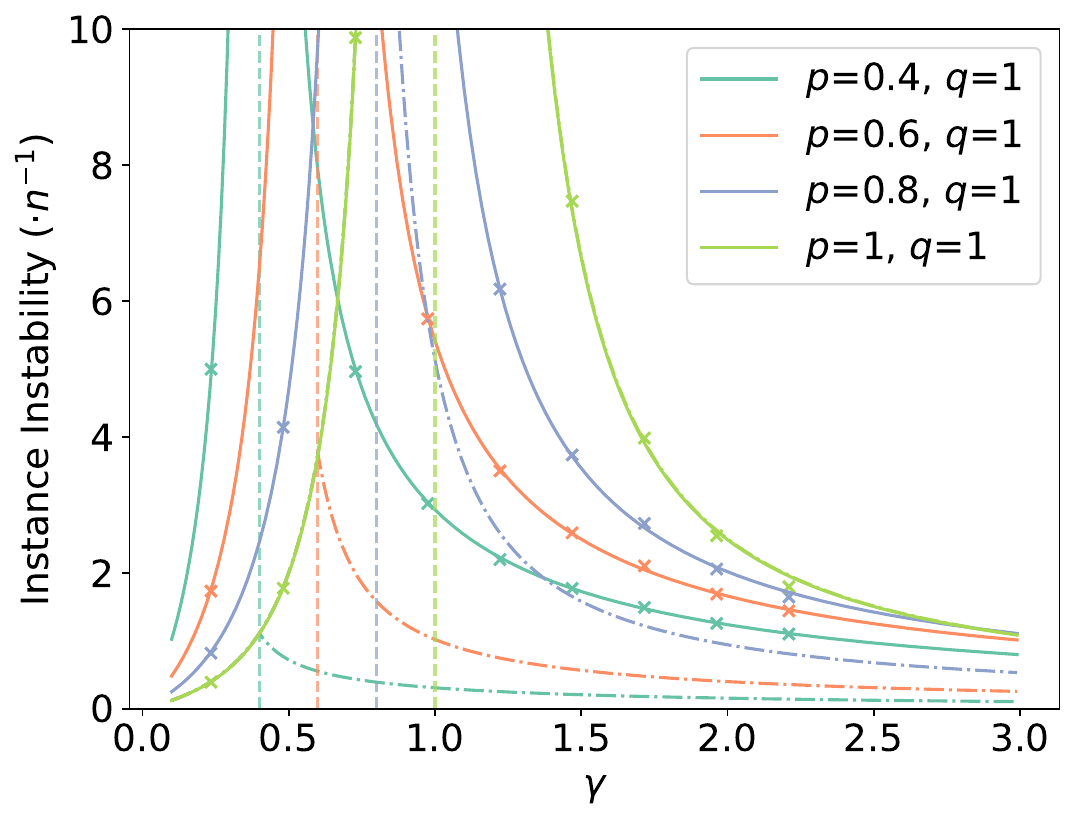}
\label{fig:instance_stability_basic_fixed_q_1}
}
\end{minipage}
\hfill
\begin{minipage}{0.32\textwidth}
\centering
\subfigure[FI with both subsampling]{
\includegraphics[width=0.95\linewidth]{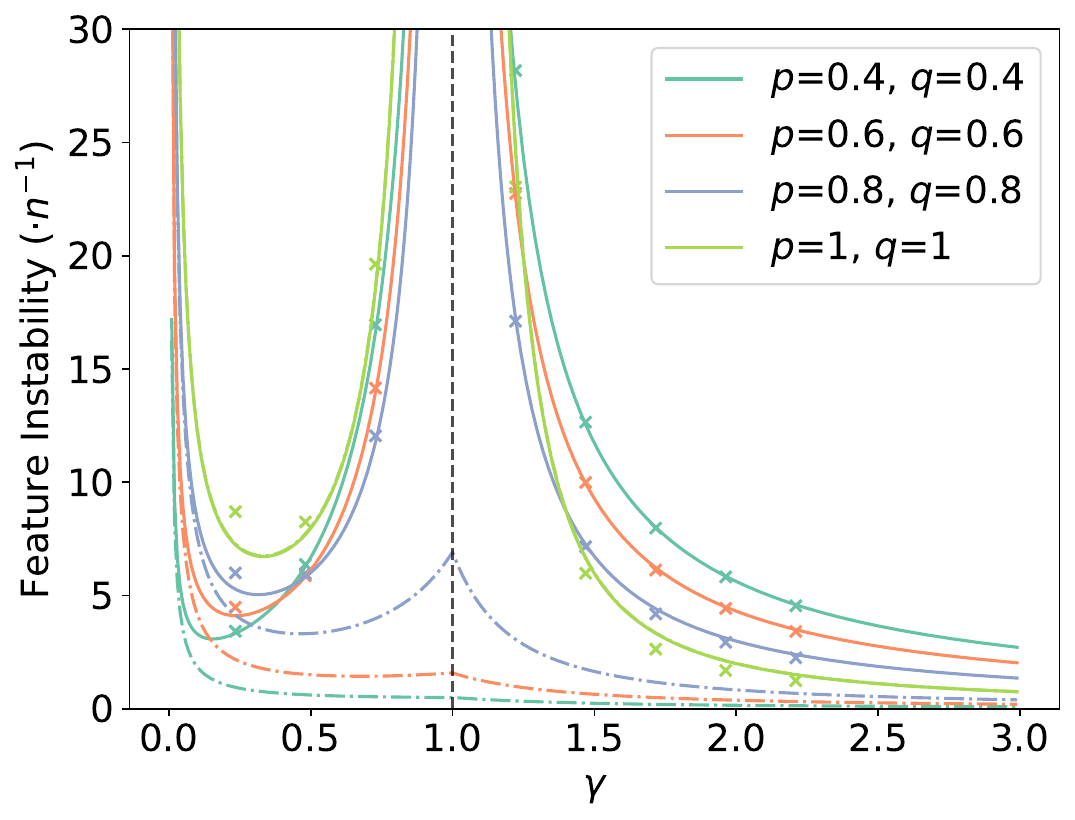}
\label{fig:feature_stability_basic_aligned}
}
\par\vspace{-0.04in}
\subfigure[FI with feature subsampling]{
\includegraphics[width=0.95\linewidth]{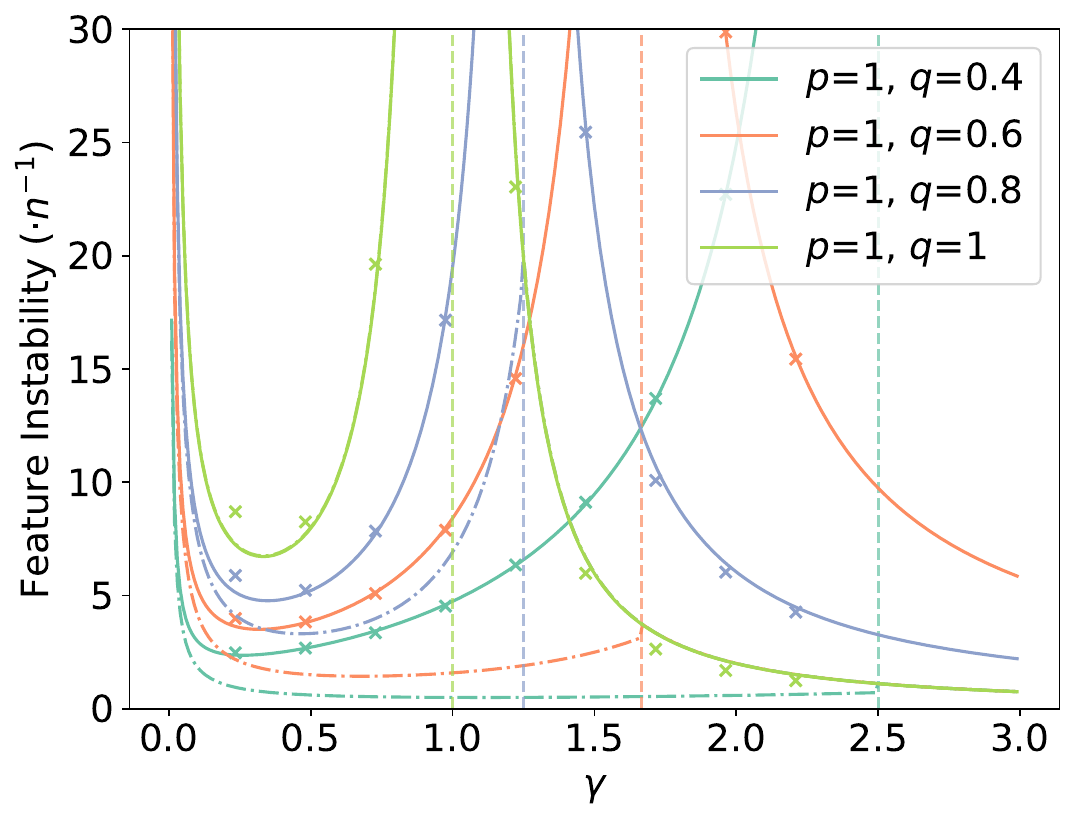}
\label{fig:feature_stability_basic_fixed_p_1}
}
\end{minipage}
\hfill
\begin{minipage}{0.32\textwidth}
\centering
\subfigure[II across bagging rounds]{
\includegraphics[width=0.99\linewidth]{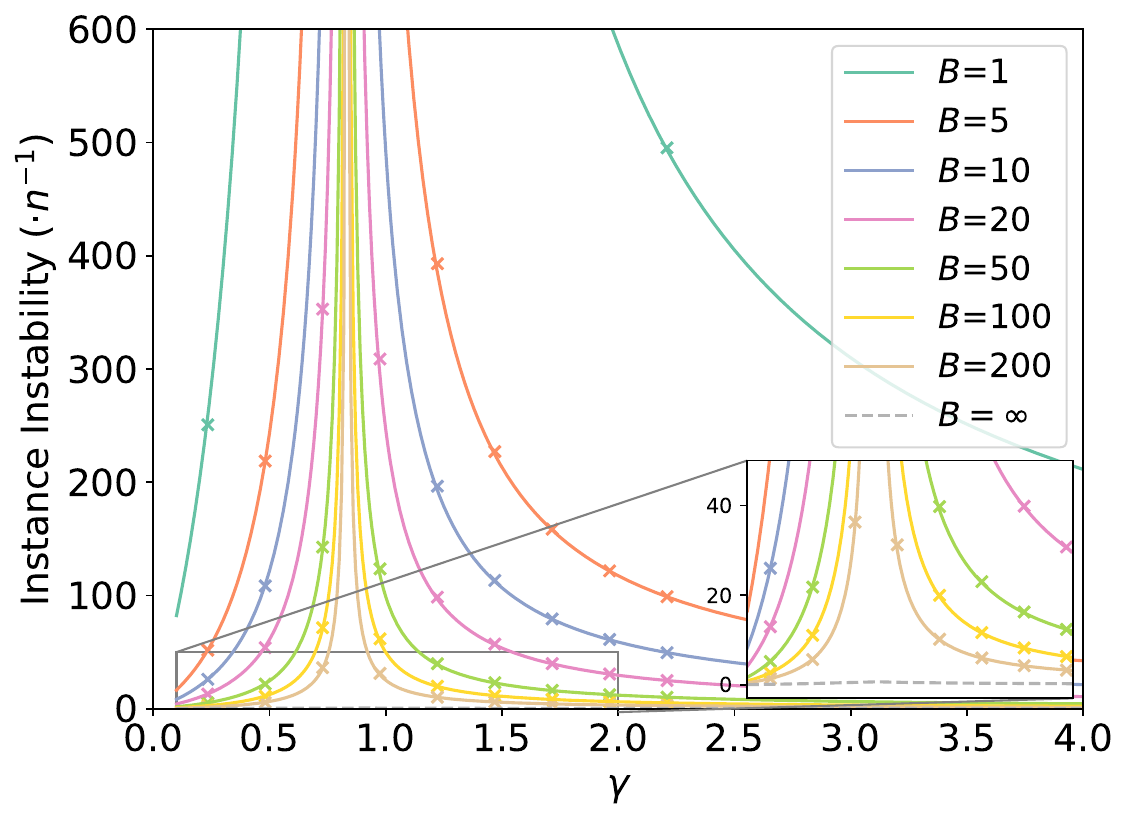}
\label{fig:instance_stability_B}
}
\par\vspace{-0.04in}
\subfigure[FI across bagging rounds]{
\includegraphics[width=0.99\linewidth]{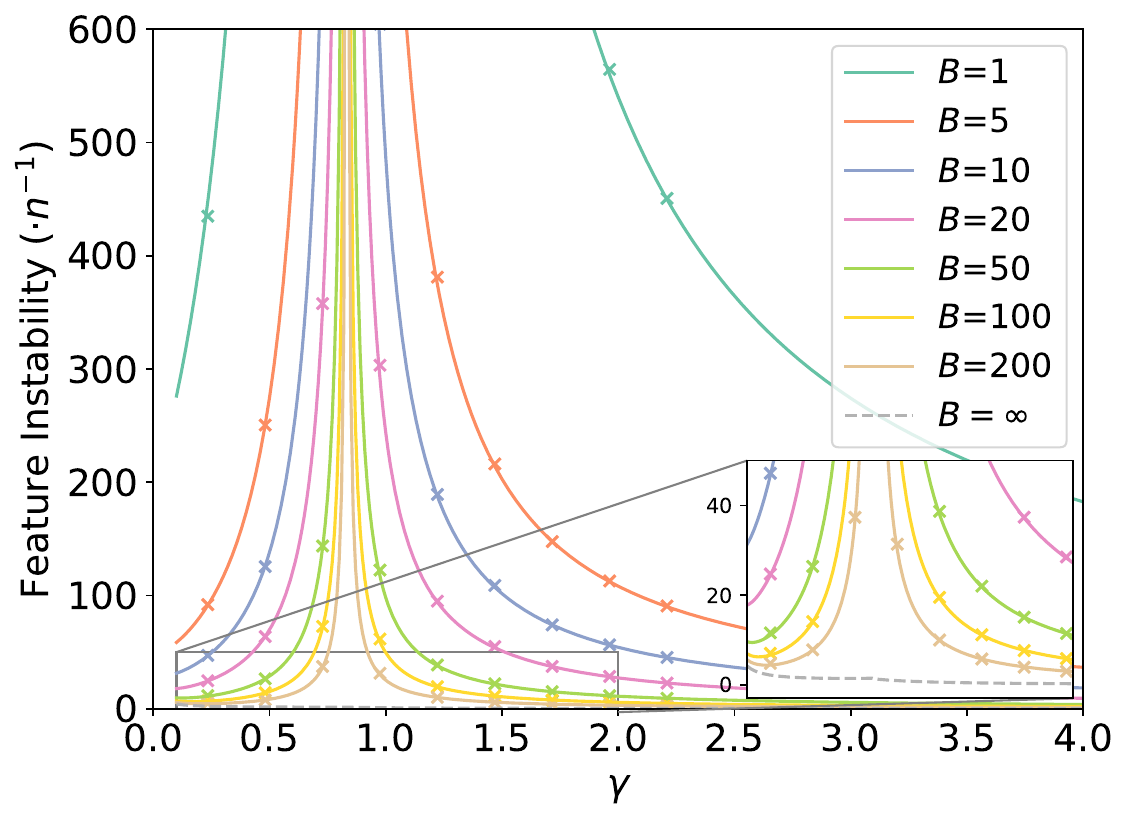}
\label{fig:feature_stability_B}
}
\end{minipage}
\vskip -0.02in
\caption{
Instability \(\EE[\|\bbeta-\bbeta^\ell\|_2^2]\) of the bagged least-squares
estimator. Panels (a)--(d) use \(B=200\) bagging rounds and vary the aspect
ratio \(\gamma\); solid lines denote theoretical predictions, crosses mark
empirical averages, and dashed curves show the infinite-bagging limit. Panels
(e)--(f) compare different numbers of bagging rounds on an \(n^{-1}\) scale;
colors indicate \(B\), solid lines denote theoretical predictions, and crosses
mark empirical averages. All panels use the same simulation setting as in
Figure~\ref{fig:stability_basic_subsampling}.
}
\vskip -0.1in
\label{fig:stability_basic_bagged}
\label{fig:stability_B}
\end{figure}

\section{Model-free Feature Stability}\label{sec:model-free-feature-stability}

In this section, we investigate FI under model-free settings.
We first develop a general-space feature-stability tool, with the scalar
\(\mathbb{R}\) result appearing as a special case
(Section~\ref{sec:extention-metric-space}), and then extend it to recursive
subsampling (Section~\ref{sec:recursivesubsampling}).
These results are applied to random forward selection and random forests in
Sections~\ref{sec:rfs} and~\ref{sec:random-forest-main}, and extended to
finite bagging rounds in Section~\ref{sec:finite-bagging}.
Unlike linear regression analysis in Section~\ref{sec:OLS}, the
model-free results use the shared-randomness, conditional-on-the-dataset
coupling in Definitions~\ref{def:stability-instance}--\ref{def:stability-feature},
namely \(\xi^j=\xi\).
The aim is therefore a uniform algorithmic guarantee
rather than an exact distributional perturbation calculation.

\subsection{Model-free Stability in General Spaces}
\label{sec:extention-metric-space}

This subsection develops the general-space stability tool used later for recursive feature subsampling, used in  random forests \citep{breiman2001random, zhang2024adaptive}. As illustrated in Figure~\ref{fig:illustration}, each tree uses a fixed subset of instances, while features are subsampled recursively at each node. This recursive, data-dependent feature subsampling makes feature-bagging analysis more challenging than instance-bagging analysis. To isolate the feature-subsampling effect, we focus on FI and first present the scalar \(\mathbb{R}\) guarantee before giving its Hilbert-space generalization.

\begin{proposition}[FI of feature bagging]\label{prop:stability-feature}
    Assume the base algorithm $\cA_0$ produces predictors satisfying
    $f(\bx)\in\cY\subseteq[-M,M]$ for every $\bx\in\cX$, and let $q<1$.
    Then the infinitely feature-bagged algorithm ${\cA}_{\infty}$ based on
    $\cA_0$ is $\phi$-feature stable whenever
    \begin{align*}
        \phi^2 \geq \frac{M^2}{d-1}\frac{q}{1-q}.
    \end{align*}
\end{proposition}

Proposition~\ref{prop:stability-feature} is the feature-wise dual of the
infinite instance-bagging guarantee of \citet{soloff2024bagging}. Under the
same bounded-output assumption, their result states that infinite instance
bagging with subsampling ratio \(p<1\) is \(\phi\)-instance stable whenever
\(\phi^2 \ge \frac{M^2}{n-1}\frac{p}{1-p}\). Thus, their guarantee controls II
under instance subsampling, whereas Proposition~\ref{prop:stability-feature}
controls FI under feature subsampling; formally, \((n,p,\mathrm{II})\) is
replaced by \((d,q,\mathrm{FI})\). This feature-side guarantee is the scalar
starting point for the general-space and recursive results below.

We now lift the scalar feature-bagging result to general output spaces. Consider algorithms of the form $\bw=\cA(\cD,\xi)$ whose outputs lie in a Hilbert space $\cH$ with inner product $\langle\cdot,\cdot\rangle$ and induced norm $\|\cdot\|_{\cH}$. We assume that the output belongs to a convex, bounded set $\cW\subset\cH$ and that the zero element of $\cH$ lies in $\cW$. The associated feature instability over $\cW$ is measured using $\|\cdot\|_{\cH}$.

\begin{definition}\label{def:stability-edit-metric}
    An algorithm $\cA$ with output in $\cW$ is $\phi$-feature stable if, for all datasets $\cD=\left\{(\bx_i, y_i)\right\}_{i=1}^n$,
\begin{align*}
\mathrm{FI}_{\cH}:=\frac{1}{d} \sum_{j = 1}^d
\mathbb{E}\left[\|\bw- \bw^{-j}\|_{\cH}^2\right] \leq \phi^2,
\end{align*}
where $\bw=\cA(\cD ; \xi)$ and $\bw^{-j}=\cA(\cD^{:,-j} ; \xi)$.
\end{definition}

The bagged estimator is defined analogously. Let
$\bw^{(b)}=\cA(\cD,\xi^{(b)})=\cA_0(\cD^{:,\bnu^{(b)}})$ be the output of the
$b$-th feature subsample. Since $\cW$ is convex,
$\bw^B=B^{-1}\sum_{b=1}^B\bw^{(b)}$ also lies in $\cW$. The infinite-bagged
estimator is $\bw^\infty:=\mathbb{E}[\cA(\cD,\xi)]$. This framework
extends the scalar case $\cW\subseteq\mathbb{R}$ above. We use the following
Hilbert-space analogue of the bounded-range assumption of
\citet{soloff2024stability}.

\begin{assumption}\label{asp:hilbert}
Let $\cW$ be a closed and convex subset of a Hilbert space with norm
$\|\cdot\|_{\cH}$. We assume that its radius
$
\operatorname{rad}(\cW)
:=\inf_{\bw \in \cW} \sup_{\bw^{\prime} \in \cW}
\left\|\bw-\bw^{\prime}\right\|_{\mathcal{H}}
$
is finite.
\end{assumption}

When $\cH=\mathbb{R}$, $\|\cdot\|_{\cH}=|\cdot|$, and
$\cW=[-M,M]$, Assumption~\ref{asp:hilbert} recovers the bounded-range condition
used in Proposition~\ref{prop:stability-feature}.

\begin{proposition}[General-space FI of feature bagging]\label{prop:stability-edit-norm}
Assume Assumption~\ref{asp:hilbert} holds, and let \(q<1\). For any base
algorithm $\cA_0$ with output in $\cW$, infinite feature bagging
${\cA}_{\infty}$ is $\phi$-feature stable whenever
\begin{align*}
 \phi^2 \geq  \frac{\operatorname{rad}^2(\cW)}{d-1}\frac{q}{1-q}.
\end{align*}
\end{proposition}

Proposition~\ref{prop:stability-edit-norm} generalizes
Proposition~\ref{prop:stability-feature} by replacing the scalar range bound
\(M\) with the Hilbert-space radius \(\operatorname{rad}(\cW)\). This
general-space form is the key input for the recursive procedures analyzed next,
including random forward selection and random forests.


\subsection{Recursive Subsampling}\label{sec:recursivesubsampling}




The previous subsection studies feature bagging where each base learner only subsamples features once. Recursive algorithms are more complex: each base learner performs multiple rounds of feature subsampling, and the subsampled features at each round are data-dependent. This recursive structure makes the analysis more challenging because the effect from one feature removal can propagate in later updates. We model this propagation through a stochastic
process \(\{\bw_t\}_{t\ge0}\), where each state \(\bw_t\) lies in a set
\(\cW_t\subseteq\cH\) and \(\cH\) is equipped with the norm
\(\|\cdot\|_{\mathcal H}\). At iteration \(t\), the update is
\[
\bw_t = \cA\!\left(\cD,(\bw_{t-1},\xi_t)\right),
\]
where \(\xi_t\) denotes the randomness injected at step \(t\). The admissible
set \(\cW_t\) may depend on \(\bw_{t-1}\), allowing the state space itself to be
data-dependent and recursive. We initialize the process at \(\bw_0\), the zero
element of \(\cH\), and assume that \(\bw_0\in\cW_t\) for all \(t\).

For a fixed dataset, define the update maps
\[
S_{\xi}(\bw) := \cA(\cD,(\bw,\xi)),
\quad
S_{\xi}^{-j}(\bw) := \cA(\cD,(\bw,\xi,j)).
\]
Here the extra argument \(j\) means that the same seed \(\xi\) first generates
the candidate feature set as in the original run, and then feature \(j\) is
removed from that set if it appears.
Using the same randomness \(\xi_{1:T}\), the original trajectory is
\begin{align}
\cA^{(T)}(\cD,\xi_{1:T})
= \bw_T
= S_{\xi_T}\circ\cdots\circ S_{\xi_1}(\bw_0),
\end{align}
whereas the trajectory after removing feature \(j\) is
\begin{align}
\cA^{(T)}(\cD,(\xi_{1:T}, j))
= \bw_T^{-j}
= S_{\xi_T}^{-j}\circ\cdots\circ S_{\xi_1}^{-j}(\bw_0).
\end{align}
Independent copies of this recursive trajectory give the finite- and
infinite-bagged outputs
\begin{align*}
\cA^{(T)}_B(\cD)=
 \frac{1}{B}\sum_{b=1}^B \bw_T^{(b)}, \quad
\cA^{(T)}_{\infty}(\cD)
= \EE_{\xi_{1:T}}[\bw_T],
\end{align*}
with their feature-removed counterparts denoted by a superscript $-j$.
This framework covers random forward selection
\citep{mentch2020randomization} and random forests
\citep{breiman2001random}, both described in the following two subsections.

To state the recursive guarantee, let
\(\Delta_t^{-j}:=\bw_t-\bw_t^{-j}\). We isolate the part of the step-\(t\)
perturbation that comes from propagating the previous feature-removal
discrepancy:
\(\operatorname{pith}_t^j
:=\EE[S_{\xi_t}^{-j}(\bw_{t-1})-S_{\xi_t}^{-j}(\bw_{t-1}^{-j})]\).
The following one-step condition requires this propagated discrepancy to grow
by at most a factor \(1+\delta_t\): for each step \(t\), there exists
\(\delta_t\ge0\) such that, for every \(j\in[d]\),
\begin{align}
\left\|\operatorname{pith}_t^j\right\|_{\cH}
\le
(1+\delta_t)\left\|\EE\left[\Delta_{t-1}^{-j}\right]\right\|_{\cH}.
\label{equ:contraction}
\end{align}
This is a Lipschitz-type bound on the expected transition map, analogous to
stepwise stability assumptions for iterative algorithms
\citep{hardt2016train, lei2020fine}.


\begin{proposition}\label{prop:stabilityofgeneralrecursiveobject}
Let Assumption~\ref{asp:hilbert} hold.  Suppose that for each recursion step
\(t\), the one-step condition~\eqref{equ:contraction} holds for every
\(j\in[d]\).  Then
\begin{align}\label{equ:recursivestabilityresult}
  \frac{1}{d} \sum_{j=1}^d
  \left\|\EE \left[\bw_T\right]
  - \EE \left[\bw^{-j}_T\right]\right\|_{\cH}^2
   \leq    \frac{M_T q}{(d-1)(1-q)},
\end{align}
where
\[
M_T :=
\left(
\sum_{t=1}^{T}
\left(\prod_{t'=t+1}^{T}(1+\delta_{t'})\right)
\sup_{\xi_{1:(t-1)}} \operatorname{rad}(\cW_t)
\right)^2 .
\]
\end{proposition}

The bound preserves the same feature-subsampling factor
\(q/((d-1)(1-q))\) as the one-step\footnote{Here one step means that each base learner only subsamples features once. } result in Proposition~\ref{prop:stability-edit-norm}. The new factor \(M_T\) is the price
of recursion: it accumulates the radii of the intermediate state spaces and the
one-step inflation factors. In the simple case
\(\operatorname{rad}(\cW_t)\equiv r\) and \(\delta_t\equiv\delta\), we obtain
\(M_T=O(r^2T^2)\) when \(\delta=0\), and
\(M_T=O(r^2(1+\delta)^{2T})\) for fixed \(\delta>0\).


Proposition~\ref{prop:stabilityofgeneralrecursiveobject} is stated for
infinite bagging to isolate the mechanism that improves feature stability.
Section~\ref{sec:finite-bagging} shows that using \(B\) finite bagging rounds
adds only an \(O(1/B)\) concentration error to the infinite-bagging instability
term. Thus, for moderately large \(B\), finite bagging inherits the same
stability behavior up to this vanishing error.


\begin{figure}[!t]
\vskip -0.1in
\centering
\subfigure[$\frac{1}{d}\sum_{j=1}^d\|\bw - \bw^{-j}\|_2^2$]{
\begin{minipage}{0.35\linewidth}
\centering
\includegraphics[width=\textwidth]{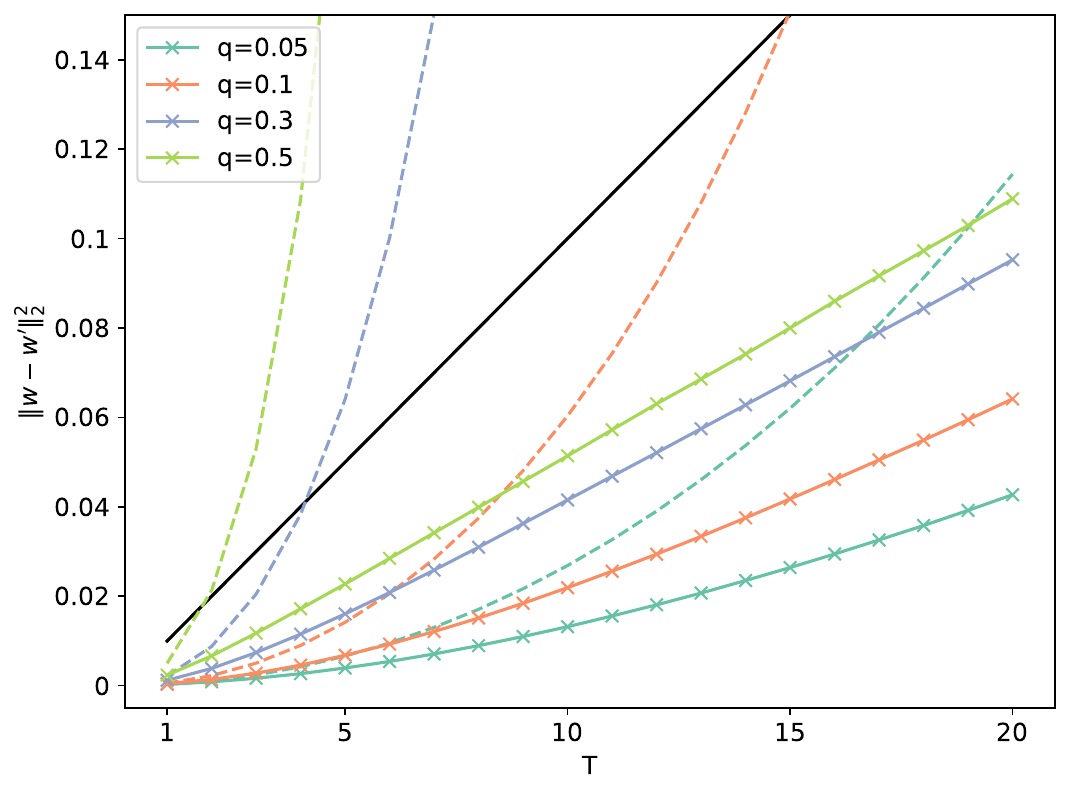}
\end{minipage}
\label{fig:forward_selection_feature_stability_theory}
}
\subfigure[$\frac{1}{d}\sum_{j=1}^d\|\bbeta - \bbeta^{:,-j}\|_2^2$]{
\begin{minipage}{0.35\linewidth}
\centering
\includegraphics[width=\textwidth]{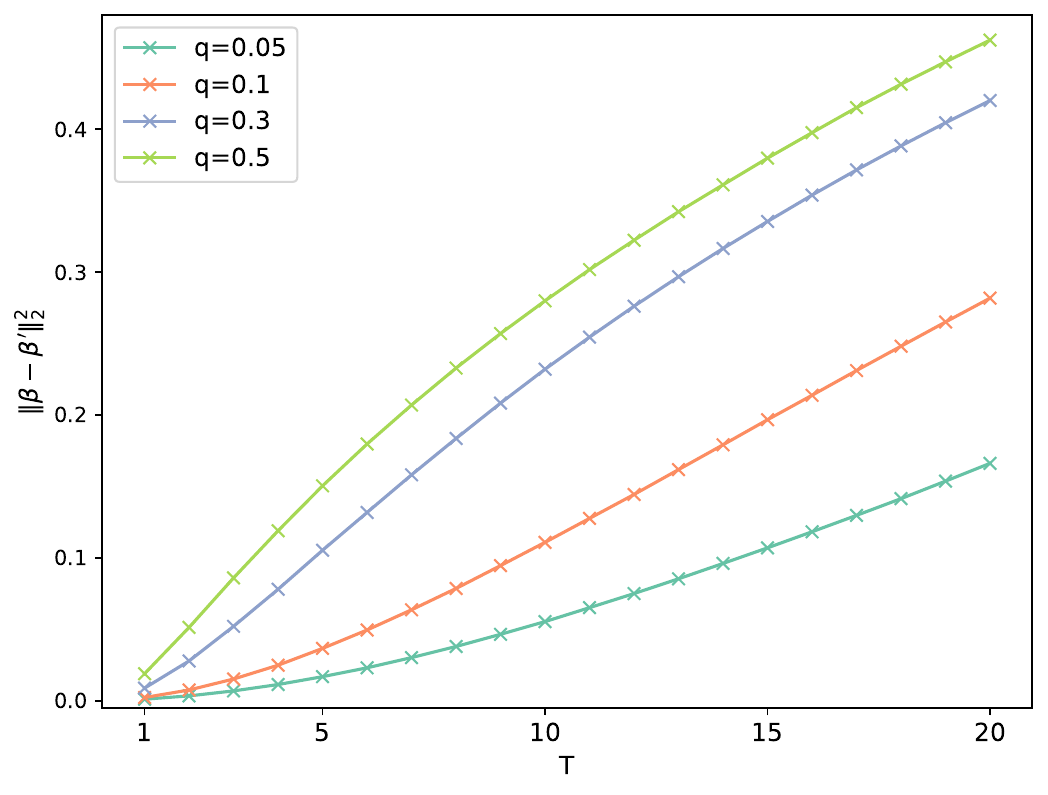}
\end{minipage}
\label{fig:forward_selection_feature_stability_l2}
}
\vskip -0.05in
\caption{
Left: Feature instability of feature-bagged random forward selection and non-bagged baseline (forward selection).  Solid curves with crosses show empirical averages over 100
repetitions for different \(q\), dashed curves show the upper bound
in~\eqref{equ:stability-of-rfs-final-bound}, and the solid black curve gives
the non-bagged baseline. Right: Parameter estimation difference of feature-bagged random forward selection for different \(q\).
}
\vskip -0.1in
\label{fig:forward_selection_feature_stability}
\end{figure}

\subsubsection{Random Forward Selection}\label{sec:rfs}

We first apply the recursive bound to random forward selection. The algorithm
builds a feature set one coordinate at a time: at each step, it samples a
fraction \(q\) of the available features and adds the sampled feature that gives
the largest reduction in residual error. This places random forward selection
within the class of data-dependent recursion algorithms covered by
Proposition~\ref{prop:stabilityofgeneralrecursiveobject}.

To represent the recursive state, encode the selected variables by
\(\bw=(w^1,\ldots,w^d)\), where \(w^j\in\{0,1\}\) and \(w^j=1\) means that
feature \(j\) is included in the fitted linear model. In this specialization,
\(\cW_t\subseteq\{0,1\}^d\subset\cH\), with \(\cH=\mathbb{R}^d\) and
\(\|\cdot\|_{\cH}=\|\cdot\|_2\). For feature-subsampled ordinary least squares,
\(w^j=1\) exactly when \(j\in\bnu\).
Let \(\bE_{\bw}\) denote the diagonal matrix with diagonal \(\bw\).
The \(b\)-th subsampled ordinary least-squares estimator is
\[
\bbeta^{(b)}
=
\bE_{\bw^{(b)}}\left(\bX \bE_{\bw^{(b)}}\right)^{\dagger}\by.
\]
where the inverse is understood as the Moore--Penrose pseudoinverse in the
overparameterized regime.
Let $\bbeta = \bX^{\dagger}\by$.

We consider the orthogonal-design case, where $\bX^{\top} \bX = n \bI_d$.
Then the feature-subsampled estimator reduces to
$\bbeta^{(b)} = \bE_{\bw^{(b)}}\bbeta$ and, with
$\bw_B:=B^{-1}\sum_{b=1}^B\bw^{(b)}$,
\begin{align*}
    \bbeta_{B} = \frac{1}{B}\sum_{b=1}^B \bbeta^{(b)} = \frac{1}{B}\sum_{b=1}^B \bE_{\bw^{(b)}}\bbeta =  \bE_{\bw_B}\bbeta.
\end{align*}
Thus, in the orthogonal setting, the stability of the averaged estimator is governed
by the selection-frequency vector \(\bw_B\). Under uniform feature subsampling,
each feature is selected with probability \(q\), so
\(\EE[\bw_B]=(q,\ldots,q)\). Under random forward selection, \(\bw_B\) is
data-dependent and records how frequently each feature is selected, thereby
capturing feature importance. We measure discrepancies between selection
vectors by the \(\ell_2\) distance. The following theorem bounds the resulting
feature instability.

The full randomized forward-selection procedure is shown in
Algorithm~\ref{alg:rfs}. At step \(t\) of the \(b\)-th bagging round, starting
from the current selection vector \(\bw_{t-1}^{(b)}\), the algorithm first
subsamples a fraction \(q\) of the \(d\) available features according to the
randomness \(\xi_t^{(b)}\), yielding a candidate set \(\bnu_t^{(b)}\). It then
selects the feature that yields the greatest reduction in residual sum of
squares:
\[
j_t^{(b)} = \arg\min_{j \in \bnu_t^{(b)}} \min_{\bbeta}
\left\| \by - \bX \bE_{\bw_{t-1}^{(b)} + \be_j} \bbeta \right\|_2^2,
\]
where \(\be_j\) denotes the \(j\)-th standard basis vector in \(\RR^d\) and
\(\bE_{\bw}\) is the diagonal matrix that projects onto the coordinates selected
by \(\bw\). The selected variable is then added to the active set,
\[
\cA(\bw_{t-1}^{(b)}) = \bw_t^{(b)} = \bw_{t-1}^{(b)} + \be_{j_t^{(b)}}.
\]
If all indices in the candidate set \(\bnu_t^{(b)}\) have already been selected,
we set \(j_t^{(b)}=\texttt{None}\) and \(\be_{\texttt{None}}=\zero\), so the
selection vector remains unchanged.

\begin{algorithm}[!t]
\caption{Randomized Forward Selection}
\label{alg:rfs}
\begin{algorithmic}
\FOR{$b \in [B]$}
\STATE Draw bootstrap sample $\cD^{(b)} =  \{(\bx_i^{(b)},y_i^{(b)}) \}_{i=1}^{n}$ from original data $\cD$.
\STATE Initialize empty active set $\bw_0=\zero$.
\FOR{$t \in [T]$}
\STATE Select subset of $s$ features uniformly at random, denoted $\bnu_t^{(b)}$.
\STATE Select
$ j_t \in \argmin_{j\in\bnu_t^{(b)}}
\norm{\by^{(b)}-\bX^{(b)}\bE_{\bw_{t-1}+\be_j}
(\bX^{(b)}\bE_{\bw_{t-1}+\be_j})^{\dagger}\by^{(b)}}_2^2$,
with deterministic tie-breaking; if every candidate is already selected, set
$\be_{j_t}=\zero$.
\STATE Update $\bw_t=\bw_{t-1}+\be_{j_t}$.
\STATE Update coefficient estimates by fitting least squares on the selected coordinates:
\[
\bE_{\bw_t}\hat{\bbeta}^{(b)}
= \argmin_{\beta} \norm{\by^{(b)}-\bX^{(b)}\bE_{\bw_t}\beta}_2^2,
\qquad
(\bI-\bE_{\bw_t})\hat{\bbeta}^{(b)}=\zero.
\]
\ENDFOR
\ENDFOR
\STATE Compute final coefficient estimates $\hat{\bbeta}=B^{-1}\sum_{b=1}^B\hat{\bbeta}^{(b)}$.
\STATE Compute predictions $\hat{\by}=\bX\hat{\bbeta}$.
\end{algorithmic}
\end{algorithm}

\begin{theorem}\label{thm:stabilityrandomforwardselection}
Let $\cA$ be the random forward selection algorithm in Algorithm~\ref{alg:rfs} with feature-subsampling ratio \(q<1\).
Then $\operatorname{rad}(\cW_t) \leq 1$ for $t\in [T]$.
Moreover, the one-step condition~\eqref{equ:contraction} in
Proposition~\ref{prop:stabilityofgeneralrecursiveobject} holds with
\(\delta_t = 4^{-1}\max\{q(1+q), q^2 t\}\).
When $T \lesssim q^{-1}$, we have
\begin{align}\label{equ:stability-of-rfs-final-bound}
 \frac{1}{d} \sum_{j=1}^d \left\|\EE\left[\bw_T\right]- \EE\left[\bw_T^{-j}\right]\right\|^2
\lesssim \frac{T^2}{d - 1} \cdot \frac{q}{1 - q}.
\end{align}
\end{theorem}

The theorem specializes
Proposition~\ref{prop:stabilityofgeneralrecursiveobject} to the selection path
of random forward selection. The bound in
\eqref{equ:stability-of-rfs-final-bound} gives a stronger stability guarantee
for smaller feature-subsampling ratios \(q\), showing that more aggressive
feature subsampling improves stability. Its quadratic dependence on \(T\) reflects the other side of the
recursive analysis: instability can accumulate along a long selection path.
For comparison, deterministic forward selection without bagging has feature
instability
\begin{align*}
\frac{1}{d} \sum_{j=1}^d
\left\|\EE \left[\bw_T\right]
- \EE \left[\bw_T^{-j}\right]\right\|^2
= \frac{2 T}{d},
\end{align*}
whose derivation is deferred to the appendix.
Hence, feature bagging is provably more stable whenever
\begin{align*}
\frac{2T}{d} \gtrsim \frac{T^2}{d - 1} \cdot \frac{q}{1 - q},
\quad \text{or equivalently,} \quad
T \lesssim \frac{1 - q}{q}.
\end{align*}
This is consistent with the condition \(T\lesssim q^{-1}\); for example, it is
satisfied when \(T=o(\sqrt d)\) and \(q=1/\sqrt d\).

Figure~\ref{fig:forward_selection_feature_stability} evaluates this prediction
under the experimental settings detailed in
Appendix~\ref{sec:additional-experiments}.
Figure~\ref{fig:forward_selection_feature_stability_theory} shows that feature
bagging consistently improves empirical feature stability, with smaller
subsampling ratios \(q\) giving larger stability gains.
Figure~\ref{fig:forward_selection_feature_stability_l2} shows a corresponding
decrease in parameter estimation difference, suggesting that the stability gain
does not come at the expense of estimation accuracy.

When \(qT\le1\), the theoretical curve upper bounds the averaged empirical
instability and certifies better stability than the non-bagged baseline. For small \(T\), it
also tracks the empirical curve closely, indicating that the bound captures the
early-stage behavior of random forward selection. As \(T\) grows, the gap
widens, consistent with the quadratic dependence on \(T\) and the conservative
one-step condition used in the analysis. Overall, the experiment supports the
same message as the theorem: feature bagging stabilizes random forward selection
most clearly in the shallow-recursion regime.

\subsubsection{Random Forests}
\label{sec:random-forest-main}
\label{sec:random-forest}

Random forests provide the second application.
In a full classification-and-regression tree, both the split feature and the split
threshold are data-dependent. A decision tree's partition can be encoded by two
vectors: one that records the feature chosen at each internal node and another
that stores the corresponding split threshold; see the \texttt{feature} and
\texttt{threshold} attributes in \citet{pedregosa2011scikit}. Both vectors have length
\(2^T-1\), where \(T\) is the \texttt{max\_depth} parameter, i.e., the maximum
number of splits from the root node to any leaf node.
\begin{figure}[!t]
    \centering
    \includegraphics[width=0.8\linewidth]{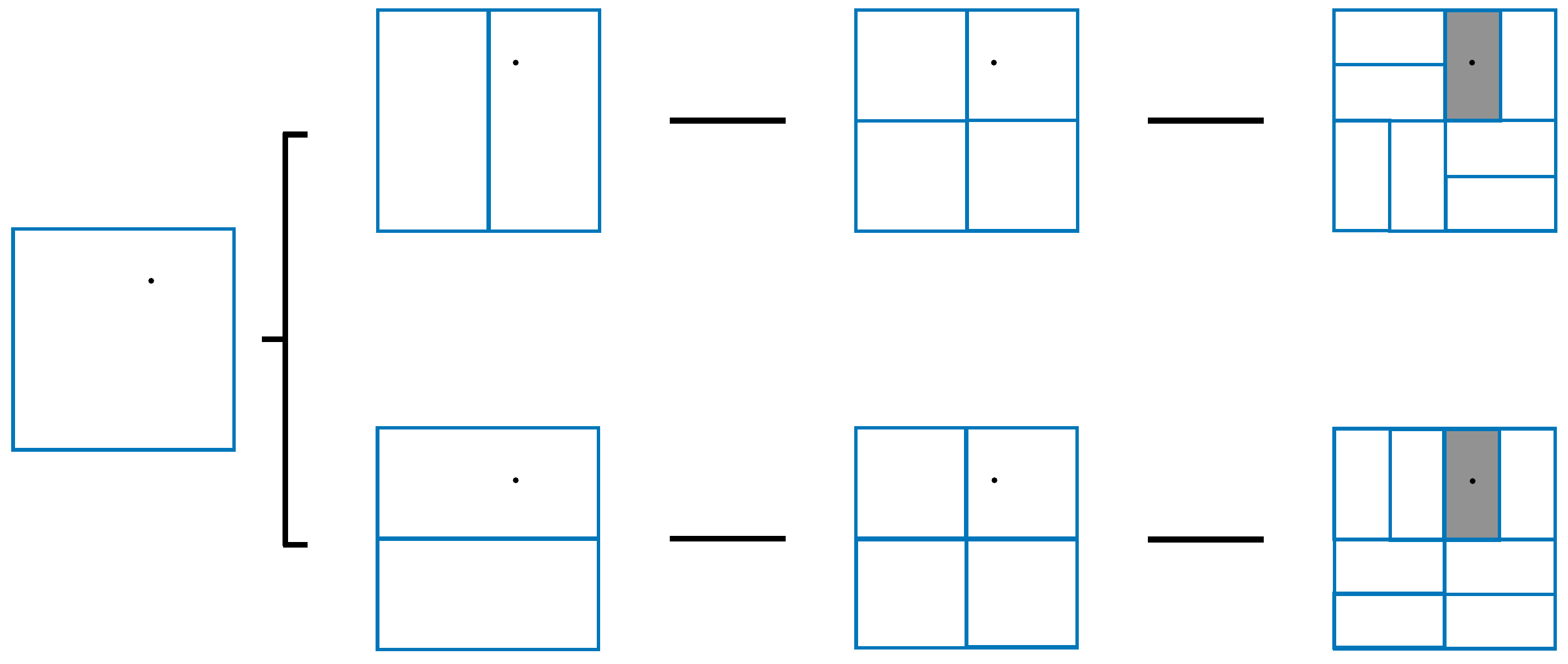}
    \caption{Two possibilities of dyadic tree partition.
    The black dot represents the target test sample $\bx$.
    The two partitions, though different in other areas, have the same representation $\bw$ that corresponds to the grey area.}
    \label{fig:partition}
\end{figure}

We analyze a simplified dyadic forest: each split is made at the midpoint of the
selected coordinate, so the local cell containing a fixed test point \(\bx\) is
determined only by the sequence of split features along the path to \(\bx\). We
further consider the max-edge version of the dyadic tree. The corresponding
local-path procedure is summarized in Algorithm~\ref{alg:dyadic-fs}. Dyadic
splitting is a simplified model of decision-tree partitions and is widely used
in nonparametric statistical estimation \citep[e.g.][]{blanchard2007optimal,
perchet2013multi, cai2023extrapolated, ma2023decision, cai2024transfer,
ma2024optimal}.

\begin{algorithm}[!t]
\caption{Max-edge dyadic randomized forest, local path at \(\bx\)}
\label{alg:dyadic-fs}
\begin{algorithmic}
\STATE Let \(\bE_{r,k}\in\RR^{T\times d}\) be the matrix with one at entry \((r,k)\) and zero elsewhere.
\FOR{$b \in [B]$}
\STATE Draw bootstrap sample $\cD^{(b)} =  \{(\bx_i^{(b)},y_i^{(b)}) \}_{i=1}^{n}$ from original data $\cD$.
\STATE Initialize local encoding $\bw_0=\zero^{T\times d}$ and local cell $L_0=[0,1]^d$.
\FOR{$t \in [T]$}
\STATE Select subset of $s$ features uniformly at random, denoted $\bnu_t^{(b)}$.
\STATE Define \(N_{t-1}^k=\|\bw_{t-1}^k\|_1\) and \(\cM(\bw_{t-1},\bnu_t^{(b)})=\argmin_{k\in\bnu_t^{(b)}}N_{t-1}^k\).
\STATE Select \(j_t\in\argmax_{k\in\cM(\bw_{t-1},\bnu_t^{(b)})}\Delta_k(\bw_{t-1};\cD^{(b)})\), with deterministic tie-breaking.
\STATE Update $\bw_t=\bw_{t-1}+\bE_{N_{t-1}^{j_t}+1,j_t}$.
\STATE Split $L_{t-1}$ along feature $j_t$ at the midpoint and set $L_t$ to the child cell containing $\bx$.
\ENDFOR
\STATE Compute the prediction at $\bx$ as the average of labels in $L_T$.
\ENDFOR
\STATE Average the $B$ predictions.
\end{algorithmic}
\end{algorithm}

Rather than analyzing the entire tree, we study a local property: the leaf node
that contains a given query point \(\bx\). This node is uniquely determined by
the sequence of features used to split the path from the root to that leaf; see
Figure~\ref{fig:partition}. To encode the node, define the feature-weight matrix
\[
\bw_t=(\bw_t^1,\ldots,\bw_t^d)\in\RR^{T\times d}.
\]
If, up to the \(t\)-th step, the leaf node that contains \(\bx\) has been split
along feature \(j\) for \(t_j\) times, then the first \(t_j\) entries of
\(\bw_t^j\) are one, while the remaining entries are zero. We write $N_t^k:=\|\bw_t^k\|_1$ and $N_t^{-j,k}:=\|(\bw_t^{-j})^k\|_1$
for the corresponding split counts in the full and feature-removed runs.

At each local split \(t\), draw a candidate feature set $\bnu_t\subset[d]$ with $|\bnu_t|=s$ and $q=s/d$
uniformly without replacement. Define the candidate-set max-edge set
\[
\cM(\bw_t,\bnu_t)
:=
\argmin_{k\in\bnu_t}N_t^k.
\]
The tree then chooses
\[
k_t
\in
\argmax_{k\in\cM(\bw_t,\bnu_t)}
\Delta_k(\bw_t;\cD),
\]
with deterministic tie-breaking. For the feature-removed run, use the coupled
candidate set $\bnu_t^{-j}:=\bnu_t\setminus\{j\}$.
If \(\bnu_t^{-j}=\varnothing\), the removed update returns the current encoding.
Otherwise define
\[
\cM^{-j}(\bw_t^{-j},\bnu_t^{-j})
:=
\argmin_{k\in\bnu_t^{-j}}N_t^{-j,k},
\]
and choose the feature in this set with the largest split decrease, using the
same deterministic tie-breaking rule. In this dyadic case, only one vector with
length \(2^T-1\) is enough to store the local tree path because the thresholds
are uniquely defined.

For instance, the \(\bx\) in Figure~\ref{fig:partition} has
\begin{align*}
 \bw_1 =  \left(  \begin{matrix}
      0 & 1 \\
      0 & 0 \\
      0 & 0
    \end{matrix}\right) \text{ or }
    \left(  \begin{matrix}
      1 & 0 \\
      0 & 0 \\
      0 & 0
    \end{matrix}\right),
    \qquad
\bw_2 =  \left(  \begin{matrix}
      1 & 1 \\
      0 & 0 \\
      0 & 0
    \end{matrix}\right),
    \qquad
\bw_3 =  \left(  \begin{matrix}
      1 & 1 \\
      0 & 1 \\
      0 & 0
    \end{matrix}\right).
\end{align*}
Moreover, given a fixed \(\bw_t\), the grey region in
Figure~\ref{fig:partition} is the unique leaf node compatible with that
encoding. Denote this region by \(L_{\bw_t}\).

We now justify the choice of the encoding \(\bw\). First, the encoding scheme is
closely related to \texttt{feature\_importances\_} \citep{pedregosa2011scikit}, which
measures a feature's global importance by the proportion of splits that use that
feature. Analogously, for a specific test point \(\bx\), we can quantify the
local importance of feature \(j\) by
\begin{align*}
    \frac{\|\bw_T^j\|_1}{\sum_{j'=1}^d \|\bw_T^{j'}\|_1}.
\end{align*}
This point-specific view is aligned with recent work on individual variable
importance, which moves beyond population-level summaries \citep{dai2025moving}.
We also demonstrate a direct relationship between the prediction at \(\bx\) and
the matrix \(\bw\) through the following proposition. Consider data generated
from the additive model
\begin{align}
    \EE_{y\mid \bx}[y]
    =
    f^{*}(\bx)
    =
    \sum_{j=1}^d f_j^{*}(\bx^j),
    \qquad
    \bx \sim \mathrm{Unif}([0,1]^d).
\end{align}

\begin{proposition}\label{prop:linear-map-tree}
Let \(L_{\bw}\) be the rectangle of the node associated with \(\bx\) under
\(\bw\), and let \(f^{\bw}\) be the decision-tree predictor associated with
\(\bw\). Since we only care about \(\bx\), we view all decision trees with the
same \(\bw\) as equivalent. Based on the above definitions, there exists a map
\(\cK\) such that
\begin{align*}
    \cK(\bw)
    =
    \EE_{\by\mid\bX,\bw}
    \left[
        f^{\bw}(\bx)
    \right].
\end{align*}
Moreover, the map \(\cK\) is affine: there exist a constant \(c_{\bx}\) and a
matrix \(\Theta_{\bx}\in\RR^{T\times d}\) such that
\begin{align*}
    \cK(\bw)
    =
    c_{\bx}
    +
    \langle \Theta_{\bx},\bw\rangle .
\end{align*}
Consequently,
\begin{align*}
    \cK(\bw_1)-\cK(\bw_2)
    =
    \langle \Theta_{\bx},\bw_1-\bw_2\rangle
    \lesssim
    \|\bw_1-\bw_2\|_F .
\end{align*}
\end{proposition}

Given Proposition~\ref{prop:linear-map-tree}, the random forest predictor is
fully determined by the ensemble feature-weight matrix \(\EE_{\xi}[\bw]\):
\begin{align*}
   \EE_{\by\mid\bX}
   \left[
        f(\bx)
   \right]
   =
   \EE_{\xi}
   \left[
        \EE_{\by\mid\bX,\bw}
        \left[
            f^{\bw}(\bx)
        \right]
   \right]
   =
   \EE_{\xi}
   \left[
        \cK(\bw)
   \right]
   =
   \cK\left(\EE_{\xi}[\bw]\right),
\end{align*}
where the last equality follows from the affine representation of \(\cK\).
Moreover, for the feature-removed forest, the affine representation gives
\begin{align*}
\left(
    \EE_{\by\mid\bX}
    \left[
        f(\bx)-f^{-j}(\bx)
    \right]
\right)^2
&
=
\left(
    \cK(\EE_{\xi}[\bw_T])
    -
    \cK(\EE_{\xi}[\bw_T^{-j}])
\right)^2  \\
&
=
\left\langle
    \Theta_{\bx},
    \EE_{\xi}[\bw_T]-\EE_{\xi}[\bw_T^{-j}]
\right\rangle^2  
\lesssim
\left\|
    \EE_{\xi}[\bw_T]-\EE_{\xi}[\bw_T^{-j}]
\right\|_F^2 .
\end{align*}
Thus, it suffices to analyze \(\EE_{\xi}[\bw_T]\) with respect to the Frobenius
norm.

We next give the instability bound of the encoding.

\begin{theorem}
\label{thm:stabilityrandomforest}
\label{thm:max-edge}
Let $\cA$ be the max-edge dyadic random forest algorithm described in
Section~\ref{sec:random-forest-main}. Then $\operatorname{rad}(\cW_t)\leq 1$, and
the one-step condition~\eqref{equ:contraction} holds with
\(\delta_t=\sqrt{2}\).
Let \(A_T = \sum_{t=1}^{T} (1+\sqrt{2})^{T-t}\). Then
\begin{align}\label{equ:stability-of-rf-final-bound}
\frac{1}{d}
 \sum_{j=1}^d
 \left\| \EE\left[\bw_T\right] - \EE\left[\bw_T^{-j}\right]
\right\|_{\rm F}^2
 \leq
 \frac{q \cdot A_T^2}{(d-1)(1-q)}. 
\end{align}
\end{theorem}

Theorem~\ref{thm:stabilityrandomforest} shows that recursive feature bagging
stabilizes the local partition induced by a dyadic tree. Its dependence on the
feature-subsampling ratio matches the one-step bound in
Proposition~\ref{prop:stability-edit-norm}: smaller \(q\) yields a stronger
feature-stability guarantee. The factor \(A_T^2\) is the price of recursion,
capturing how perturbations accumulate along the tree path. For comparison, a
non-bagged depth-\(T\) dyadic tree has feature instability \(2T/d\), as derived
in Appendix~\ref{app:non-bagged-stability}. Thus, for shallow trees and
sufficiently small \(q\), feature bagging provides a strict stability
improvement.

\subsection{Finite Bagging}
\label{sec:finite-bagging}

The preceding model-free guarantees focus on infinite bagging, whereas practical
ensembles use a finite number \(B\) of bagging rounds. We now show that the
infinite-bagging bounds remain representative when \(B\) is moderately large.

\begin{proposition}[Feature instability of finite bagging]\label{prop:stability-set-finite}
Let Assumption~\ref{asp:hilbert} hold. Then for a bagged algorithm \(\cA\) with
\(B\) bagging rounds and outputs \(\bw^{(b)}\in\cW\), \(b\in[B]\), there holds
\begin{align*}
  \frac{1}{d} \sum_{j=1}^d
  \left\|
  \frac{1}{B}\sum_{b=1}^B\bw^{(b)}
  -
  \frac{1}{B}\sum_{b=1}^B\bw^{(b),-j}
  \right\|^2
  \leq
  \frac{\sqrt{3}}{d} \sum_{j=1}^d
  \left\|\EE_{\xi}[\bw]-\EE_{\xi}[\bw^{-j}]\right\|^2
  +
  \frac{6\sqrt{3}\operatorname{rad}(\cW)^2
  \log\left(\frac{d+1}{\delta}\right)}{B}
\end{align*}
with probability \(1-\delta\).
\end{proposition}

The proof of Proposition~\ref{prop:stability-set-finite} is given in
Appendix~\ref{app:finite-bagging}; it is a direct application of Hoeffding's
inequality and the union bound. The finite-bagging feature instability exceeds
the infinite-bagging feature instability by a concentration term that decays
linearly with the number of bagging rounds \(B\). On the constant, the radius of
the overall algorithm \(\cA\) is no larger than the sum of the suprema of
\(\operatorname{rad}(\cW_t)\). Combining this finite-\(B\) concentration term
with Proposition~\ref{prop:stabilityofgeneralrecursiveobject}, it suffices to
set
\begin{align}\label{equ:guarantee-B-feature}
    B \gtrsim d \cdot \frac{1-q}{q}.
\end{align}
In Figure~\ref{fig:feature_stability_B_max_features}, empirical feature
stability improves as \(B\) increases. For larger \(q\), fewer bagging rounds
are needed to reach the theoretical stability regime predicted by
\eqref{equ:guarantee-B-feature}.

As observed by \citet{soloff2024bagging}, finite instance bagging is also
harmless when \(B\gtrsim n(1-p)/p\). Though the results are dual, the feature
side can be much more practical when \(d\ll n\). For instance, a typical tabular
scale \((n,d)\sim(10^5,10^2)\) leads to a requirement of order \(10^2\) bagging
rounds for feature stability, as opposed to order \(10^5\) rounds for instance
stability.

\begin{figure}[!t]
\centering
\subfigure[RFS, $q = 0.05$]{
\begin{minipage}{0.4\linewidth}
\centering
\includegraphics[width=\textwidth]{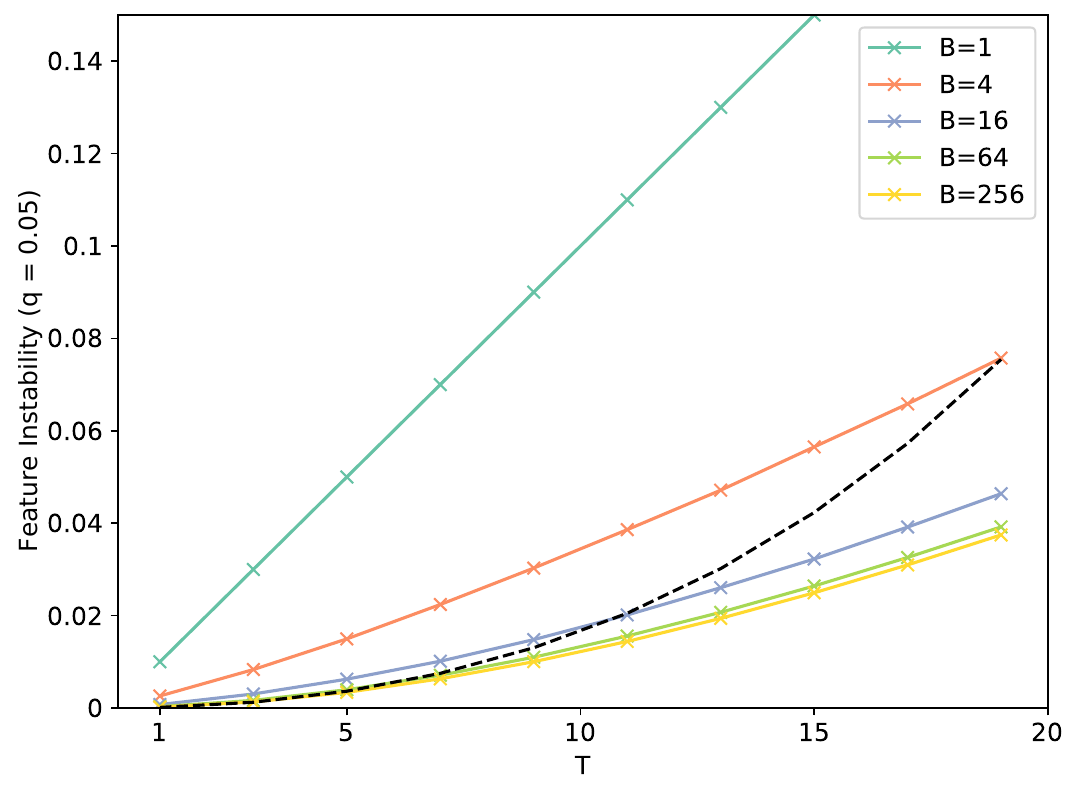}
\end{minipage}
\label{fig:forward_selection_feature_stability_B_max_features_0.05}
}
\subfigure[RFS, $q = 0.1$]{
\begin{minipage}{0.4\linewidth}
\centering
\includegraphics[width=\textwidth]{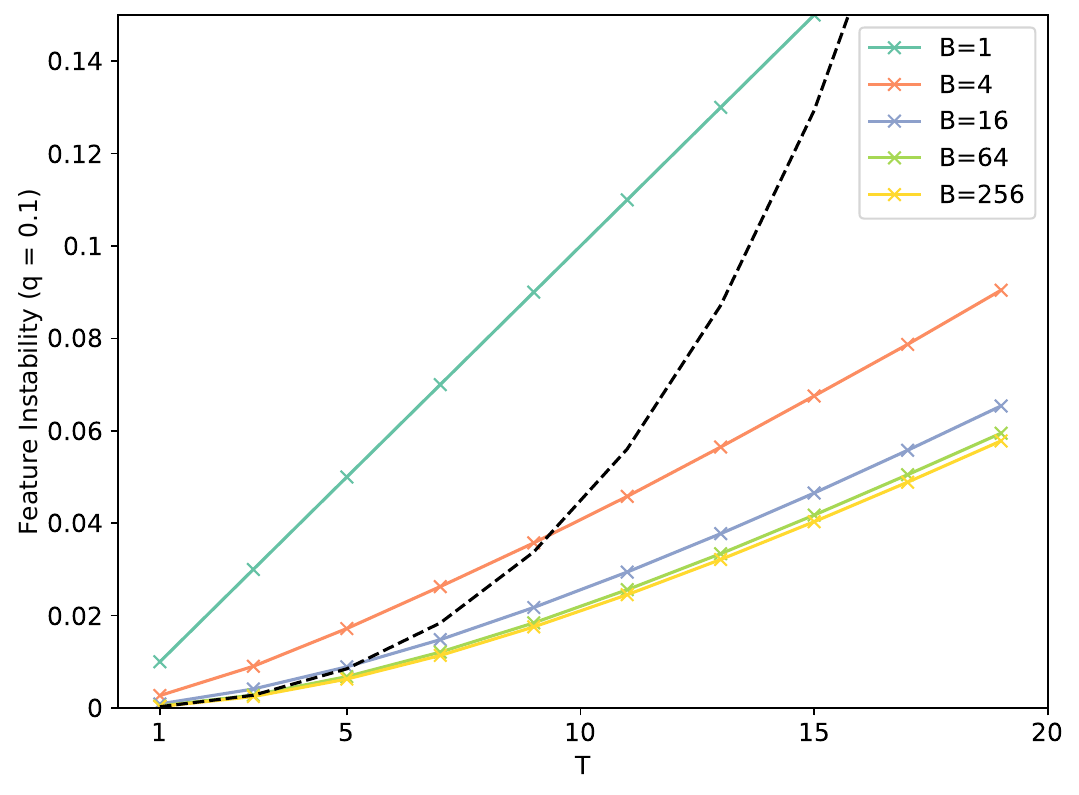}
\end{minipage}
\label{fig:forward_selection_feature_stability_B_max_features_0.1}
}
\subfigure[RF, $q = 0.1$]{
\begin{minipage}{0.4\linewidth}
\centering
\includegraphics[width=\textwidth]{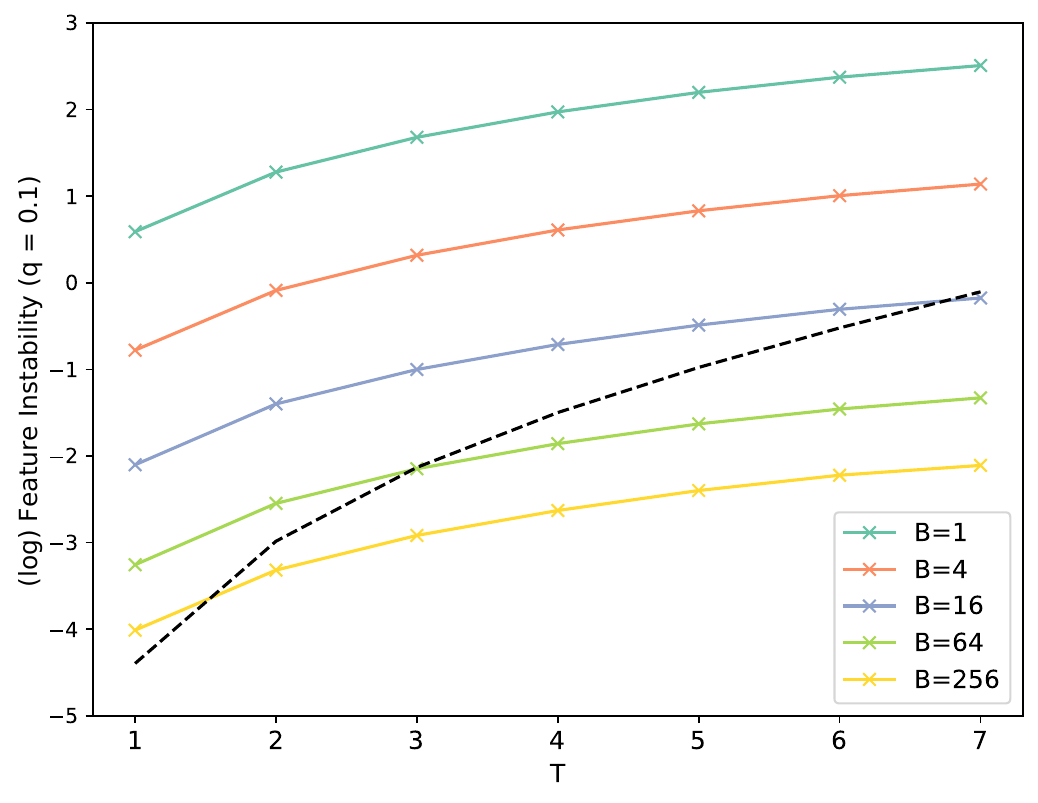}
\end{minipage}
\label{fig:random_forest_feature_stability_B_max_features_0.1}
}
\subfigure[RF, $q = 0.2$]{
\begin{minipage}{0.4\linewidth}
\centering
\includegraphics[width=\textwidth]{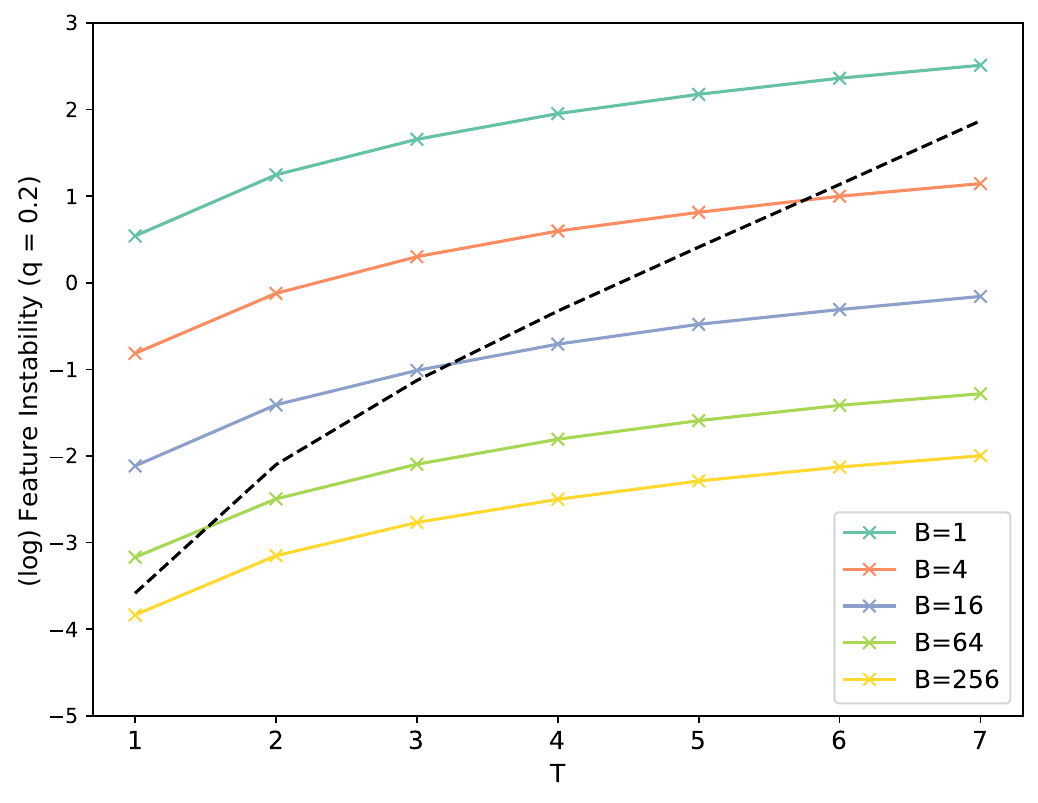}
\end{minipage}
\label{fig:random_forest_feature_stability_B_max_features_0.2}
}
\caption{
Comparison of feature instability across different numbers of bagging rounds \(B\).
Each cross corresponds to 100 repetitions. The RFS and RF simulation settings
are the same as those described in Appendix~\ref{sec:additional-experiments}.
The black lines are instability bounds from
Theorems~\ref{thm:stabilityrandomforwardselection} and
\ref{thm:stabilityrandomforest}.
}
\label{fig:feature_stability_B_max_features}
\end{figure}

\section{Conclusion}

We showed that feature bagging provides a principled route to algorithmic stabilization. By introducing FI as the feature-side counterpart to II, we made it possible to study how bagging improves stability along both the instance and feature axes. In linear regression, we derived sharp asymptotic characterizations of II and FI; in a model-free setting inspired by recursive feature subsampling in random forests, we obtained tight upper bounds for FI in early steps. Together, these results show that feature bagging consistently improves stability, that more aggressive feature subsampling can yield stronger stabilization, and that a modest number of bagging rounds is enough to approach the infinite-bagging stability level.

\section*{Acknowledgments}

We thank the anonymous reviewers for their constructive comments and suggestions, which have helped improve the clarity and presentation of this work. The bulk of the work was carried out while YM was a visiting student at MBZUAI and University of Toronto. QS was partially supported by NSERC Grant RGPIN-2026-06888, Compute Canada, and MBZUAI.

\section*{Impact Statement}

This paper presents work whose goal is to advance the field of Machine
Learning. There are many potential societal consequences of our work, none
which we feel must be specifically highlighted here.

\bibliographystyle{apalike}
\bibliography{ref}

\newpage
\stopcontents[main]
\appendix
\renewcommand{\theequation}{S.\arabic{equation}}
\renewcommand{\thetable}{S.\arabic{table}}
\renewcommand{\thefigure}{S.\arabic{figure}}
\renewcommand{\thesection}{S.\arabic{section}}
\newcommand{\thelemma}{S.\arabic{lemma}}

\vspace{30pt}
\noindent{\bf \LARGE Appendix}

\startcontents[sections]
\section*{\contentsname}
\printcontents[sections]{}{1}{}

\section{Experiment Details of PMI between Feature Instability and Generalization}
\label{app:pmi_fs_generalization}

This appendix provides full experimental details and quantitative analyses supporting
Section~\ref{sec:exp_for_feature_stability_meaning}.

\subsection{Data Generating Process}

We adopt the MARSadd setting from \citet{friedman1991multivariate},
a standard benchmark for evaluating random forests
\citep{mentch2020randomization,curth2024random}.
For each experiment, we independently sample
$X \in [0,1]^{n \times d}$ from the uniform distribution and generate responses according to
\begin{align*}
f(x) = (1-\nu) f_{\text{weak}}(x) + \nu f_{\text{strong}}(x),
\end{align*}
where
\[
f_{\text{weak}}(x) = \frac{1}{\sqrt{d-4}} \sum_{j=5}^d x_j
\]
represents a linear but low signal-to-noise component, in which the
signal is distributed across many features.
Let
\[
f_{\text{strong}}(x)
= 0.1 e^{4x_1}
+ \frac{4}{1+e^{-20x_2+10}}
+ 3x_3 + 2x_4
\]
represents a concentrated but nonlinear signal.
Independent Gaussian noise $\epsilon \sim \mathcal{N}(0,\sigma^2)$ is added to form
$y = f(x) + \epsilon$.
An independent test set of equal size is generated using the same procedure.
The mixing parameter $\nu \in [0,1]$ controls the dominance of nonlinear structure.

\subsection{Model Training and Instability Estimation}

The random forest models were trained over a predefined hyperparameter grid
to systematically explore the effects of model complexity and ensemble
diversity.
The number of estimators was fixed at $n_{\text{estimators}}=256$.
To control tree diversity, the maximum feature ratio (\texttt{max\_features})
and the maximum sample ratio (\texttt{max\_samples}) were varied over
$\{0.1, 0.3, 0.5, 0.7\}$.
Model complexity was adjusted by varying the maximum tree depth
(\texttt{max\_depth}) from 1 to 7, while the minimum number of samples per leaf
(\texttt{min\_samples\_leaf}) was fixed at 1.

This design enables a systematic investigation of model complexity along both
the feature-subsampling and depth dimensions.
In particular, \texttt{max\_features} and \texttt{max\_samples} control the
diversity of individual trees, whereas \texttt{max\_depth} governs the
bias-variance trade-off of the ensemble.
Fixing the number of estimators improves comparability across runs.

For each configuration of hyperparameters  $\{$max\_features, max\_samples, max\_depth$\}$, we train a RandomForestRegressor on the training data and compute both training and test mean squared errors (MSE).
Both instance and feature instability metrics are then measured as in \eqref{equ:informal-instance-stability} and \eqref{equ:informal-feature-stability}, by retraining the model on the data with the removed feature or instance.

All configurations are repeated 500 times with different random seeds, ensuring robust Monte Carlo estimates.
For each unique hyperparameter configuration, we compute the mean instability scores, mean squared errors, and define the generalization gap as
\begin{align*}
\text { Gap }=\text{MSE}_{\text{test}}-\text{MSE}_{\text{train}}.
\end{align*}
Results are filtered for fixed conditions ($n=500, d=10, \sigma=1$).
The MSE$_{\text{test}}$ is computed using additional independent test samples, and MSE$_{\text{train }}$ is the in-sample training error.

\subsection{Statistical Dependence Measures}

We then investigate the statistical dependence between instability measures and the generalization gap. Three types of associations are evaluated:
\begin{itemize}
\item Linear dependence:
$R^2$ values from linear regression models of the form
Gap $\sim S_{\text {feature },}$
Gap $\sim S_{\text {instance }}$, and
Gap $\sim\left(S_{\text {feature }}, S_{\text {instance }}\right)$.
\item Mutual information (MI):
Nonlinear associations are estimated using the mutual\_info\_regression function from scikit-learn.
\item Partial mutual information (PMI):
To isolate the unique contribution of each instability measure, we apply a residual-based approach using random forest regressions:
we regress out the conditioning variable (e.g., \(S_{\text{instance}}\)) from both the predictor and target, and compute the mutual information between the residuals via the MINE estimator.
This yields estimates of \(\operatorname{PMI}(S_{\text{feature}}; \mathrm{Gap} \mid S_{\text{instance}})\) and \(\operatorname{PMI}(S_{\text{instance}}; \mathrm{Gap} \mid S_{\text{feature}})\).
\end{itemize}

\subsection{More Explanation about Partial Mutual Information}

We provide an overview of partial mutual information used in the experiments.

Mutual information (MI) quantifies the amount of shared information between two random variables $X$ and $Y$.
The partial mutual information (PMI) aims to measure the direct association between $X$ and $Y$ while removing the indirect influence of $Z$, typically by regressing out $Z$ from both $X$ and $Y$ and then computing the mutual information between the residuals. Formally, if $\tilde{X}$ and $\tilde{Y}$ denote the residuals obtained from regressing $X$ and $Y$ on $Z$, respectively, then
\begin{equation}
    \mathrm{PMI}(X; Y \mid Z) = I(\tilde{X}; \tilde{Y}).
\end{equation}
Thus, PMI can be viewed as a nonparametric extension of the partial correlation concept to information-theoretic measures.
Consider three Gaussian random variables satisfying
\begin{align}
    X &= g_1(Z) + \epsilon_X, \\
    Y &= g_2(Z) + \epsilon_Y,
\end{align}
where $\epsilon_X$ and $\epsilon_Y$ are independent Gaussian noises and $Z$ is a common latent variable. In this case, $X$ and $Y$ are dependent due to their shared dependence on $Z$, yet they contain no direct interaction.
If we compute the partial mutual information $\mathrm{PMI}(X; Y \mid Z)$, it is zero because, after regressing out $Z$, the residuals $\tilde{X}$ and $\tilde{Y}$ are independent.
However, if we slightly modify the model by introducing a direct coupling:
\begin{equation}
    Y = g_1(Z) + g_2(X) + \epsilon_Y,
\end{equation}
then $\mathrm{PMI}(X; Y \mid Z) > 0$, indicating a direct dependence between $X$ and $Y$.
If we treat the generalization gap as \(Y\), and take \((X,Z)\) to be either \((\mathrm{FI},\mathrm{II})\) or \((\mathrm{II},\mathrm{FI})\), then both PMIs should be positive when both measures contain generalization-relevant information.

\subsection{Analysis}

We investigate how FI and II jointly and individually explain variation in the generalization gap \(Y\), under two configurations of the hyperparameter \(\nu\) (\(\nu = 1\) and \(\nu = 0\)). We report both deterministic correlations (via \(R^2\)) and information-theoretic dependencies (via mutual information and partial mutual information).

\paragraph{Case 1: $\nu = 1$.}

When \(\nu = 1\), the individual explanatory powers of II and FI are moderate: \(R^2(\text{II} \rightarrow Y) = 0.143\) and \(R^2(\text{FI} \rightarrow Y) = 0.120\). However, their joint contribution increases substantially, achieving \(R^2(\text{FI + II} \rightarrow Y) = 0.345\). This suggests a complementary effect between II and FI: neither measure alone suffices to capture the variability of the generalization gap, but together they explain a larger fraction of variance.

The mutual information (MI) values support this observation. The joint dependency between \((\text{FI}, \text{II})\) and \(Y\) is substantial (\(I((\text{FI}, \text{II}); Y) = 1.345\)), with asymmetric contributions: \(\text{PMI}(\text{FI}; Y \mid \text{II}) = 0.469\), \(\text{PMI}(\text{II}; Y \mid \text{FI}) = 0.439\), and mutual dependence between FI and II (\(I(\text{FI}; \text{II}) = 0.558\)). These results indicate that both measures convey overlapping but not redundant information about generalization. FI contributes slightly more conditional information about the generalization gap once II is known.

\paragraph{Case 2: $\nu = 0$.}

At \(\nu = 0\), the relationships become markedly stronger. \(R^2(\text{II} \rightarrow Y) = 0.231\), \(R^2(\text{FI} \rightarrow Y) = 0.546\), and \(R^2(\text{FI + II} \rightarrow Y) = 0.675\). The dominant role of FI indicates that feature-level perturbations are the primary source of generalization-relevant variation, while II contributes additional but smaller explanatory power.

Information-theoretic results are consistent: \(I((\text{FI}, \text{II}); Y) = 1.062\) (with FI dominating), \(\text{PMI}(\text{FI}; Y \mid \text{II}) = 0.498\), \(\text{PMI}(\text{II}; Y \mid \text{FI}) = 0.402\), and \(I(\text{FI}; \text{II}) = 0.748\). Compared with the \(\nu = 1\) case, the interdependence between FI and II increases, suggesting that when regularization or noise is reduced, the two measures become more correlated. Partial mutual information values also reveal that the unique contribution of each variable to generalization decreases slightly, consistent with the growing redundancy between FI and II.

\paragraph{Interpretation}

Overall, these findings highlight two trends: (1) complementarity, where FI and II jointly enhance generalization predictability, especially when they are less correlated (as in \(\nu = 1\)), and (2) a redundancy-dominance shift, where for smaller \(\nu\), FI becomes the dominant predictor while its correlation with II strengthens, reducing the independent information each measure provides.
This suggests that different regimes of model regularization or noise change how II and FI interact with generalization: from complementary (when \(\nu = 1\)) to partially redundant but stronger overall predictors (when \(\nu = 0\)).

\subsection{Additional Robustness Checks}
\label{app:pmi_robustness_checks}

The preceding analysis is a focused synthetic random-forest study at two
representative values of \(\nu\).  The robustness study here strictly broadens
that analysis along four axes: six synthetic signal regimes
\(\nu\in\{0,0.2,0.4,0.6,0.8,1\}\), four base estimator families (RF, DT,
AdaBoost, and GBRT), four PMI residualizers (RF, DT, AdaBoost, and
GBRT), and four completed real tabular regression benchmarks
(\texttt{diabetes}, \texttt{abalone}, \texttt{cpu\_act}, and
\texttt{house\_prices}).  The base estimator grids contain 105, 35, 63, and 63
hyperparameter settings, respectively, and each reported summary uses 100
bootstrap replicates, each subsampling one third of the corresponding grid.

Tables~\ref{tab:pmi-synthetic-full}--\ref{tab:dr2-real-full} report the
conditional contribution of FI after accounting for II.  The entries are
positive across all displayed synthetic and real-data settings under both
\pmifi{} and incremental \dRtwo.  Across the 96 synthetic
estimator/residualizer/\(\nu\) groups and the 64 real-data
estimator/residualizer/dataset groups, the bootstrap positive rate is 1.0 for
every group.  For \dRtwo, the values repeat across residualizer columns because
the residualizer is only used in the PMI calculation.  Thus, the conclusion
from the focused analysis above persists under a broader estimator grid,
multiple dependence residualizers, and real tabular data.

\begin{table}[p]
\centering
\caption{Synthetic \pmifi{} across estimator families and PMI residualizers (6 seeds/config, 100 bootstrap repetitions).}
\label{tab:pmi-synthetic-full}
\scriptsize
\resizebox{\textwidth}{!}{%
\begin{tabular}{lcccccccccccccccc}
\toprule
& \multicolumn{4}{c}{RF residualizer} & \multicolumn{4}{c}{DT residualizer} & \multicolumn{4}{c}{AdaBoost residualizer} & \multicolumn{4}{c}{GBRT residualizer} \\
\cmidrule(lr){2-5}\cmidrule(lr){6-9}\cmidrule(lr){10-13}\cmidrule(lr){14-17}
\shortstack{Synthetic\\setting} & RF & DT & Ada & GBRT & RF & DT & Ada & GBRT & RF & DT & Ada & GBRT & RF & DT & Ada & GBRT \\
\midrule
\(\nu = 0.0\) & 0.380 & 0.633 & 0.485 & 0.376 & 0.417 & 0.841 & 0.598 & 0.696 & 0.375 & 0.582 & 0.460 & 0.350 & 0.337 & 0.486 & 0.432 & 0.343 \\
\(\nu = 0.2\) & 0.251 & 0.496 & 0.454 & 0.443 & 0.325 & 0.723 & 0.590 & 0.644 & 0.263 & 0.441 & 0.455 & 0.404 & 0.239 & 0.391 & 0.378 & 0.360 \\
\(\nu = 0.4\) & 0.238 & 0.378 & 0.347 & 0.375 & 0.252 & 0.555 & 0.428 & 0.525 & 0.239 & 0.342 & 0.349 & 0.365 & 0.223 & 0.335 & 0.303 & 0.340 \\
\(\nu = 0.6\) & 0.228 & 0.373 & 0.522 & 0.306 & 0.299 & 0.495 & 0.636 & 0.424 & 0.249 & 0.344 & 0.512 & 0.281 & 0.226 & 0.352 & 0.427 & 0.255 \\
\(\nu = 0.8\) & 0.223 & 0.322 & 0.582 & 0.368 & 0.236 & 0.404 & 0.682 & 0.550 & 0.229 & 0.307 & 0.595 & 0.348 & 0.222 & 0.294 & 0.477 & 0.318 \\
\(\nu = 1.0\) & 0.250 & 0.303 & 0.458 & 0.414 & 0.291 & 0.378 & 0.585 & 0.537 & 0.265 & 0.291 & 0.488 & 0.397 & 0.224 & 0.302 & 0.379 & 0.381 \\
\bottomrule
\end{tabular}%
}
\end{table}

\begin{table}[p]
\centering
\caption{Synthetic \dRtwo{} across estimator families and PMI residualizers (6 seeds/config, 100 bootstrap repetitions).}
\label{tab:dr2-synthetic-full}
\scriptsize
\resizebox{\textwidth}{!}{%
\begin{tabular}{lcccccccccccccccc}
\toprule
& \multicolumn{4}{c}{RF residualizer} & \multicolumn{4}{c}{DT residualizer} & \multicolumn{4}{c}{AdaBoost residualizer} & \multicolumn{4}{c}{GBRT residualizer} \\
\cmidrule(lr){2-5}\cmidrule(lr){6-9}\cmidrule(lr){10-13}\cmidrule(lr){14-17}
\shortstack{Synthetic\\setting} & RF & DT & Ada & GBRT & RF & DT & Ada & GBRT & RF & DT & Ada & GBRT & RF & DT & Ada & GBRT \\
\midrule
\(\nu = 0.0\) & 0.152 & 0.485 & 0.246 & 0.231 & 0.152 & 0.485 & 0.246 & 0.231 & 0.152 & 0.485 & 0.246 & 0.231 & 0.152 & 0.485 & 0.246 & 0.231 \\
\(\nu = 0.2\) & 0.064 & 0.402 & 0.300 & 0.273 & 0.064 & 0.402 & 0.300 & 0.273 & 0.064 & 0.402 & 0.300 & 0.273 & 0.064 & 0.402 & 0.300 & 0.273 \\
\(\nu = 0.4\) & 0.014 & 0.320 & 0.204 & 0.334 & 0.014 & 0.320 & 0.204 & 0.334 & 0.014 & 0.320 & 0.204 & 0.334 & 0.014 & 0.320 & 0.204 & 0.334 \\
\(\nu = 0.6\) & 0.028 & 0.310 & 0.510 & 0.332 & 0.028 & 0.310 & 0.510 & 0.332 & 0.028 & 0.310 & 0.510 & 0.332 & 0.028 & 0.310 & 0.510 & 0.332 \\
\(\nu = 0.8\) & 0.011 & 0.133 & 0.496 & 0.420 & 0.011 & 0.133 & 0.496 & 0.420 & 0.011 & 0.133 & 0.496 & 0.420 & 0.011 & 0.133 & 0.496 & 0.420 \\
\(\nu = 1.0\) & 0.016 & 0.095 & 0.316 & 0.454 & 0.016 & 0.095 & 0.316 & 0.454 & 0.016 & 0.095 & 0.316 & 0.454 & 0.016 & 0.095 & 0.316 & 0.454 \\
\bottomrule
\end{tabular}%
}
\end{table}

\begin{table}[p]
\centering
\caption{Real-data \pmifi{} across estimator families and PMI residualizers (6 seeds/config, 100 bootstrap repetitions).}
\label{tab:pmi-real-full}
\scriptsize
\resizebox{\textwidth}{!}{%
\begin{tabular}{lcccccccccccccccc}
\toprule
& \multicolumn{4}{c}{RF residualizer} & \multicolumn{4}{c}{DT residualizer} & \multicolumn{4}{c}{AdaBoost residualizer} & \multicolumn{4}{c}{GBRT residualizer} \\
\cmidrule(lr){2-5}\cmidrule(lr){6-9}\cmidrule(lr){10-13}\cmidrule(lr){14-17}
Real dataset & RF & DT & Ada & GBRT & RF & DT & Ada & GBRT & RF & DT & Ada & GBRT & RF & DT & Ada & GBRT \\
\midrule
Abalone & 0.219 & 0.334 & 0.384 & 0.292 & 0.232 & 0.387 & 0.452 & 0.573 & 0.227 & 0.316 & 0.402 & 0.283 & 0.216 & 0.390 & 0.363 & 0.276 \\
CPU Act & 0.264 & 0.387 & 0.278 & 0.283 & 0.330 & 0.438 & 0.296 & 0.361 & 0.279 & 0.372 & 0.281 & 0.266 & 0.252 & 0.471 & 0.293 & 0.269 \\
Diabetes & 0.358 & 0.368 & 0.549 & 0.437 & 0.412 & 0.444 & 0.597 & 0.538 & 0.364 & 0.332 & 0.524 & 0.404 & 0.339 & 0.328 & 0.505 & 0.398 \\
House prices & 0.216 & 0.310 & 0.404 & 0.365 & 0.222 & 0.313 & 0.399 & 0.436 & 0.219 & 0.288 & 0.373 & 0.355 & 0.213 & 0.294 & 0.368 & 0.347 \\
\bottomrule
\end{tabular}%
}
\end{table}

\begin{table}[p]
\centering
\caption{Real-data \dRtwo{} across estimator families and PMI residualizers (6 seeds/config, 100 bootstrap repetitions).}
\label{tab:dr2-real-full}
\scriptsize
\resizebox{\textwidth}{!}{%
\begin{tabular}{lcccccccccccccccc}
\toprule
& \multicolumn{4}{c}{RF residualizer} & \multicolumn{4}{c}{DT residualizer} & \multicolumn{4}{c}{AdaBoost residualizer} & \multicolumn{4}{c}{GBRT residualizer} \\
\cmidrule(lr){2-5}\cmidrule(lr){6-9}\cmidrule(lr){10-13}\cmidrule(lr){14-17}
Real dataset & RF & DT & Ada & GBRT & RF & DT & Ada & GBRT & RF & DT & Ada & GBRT & RF & DT & Ada & GBRT \\
\midrule
Abalone & 0.017 & 0.212 & 0.295 & 0.114 & 0.017 & 0.212 & 0.295 & 0.114 & 0.017 & 0.212 & 0.295 & 0.114 & 0.017 & 0.212 & 0.295 & 0.114 \\
CPU Act & 0.087 & 0.202 & 0.028 & 0.027 & 0.087 & 0.202 & 0.028 & 0.027 & 0.087 & 0.202 & 0.028 & 0.027 & 0.087 & 0.202 & 0.028 & 0.027 \\
Diabetes & 0.100 & 0.217 & 0.469 & 0.314 & 0.100 & 0.217 & 0.469 & 0.314 & 0.100 & 0.217 & 0.469 & 0.314 & 0.100 & 0.217 & 0.469 & 0.314 \\
House prices & 0.016 & 0.047 & 0.238 & 0.048 & 0.016 & 0.047 & 0.238 & 0.048 & 0.016 & 0.047 & 0.238 & 0.048 & 0.016 & 0.047 & 0.238 & 0.048 \\
\bottomrule
\end{tabular}%
}
\end{table}

\clearpage

We also performed a noise-control check by replacing FI with Gaussian and
permuted-noise baselines and rerunning the same conditional-dependence
analysis.  Tables~\ref{tab:pmi-noise-synthetic-full}--\ref{tab:dr2-noise-real-full}
report the differences between the FI-based statistic and the corresponding
noise-based statistic, averaged over the two noise controls.  The incremental
\dRtwo{} advantage of FI over the noise baseline is positive in all 320
noise-control aggregate entries.  The \pmifi{} advantage over \pminoise{} is
positive in 319 out of 320 entries; the only negative entry is \(-7.7\times
10^{-4}\), which is numerically negligible.  Thus, the observed conditional
signal is not explained by simply adding another random covariate.

\begin{table}[p]
\centering
\caption{Synthetic noise-control: average PMI difference \(\pmifi-\pminoise\) across estimator families and PMI residualizers (averaged over Gaussian and permuted noise; 6 seeds/config, 100 bootstrap repetitions).}
\label{tab:pmi-noise-synthetic-full}
\scriptsize
\resizebox{\textwidth}{!}{%
\begin{tabular}{lcccccccccccccccc}
\toprule
& \multicolumn{4}{c}{RF residualizer} & \multicolumn{4}{c}{DT residualizer} & \multicolumn{4}{c}{AdaBoost residualizer} & \multicolumn{4}{c}{GBRT residualizer} \\
\cmidrule(lr){2-5}\cmidrule(lr){6-9}\cmidrule(lr){10-13}\cmidrule(lr){14-17}
\shortstack{Synthetic\\setting} & RF & DT & Ada & GBRT & RF & DT & Ada & GBRT & RF & DT & Ada & GBRT & RF & DT & Ada & GBRT \\
\midrule
\(\nu = 0.0\) & 0.167 & 0.371 & 0.241 & 0.116 & 0.201 & 0.565 & 0.357 & 0.452 & 0.170 & 0.325 & 0.216 & 0.101 & 0.118 & 0.191 & 0.188 & 0.096 \\
\(\nu = 0.2\) & 0.042 & 0.239 & 0.206 & 0.198 & 0.119 & 0.454 & 0.344 & 0.394 & 0.045 & 0.185 & 0.204 & 0.160 & 0.027 & 0.128 & 0.112 & 0.120 \\
\(\nu = 0.4\) & 0.028 & 0.115 & 0.103 & 0.134 & 0.040 & 0.280 & 0.181 & 0.270 & 0.027 & 0.100 & 0.104 & 0.122 & 0.007 & 0.074 & 0.053 & 0.099 \\
\(\nu = 0.6\) & 0.020 & 0.095 & 0.262 & 0.049 & 0.079 & 0.241 & 0.382 & 0.173 & 0.042 & 0.094 & 0.247 & 0.037 & 0.011 & 0.063 & 0.174 & 0.012 \\
\(\nu = 0.8\) & 0.008 & 0.050 & 0.325 & 0.135 & 0.026 & 0.139 & 0.424 & 0.323 & 0.021 & 0.039 & 0.344 & 0.120 & 0.006 & 0.029 & 0.209 & 0.086 \\
\(\nu = 1.0\) & 0.042 & 0.039 & 0.221 & 0.162 & 0.079 & 0.098 & 0.338 & 0.298 & 0.047 & 0.019 & 0.241 & 0.150 & 0.011 & 0.029 & 0.131 & 0.123 \\
\bottomrule
\end{tabular}%
}
\end{table}

\begin{table}[p]
\centering
\caption{Synthetic noise-control: average \dRtwo{} difference \(\drfi-\drnoise\) across estimator families and PMI residualizers (averaged over Gaussian and permuted noise; 6 seeds/config, 100 bootstrap repetitions).}
\label{tab:dr2-noise-synthetic-full}
\scriptsize
\resizebox{\textwidth}{!}{%
\begin{tabular}{lcccccccccccccccc}
\toprule
& \multicolumn{4}{c}{RF residualizer} & \multicolumn{4}{c}{DT residualizer} & \multicolumn{4}{c}{AdaBoost residualizer} & \multicolumn{4}{c}{GBRT residualizer} \\
\cmidrule(lr){2-5}\cmidrule(lr){6-9}\cmidrule(lr){10-13}\cmidrule(lr){14-17}
\shortstack{Synthetic\\setting} & RF & DT & Ada & GBRT & RF & DT & Ada & GBRT & RF & DT & Ada & GBRT & RF & DT & Ada & GBRT \\
\midrule
\(\nu = 0.0\) & 0.152 & 0.477 & 0.240 & 0.235 & 0.152 & 0.475 & 0.241 & 0.234 & 0.152 & 0.475 & 0.240 & 0.235 & 0.153 & 0.476 & 0.240 & 0.234 \\
\(\nu = 0.2\) & 0.062 & 0.408 & 0.306 & 0.254 & 0.061 & 0.407 & 0.307 & 0.255 & 0.061 & 0.407 & 0.307 & 0.253 & 0.061 & 0.411 & 0.305 & 0.254 \\
\(\nu = 0.4\) & 0.011 & 0.294 & 0.194 & 0.333 & 0.011 & 0.295 & 0.195 & 0.334 & 0.011 & 0.295 & 0.195 & 0.333 & 0.011 & 0.296 & 0.195 & 0.334 \\
\(\nu = 0.6\) & 0.027 & 0.292 & 0.516 & 0.335 & 0.027 & 0.290 & 0.517 & 0.335 & 0.027 & 0.289 & 0.515 & 0.335 & 0.027 & 0.290 & 0.516 & 0.336 \\
\(\nu = 0.8\) & 0.011 & 0.137 & 0.448 & 0.422 & 0.011 & 0.141 & 0.447 & 0.423 & 0.011 & 0.142 & 0.447 & 0.423 & 0.011 & 0.140 & 0.448 & 0.423 \\
\(\nu = 1.0\) & 0.013 & 0.083 & 0.316 & 0.463 & 0.013 & 0.085 & 0.318 & 0.465 & 0.013 & 0.082 & 0.319 & 0.463 & 0.013 & 0.085 & 0.317 & 0.463 \\
\bottomrule
\end{tabular}%
}
\end{table}

\begin{table}[p]
\centering
\caption{Real-data noise-control: average PMI difference \(\pmifi-\pminoise\) across estimator families and PMI residualizers (averaged over Gaussian and permuted noise; 6 seeds/config, 100 bootstrap repetitions).}
\label{tab:pmi-noise-real-full}
\scriptsize
\resizebox{\textwidth}{!}{%
\begin{tabular}{lcccccccccccccccc}
\toprule
& \multicolumn{4}{c}{RF residualizer} & \multicolumn{4}{c}{DT residualizer} & \multicolumn{4}{c}{AdaBoost residualizer} & \multicolumn{4}{c}{GBRT residualizer} \\
\cmidrule(lr){2-5}\cmidrule(lr){6-9}\cmidrule(lr){10-13}\cmidrule(lr){14-17}
Real dataset & RF & DT & Ada & GBRT & RF & DT & Ada & GBRT & RF & DT & Ada & GBRT & RF & DT & Ada & GBRT \\
\midrule
Abalone & 0.007 & 0.068 & 0.153 & 0.051 & 0.022 & 0.096 & 0.224 & 0.325 & 0.020 & 0.047 & 0.167 & 0.043 & 0.000 & 0.025 & 0.116 & 0.033 \\
CPU Act & 0.052 & 0.114 & 0.032 & 0.040 & 0.119 & 0.149 & 0.045 & 0.119 & 0.067 & 0.093 & 0.027 & 0.021 & 0.039 & 0.039 & 0.042 & 0.017 \\
Diabetes & 0.140 & 0.115 & 0.294 & 0.200 & 0.192 & 0.172 & 0.337 & 0.302 & 0.151 & 0.062 & 0.272 & 0.149 & 0.118 & 0.063 & 0.243 & 0.152 \\
House prices & 0.005 & 0.045 & 0.150 & 0.120 & 0.004 & 0.043 & 0.143 & 0.194 & 0.009 & 0.030 & 0.131 & 0.108 & 0.001 & 0.007 & 0.122 & 0.104 \\
\bottomrule
\end{tabular}%
}
\end{table}

\begin{table}[p]
\centering
\caption{Real-data noise-control: average \dRtwo{} difference \(\drfi-\drnoise\) across estimator families and PMI residualizers (averaged over Gaussian and permuted noise; 6 seeds/config, 100 bootstrap repetitions).}
\label{tab:dr2-noise-real-full}
\scriptsize
\resizebox{\textwidth}{!}{%
\begin{tabular}{lcccccccccccccccc}
\toprule
& \multicolumn{4}{c}{RF residualizer} & \multicolumn{4}{c}{DT residualizer} & \multicolumn{4}{c}{AdaBoost residualizer} & \multicolumn{4}{c}{GBRT residualizer} \\
\cmidrule(lr){2-5}\cmidrule(lr){6-9}\cmidrule(lr){10-13}\cmidrule(lr){14-17}
Real dataset & RF & DT & Ada & GBRT & RF & DT & Ada & GBRT & RF & DT & Ada & GBRT & RF & DT & Ada & GBRT \\
\midrule
Abalone & 0.011 & 0.184 & 0.312 & 0.103 & 0.013 & 0.186 & 0.312 & 0.103 & 0.012 & 0.184 & 0.311 & 0.103 & 0.013 & 0.183 & 0.311 & 0.103 \\
CPU Act & 0.083 & 0.173 & 0.018 & 0.020 & 0.084 & 0.171 & 0.017 & 0.021 & 0.082 & 0.175 & 0.019 & 0.021 & 0.082 & 0.172 & 0.017 & 0.021 \\
Diabetes & 0.092 & 0.201 & 0.430 & 0.327 & 0.092 & 0.200 & 0.430 & 0.328 & 0.092 & 0.201 & 0.431 & 0.327 & 0.091 & 0.199 & 0.431 & 0.327 \\
House prices & 0.010 & 0.035 & 0.238 & 0.046 & 0.011 & 0.035 & 0.236 & 0.046 & 0.010 & 0.033 & 0.238 & 0.047 & 0.009 & 0.036 & 0.239 & 0.046 \\
\bottomrule
\end{tabular}%
}
\end{table}

\begin{table}[p]
\centering
\caption{Noise-control consistency against the Gaussian baseline. ``FI wins on'' is the average, within each scope, of the bootstrap win rate for a fixed completed row; that row-level win rate is the fraction of bootstrap replicates in which the FI-based statistic exceeds its Gaussian-noise counterpart.}
\label{tab:noise-control-consistency}
\small
\setlength{\tabcolsep}{3pt}
\begin{tabular}{@{}lcccc@{}}
\toprule
Scope & \shortstack{mean\\\dRtwo{} diff} & \shortstack{FI wins\\on \dRtwo} & \shortstack{mean\\PMI diff} & \shortstack{FI wins\\on PMI} \\
\midrule
Real & 0.143 & 0.935 & 0.103 & 0.852 \\
Synthetic subset & 0.252 & 0.972 & 0.155 & 0.907 \\
\bottomrule
\end{tabular}
\end{table}

\section{Contents Related to Stability of Bagged Linear Regression}
\label{app:linear-stability-proofs}

\subsection{Error Analysis}\label{sec:error-analysis}

In the general setting, we may remove both instance $i$ and feature $j$ simultaneously. The updated sketching matrices $\bU_b'$ and $\bV_b'$ reflect these removals as defined above. The corresponding modified bagged estimator is:
\[
\bbeta' := \frac{1}{B} \sum_{b=1}^B \bV_b' (\bU_b'^{\top} \bX \bV_b')^{\dagger} \bU_b'^{\top} \by.
\]

For notational simplicity, we use $\bU_b'$ and $\bV_b'$ throughout to denote the sketching matrices under the appropriate removal operation. If only instance subsampling is applied (i.e., no feature subsampling), then $\bV_b' = \bV_b$ for all $b$; conversely, if only feature sampling is used, we have $\bU_b' = \bU_b$.
For notational simplicity, we define the weighting matrices \begin{align*}
g(\bX) = \frac{1}{B} \sum_{b=1}^B \bV_b (\bU_b^{\top} \bX \bV_b)^{\dagger} \bU_b^{\top}, g'(\bX) = \frac{1}{B} \sum_{b=1}^B \bV_b' (\bU_b^{\top\prime} \bX \bV_b')^{\dagger} \bU_b^{\top\prime},
\end{align*}
so that $\bbeta = g(\bX) \by$, $\bbeta' = g'(\bX) \by$.
We decompose the instability into
\begin{align*}
    \mathbb{E} \left[\left\|\bbeta - \bbeta^{\prime} \right\|_2^2\right] = &  \mathbb{E}_{\bX, \by, \bU, \bV} \left[\by^{\top} \left(g(\bX) - g'(\bX)\right)^{\top} \left(g(\bX) - g'(\bX)\right) \by\right]\\
 = & \underbrace{\sigma^2 \mathbb{E}_{\bX, \bU, \bV} \left[\operatorname{tr}\left( \left(g(\bX) - g'(\bX)\right)^{\top} \left(g(\bX) - g'(\bX)\right) \right)\right]}_{\textbf{Variance}}  \\
 + & \underbrace{\frac{1}{d}\mathbb{E}_{\bX, \bU, \bV} \left[\operatorname{tr}\left(\bX^{\top} \left(g(\bX) - g'(\bX)\right)^{\top} \left(g(\bX) - g'(\bX)\right)\bX \right)\right]}_{\textbf{Bias}}.
\end{align*}
The two terms correspond to the impact of removal on the variance and bias terms, respectively.
The following propositions give the exact values of the variance and the bias term.

\begin{proposition}[\textbf{Variance term of instability}]\label{prop:varianceofstability}
Assume the above assumptions hold.
Then the expectation of the variance term is
\begin{align*}
 \sigma^2 \mathbb{E}_{\bX, \bU, \bV} \left[\operatorname{tr}\left( \left(g(\bX) - g'(\bX)\right)^{\top} \left(g(\bX) - g'(\bX)\right) \right)\right]
=
\sigma^2 \left(\frac{\Delta^{V=}}{B} + \frac{(B-1)\Delta^{V\neq}}{B}\right),
\end{align*}
where
\begin{align*}
\lim_{n,d\to \infty} \Delta^{V=} :=
\begin{cases}
\frac{2 \gamma q (1 - p)}{(p - \gamma q)(1 - \gamma q)} & \text{ if }  \gamma q < p \text{ and } \text{ $i$-th instance removed } \\
\frac{2 \gamma pq (1 - q)}{(p - \gamma q)(p - \gamma q^2)}
& \text{ if }  \gamma q < p \text{ and } \text{ $j$-th feature removed } \\
\frac{2 \gamma q (p - p^2)}{( \gamma q - p)( \gamma q - p^2)}
& \text{ if }  \gamma q > p \text{ and } \text{ $i$-th instance removed } \\
\frac{2 \gamma q (1 - p)}{( \gamma q - p)( \gamma - p)}& \text{ if }  \gamma q > p \text{ and } \text{ $j$-th feature removed }
\end{cases}
\end{align*}
and
\begin{align*}
\lim_{n,d\to \infty}  n \Delta^{V\neq} :=
\begin{cases}
 \frac{\gamma q^2}{(1 - \gamma q^2)^2} & \text{ if }  \gamma q < p \text{ and } \text{ $i$-th instance removed } \\
 \frac{q^2}{(1 - \gamma q^2)^2}
& \text{ if }  \gamma q < p \text{ and } \text{ $j$-th feature removed } \\
 \frac{\gamma p^2}{(\gamma - p^2)^2}
& \text{ if }  \gamma q > p \text{ and } \text{ $i$-th instance removed } \\
 \frac{ p^2}{(\gamma - p^2)^2}
& \text{ if }  \gamma q > p \text{ and } \text{ $j$-th feature removed }.
\end{cases}
\end{align*}
\end{proposition}

\begin{proposition}[\textbf{Bias term of instability}]\label{prop:biasofstability}
Assume the above assumptions hold.
Then, the expectation of the bias term is
\begin{align*}
  \frac{1}{d} \; \mathbb{E}_{\bX, \bU, \bV} \left[\operatorname{tr}\left(\bX^{\top} \left(g(\bX) - g'(\bX)\right)^{\top} \left(g(\bX) - g'(\bX)\right)\bX \right)\right]
=
\frac{B-1}{B}\Delta^{B\neq} + \frac{1}{B}\Delta^{B =},
\end{align*}
where
\begin{align*}
\lim_{n,d\to \infty}\Delta^{B=} :=
\begin{cases}
    \frac{2 \gamma q (1- q)(1 - p)}{(p - \gamma q)(1 - \gamma q)}  & \text{ if }  \gamma q < p \text{ and } \text{ $i$-th instance removed } \\
    \frac{2 \gamma pq (1 - q)^2 }{(p - \gamma q)(p - \gamma q^2)} + \frac{2  pq (1 - q)}{p - \gamma q^2} & \text{ if }  \gamma q < p \text{ and } \text{ $j$-th feature removed } \\
     \frac{2\gamma p q(\gamma - p) (1- p) }{\gamma (\gamma q - p)(\gamma q - p^2)} -  \frac{2 (1 - p) p^2 }{\gamma(\gamma q - p^2)} & \text{ if }  \gamma q > p \text{ and } \text{ $i$-th instance removed } \\
    \frac{ 2 p (1 - q) }{ \gamma  q - p}  & \text{ if }  \gamma q > p \text{ and } \text{ $j$-th feature removed }.
\end{cases}
\end{align*}
and
\begin{align*}
\lim_{n,d\to \infty} n\Delta^{B\neq} :=
\begin{cases}
    \frac{  \gamma q^2 (1 - q)^2 }{(1 - \gamma q^2)^2} & \text{ if }  \gamma q < p \text{ and } \text{ $i$-th instance removed } \\
    \frac{   q^2 (1 +\gamma - 2\gamma q)  }{\gamma (1 - \gamma q^2)^2}  & \text{ if }  \gamma q < p \text{ and } \text{ $j$-th feature removed } \\
    \frac{p^2(\gamma - p)^2}{\gamma (\gamma - p^2)^2} & \text{ if }  \gamma q > p \text{ and } \text{ $i$-th instance removed } \\
    \frac{p^2(1 - 2 p + \gamma )}{(\gamma - p^2)^2}  & \text{ if }  \gamma q > p \text{ and } \text{ $j$-th feature removed }.
\end{cases}
\end{align*}
\end{proposition}

\begin{proof}[Proof of Theorem~\ref{thm:stabilitylinear-decomp-subscript}]

Theorem~\ref{thm:stabilitylinear-decomp-subscript}, including all entries of Table~\ref{tab:linear-stability-limits},  follows by applying Propositions~\ref{prop:varianceofstability} and~\ref{prop:biasofstability} to
the two perturbations \(\ell=-i\) and \(\ell=-j\), and then collecting the
corresponding variance and bias terms.

\end{proof}

\subsection{Useful Lemmas}\label{sec:twousefullemmas}

\begin{lemma}[Expectation of matrix under MP-law]\label{lem:mplaw}
Let $\bX$ be a $n\times d$ matrix with each row sampled from $\mathcal{N}(\mathbf{0}, \Sigma)$, where $\Sigma$ is invertible.
Then, there holds
\begin{align*}
    \mathbb{E}_{\bX} \left[\bX^{\top} \bX\right] = n \Sigma, \;\;\; \mathbb{E}_{\bX} \left[(\bX^{\top} \bX)^{-1}\right] = \frac{1}{ n - d - 1} \Sigma^{-1}.
\end{align*}
\end{lemma}
\begin{proof}[Proof of \ref{lem:mplaw}]
The conclusions follow from standard properties of inverse Wishart distribution, see e.g. \citet{haff1979identity}.
\end{proof}

\begin{lemma}[Lemma A.1 of \citet{lejeune2020implicit}]
\label{lem:lemmaoflejeune1}
    Let $\bV$ and $\bV^c$ be the selection matrix corresponds to $\bnu$ and $\bnu^c$, respectively.
    Then for any random matrix $\bX \in \mathbb{R}^{n\times d}$ whose entries are sampled i.i.d. from the standard Gaussian distribution, we have the following holds true
    \begin{align*}
     \mathbb{E}_{\bX \bV^c} \left[\bV^{\top} \bX^{\dagger}\right] = (\bX \bS)^{\dagger} \text{ if } n > d, \text{ and }\;\; \mathbb{E}_{\bU^{c\top}\bX} \left[ \bX^{ \dagger} T\right] = (\bU^{\top}\bX)^{\dagger} \text{ if } n < d.
    \end{align*}
    Moreover, if $h(\bX \bS)$ and $\bX \bV^{c}$ are independent, there holds
\begin{align*}
    \mathbb{E}_{\bX \bV^c} \left[\bV^{c\top}\bX^{\top}  h(\bX\bS) \bV^{\top} \bX^{\dagger}\right] = \mathbf{0} \text{ if } n > d, \text{ and }\;\; \mathbb{E}_{\bU^{c\top}X} \left[X^{\top} \bU^{c}  h(\bU^{\top}X)  \bX^{\dagger} T\right] = \mathbf{0} \text{ if } n < d.
\end{align*}
\end{lemma}







The assumption $\mathbb{E}_T\left[ \bU_1^{\top} \bU_1\right] = \mathbb{E}_T\left[ \bU_2^{\top} \bU_2\right] = \frac{|\bmu_1|}{\tr(\Lambda)} \Lambda$ appears less natural, but is important to unify the results for $T$ and $T'$.

\begin{lemma}[Variance]\label{lem:variance-resorted-lemma}
For $\bmu_1, \bmu_2, \bnu_1, \bnu_2$ sampled independently, let $\bU_1,\bU_2,\bV_1,\bV_2$ be their selection matrix.
Suppose $\Lambda_1, \Lambda_2$ and $\Theta_1, \Theta_2$ are diagonal matrices with only $1$s and $0$s on their diagonal.
Then let $\mathbb{E}_T\left[ \bU_1 \bU_1^{\top}\right] = \frac{|\bmu_1|}{\tr(\Lambda_1)} \Lambda_1 $, $ \mathbb{E}_T\left[ \bU_2 \bU_2^{\top}\right] = \frac{|\bmu_2|}{\tr(\Lambda_2)} \Lambda_2 $, $\mathbb{E}_S\left[ \bV_1 \bV_1^{\top}\right] = \frac{|\bnu_1|}{\tr(\Theta_1)} \Theta_1 $, $\mathbb{E}_S\left[ \bV_2 \bV_2^{\top}\right] = \frac{|\bnu_2|}{\tr(\Theta_2)} \Theta_2$.
Assume we always have $\Lambda_1 \Lambda_2 \in \{\Lambda_1, \Lambda_2\}$ and $\Theta_1 \Theta \in  \{\Theta_1, \Theta_2\}$.
Define
\begin{align*}
     \Lambda^{\wedge} = \Lambda_1 \Lambda_2,  \;\;  \Lambda^{\vee} =  \Lambda_1 + \Lambda_2 - \Lambda^{\wedge}, \text{ and } \;\; \Theta^{\wedge} = \Theta_1 \Theta_2,  \;\;  \Theta^{\vee} =  \Theta_1 + \Theta_2 - \Theta^{\wedge}.
\end{align*}
Then there holds
\begin{align}
\label{equ:lem-variance-diff-b}
\mathbb{E}_{\bX, \bS, \bT}\left[\operatorname{tr}\left(\bU_1 (\bU_1^{\top} \bX \bV_1)^{\dagger \top } \bV_1^{\top} \bV_2(\bU_2^{\top} \bX \bV_2)^{\dagger} \bU_2^{\top}\right)\right] = \begin{cases}
    \frac{|\bnu_1 \cap \bnu_2|}{\tr(\Lambda^{\vee}) -  |\bnu_1 \cap \bnu_2| - 1} & \text{ if } |\bnu_1| < |\bmu_1|, |\bnu_2| < |\bmu_2|,\\
     \frac{|\bmu_1\cap \bmu_2|}{ \tr(\Theta^{\vee}) - |\bmu_1\cap \bmu_2| - 1} & \text{ if } |\bnu_1| > |\bmu_1|, |\bnu_2| > |\bmu_2|.
\end{cases}
\end{align}
Moreover, even if $\bmu_1$ and $\bmu_2$ are not independent, and $\bnu_1$ and $\bnu_2$ are not independent, yet we know $\bmu_1 \subset \bmu_2$, $\bnu_1 \subset \bnu_2$, the quantity becomes
\begin{align}
\label{equ:lem-variance-same-b}
\mathbb{E}_{\bX, \bS, \bT}\left[\operatorname{tr}\left(\bU_1 (\bU_1^{\top} \bX \bV_1)^{\dagger \top } \bV_1^{\top} \bV_2(\bU_2^{\top} \bX \bV_2)^{\dagger} \bU_2^{\top}\right)\right] = \begin{cases}
    \frac{|\bnu_1|}{ |\bmu_2|-  |\bnu_1 | - 1} & \text{ if } |\bnu_1| < |\bmu_1|, |\bnu_2| < |\bmu_2|,\\
     \frac{|\bmu_1|}{ |\bnu_2| - |\bmu_1| - 1} & \text{ if } |\bnu_1| > |\bmu_1|, |\bnu_2| > |\bmu_2|.
\end{cases}
\end{align}
\end{lemma}

\begin{proof}[Proof of Lemma \ref{lem:variance-resorted-lemma}]
By definition, $\{\Lambda_1 ,\Lambda_2\} = \{\Lambda^{\wedge}, \Lambda^{\vee}\}$.
W.o.l.g., we assume $(\Lambda_1 ,\Lambda_2) = (\Lambda^{\wedge}, \Lambda^{\vee})$.
Similarly, w.o.l.g., we assume $(\Theta_1,\Theta_2) = (\Theta^{\wedge},\Theta^{\vee} )$.

\textbf{The first circumstance}\;\;
We begin with the first circumstances.
Let $\bV_{1\cap2}$ be the selection matrix associated with $\bnu_1\cap \bnu_2$.
In this case, we can directly use the first argument of Lemma \ref{lem:lemmaoflejeune1} to get
\begin{align*}
& \mathbb{E}_{\bX, \bS, \bT}\left[\operatorname{tr}\left(\bU_1 (\bU_1^{\top} \bX \bV_1)^{\dagger \top } \bV_1^{\top} \bV_2(\bU_2^{\top} \bX \bV_2)^{\dagger} \bU_2^{\top}\right)\right] \\
= &
\mathbb{E}_{\bX, \bS, \bT}\left[\operatorname{tr}\left(\bU_1 (\bU_1^{\top} \bX \bV_1)^{\dagger \top } \bV_1^{\top}  \bV_{1\cap2} \bV_{1\cap2}^{\top}\bV_2(\bU_2^{\top} \bX \bV_2)^{\dagger} \bU_2^{\top}\right)\right]\\
= &
\mathbb{E}_{\bX, \bS, \bT}\left[\operatorname{tr}\left(\bU_1 (\bU_1^{\top} \bX \bV_{1\cap2})^{\dagger \top } (\bU_2^{\top} \bX \bV_{1\cap2})^{\dagger} \bU_2^{\top}\right)\right].
\end{align*}
Let $\Pi_{A} = I - A^{\top \dagger} A^{\top}$ be the projection matrix associated to $A$.
Then, we can break the above equation to get
\begin{align}\nonumber
    & \mathbb{E}_{\bX, \bS, \bT}\left[\operatorname{tr}\left(\bU_1 (\bU_1^{\top} \bX \bV_{1\cap2})^{\dagger \top } (\bU_2^{\top} \bX \bV_{1\cap2})^{\dagger} \bU_2^{\top}\right)\right] \\
    = & \mathbb{E}_{\bX, \bS, \bT}\left[\operatorname{tr}\left((\bU_2^{\top} \bX \bV_{1\cap2})^{\dagger} \bU_2^{\top} \bU_1 (\bU_1^{\top} \bX \bV_{1\cap2})^{\dagger \top } \right)\right] \nonumber\\
    = & \mathbb{E}_{\bX, \bS, \bT}\left[\operatorname{tr}\left((\bU_2^{\top} \bX \bV_{1\cap2})^{\dagger} \bU_2^{\top} \left(\Pi_{\bX^{\vee} \bV_{1\cap2}} + \bX^{\vee} \bV_{1\cap2} (\bV_{1\cap2}^{\top} \bX^{\vee\top} \bX^{\vee} \bV_{1\cap2})^{-1} \bV_{1\cap2}^{\top} \bX^{\vee \top}  \right) \bU_1 (\bU_1^{\top} \bX \bV_{1\cap2})^{\dagger \top } \right)\right],
    \label{equ:variance-resorted-lemma-1}
\end{align}
where we define $\bX^{\vee} = \Lambda^{\vee} \bX$ as a temporary notation.
We scope in the first half of this quantity, where
\begin{align*}
    (\bU_2^{\top} \bX \bV_{1\cap2})^{\dagger} \bU_2^{\top} \Pi_{\bX^{\vee} \bV_{1\cap2}} =  (\bV_{1\cap2}^{\top} \bX^{\top} \bU_2 \bU_2^{\top} \bX \bV_{1\cap2})^{-1} \bV_{1\cap2}^{\top} \bX^{\top} \bU_2\bU_2^{\top} \Pi_{\bX^{\vee} \bV_{1\cap2}}.
\end{align*}
Since the distribution of $\bU_2\bU_2^{\top}$ is independent of the rest conditioned on we know $|\bmu_2|$, we have
\begin{align*}
    \mathbb{E}_{T} \left[(\bU_2^{\top} \bX \bV_{1\cap2})^{\dagger} \bU_2^{\top} \Pi_{\bX^{\vee} \bV_{1\cap2}}\right] = & (\bV_{1\cap2}^{\top} \bX^{\top} \bU_2 \bU_2^{\top} \bX \bV_{1\cap2})^{-1} \bV_{1\cap2}^{\top} \bX^{\top} \frac{|\bmu_2|}{\tr(\Lambda_2)}  \Lambda_2  \Pi_{\bX^{\vee} \bV_{1\cap2}} \\
    = & (\bV^{\top}_{1\cap2} \bX^{\top} \bU_2 \bU_2^{\top} \bX \bV_{1\cap2})^{-1} \bV^{\top}_{1\cap2} \bX^{\vee\top}  \Pi_{\bX^{\vee} \bV_{1\cap2}} = \mathbf{0},
\end{align*}
where we recall that we assumed $\Lambda^{\vee} = \Lambda_2$.
This brings \eqref{equ:variance-resorted-lemma-1} into
\begin{align*}
& \mathbb{E}_{\bX, \bS, \bT}\left[\operatorname{tr}\left(\bU_1 (\bU_1^{\top} \bX \bV_{1\cap2})^{\dagger \top } (\bU_2^{\top} \bX \bV_{1\cap2})^{\dagger} \bU_2^{\top}\right)\right]\\
= & \mathbb{E}_{\bX, \bS, \bT}\left[\operatorname{tr}\left((\bU_2^{\top} \bX \bV_{1\cap2})^{\dagger} \bU_2^{\top} \bX^{\vee} \bV_{1\cap2} (\bV_{1\cap2}^{\top} \bX^{\vee\top} \bX^{\vee} \bV_{1\cap2})^{-1} \bV_{1\cap2}^{\top} \bX^{\vee \top}   \bU_1 (\bU_1^{\top} \bX \bV_{1\cap2})^{\dagger \top } \right)\right]
\\
\overset{(i)}{=} & \mathbb{E}_{\bX, \bS, \bT}\left[\operatorname{tr}\left((\bU_2^{\top} \bX \bV_{1\cap2})^{\dagger} \bU_2^{\top} \bX \bV_{1\cap2} (\bV_{1\cap2}^{\top} \bX^{\vee\top} \bX^{\vee} \bV_{1\cap2})^{-1} \bV_{1\cap2}^{\top} \bX^{ \top}   \bU_1 (\bU_1^{\top} \bX \bV_{1\cap2})^{\dagger \top } \right)\right] \\
= & \mathbb{E}_{\bX, \bS, \bT}\left[\operatorname{tr}\left( (\bV_{1\cap2}^{\top} \bX^{\vee\top} \bX^{\vee} \bV_{1\cap2})^{-1} \right)\right].
\end{align*}
To show step $(i)$, we observe the fact that
\begin{align}
\label{equ:variance-resorted-lemma-2}
    \bU_1^{\top} \bX^{\vee} = \bU_1^{\top} \bX, \text{ and }\;\; \bU_2^{\top} \bX^{\vee} = \bU_2^{\top} \bX.
\end{align}
This is because $\Lambda^{\vee} $ is a diagonal matrix with only $1$s and $0s$ on its diagonal, while on its positions with $0$s, $\bU_1$ and $\bU_2$ must have zero rows.
Otherwise, the expectation of $\bU_1$ and $\bU_2$ can not have zero expectation on these positions.
The last step is due to Lemma \ref{lem:mplaw}, where we have
\begin{align*}
    \mathbb{E}_{\bX, \bS, \bT}\left[\operatorname{tr}\left( (\bV_{1\cap2}^{\top} \bX^{\vee\top} \bX^{\vee} \bV_{1\cap2})^{-1} \right)\right] = \frac{|\bnu_1 \cap \bnu_2|}{\tr(\Lambda^{\vee}) -  |\bnu_1 \cap \bnu_2| - 1}.
\end{align*}
To show the second condition, we notice that the operations between
\begin{align*}
    \mathbb{E}_{\bX, \bS, \bT}\left[\operatorname{tr}\left(\bU_1 (\bU_1^{\top} \bX \bV_1)^{\dagger \top } \bV_1^{\top} \bV_2(\bU_2^{\top} \bX \bV_2)^{\dagger} \bU_2^{\top}\right)\right] =
    \mathbb{E}_{\bX, \bS, \bT}\left[\operatorname{tr}\left(\bU_1 (\bU_1^{\top} \bX \bV_{1\cap2})^{\dagger \top } (\bU_2^{\top} \bX \bV_{1\cap2})^{\dagger} \bU_2^{\top}\right)\right]
\end{align*}
are still lawful, where we continue to have
\begin{align*}
& \mathbb{E}_{\bX, \bS, \bT}\left[\operatorname{tr}\left(\bU_1 (\bU_1^{\top} \bX \bV_{1\cap2})^{\dagger \top } (\bU_2^{\top} \bX \bV_{1\cap2})^{\dagger} \bU_2^{\top}\right)\right] \\
= & \mathbb{E}_{\bX, \bS, \bT}\left[\operatorname{tr}\left( (\bV_{1\cap2}^{\top} \bX^{\top} \bU_2 \bU_2^{\top} \bX \bV_{1\cap2})^{-1} \bV_{1\cap2}^{\top} \bX^{\top} \bU_2 \bU_2^{\top} \bU_1 \bU_1^{\top} \bX \bV_{1\cap2} ( \bV_{1\cap2}^\top \bX^{\top}\bU_1 \bU_1^{\top} \bX \bV_{1\cap2} )^{-1}  \right)\right] \\
\overset{(ii)}{=}& \mathbb{E}_{\bX, \bS, \bT}\left[\operatorname{tr}\left( (\bV_{1\cap2}^{\top} \bX^{\top} \bU_2 \bU_2^{\top} \bX \bV_{1\cap2})^{-1} \bV_{1\cap2}^{\top} \bX^{\top}  \bU_1 \bU_1^{\top} \bX \bV_{1\cap2} ( \bV_{1\cap2}^\top \bX^{\top}\bU_1 \bU_1^{\top} \bX \bV_{1\cap2} )^{-1}  \right)\right]  \\
= & \mathbb{E}_{\bX, \bS, \bT}\left[\operatorname{tr}\left( (\bV_{1\cap2}^{\top} \bX^{\top} \bU_2 \bU_2^{\top} \bX \bV_{1\cap2})^{-1} \right)\right] = \frac{|\bnu_1|}{|\bmu_2| - |\bnu_1| - 1},
\end{align*}
where $(ii)$ used the assumption that $\bmu_1\subset \bmu_2$, and the last step used Lemma \ref{lem:mplaw}.
We notice that $\bnu_1\cap \bnu_2 = \bnu_1$.

\textbf{The second circumstance}\;\;
Then we proceed to the second circumstance where $|\bnu_1| > |\bmu_1|, |\bnu_2| > |\bmu_2|$.
This is an analog to the first circumstance, yet we operate on the dual dimension.
Let $\bU_{1\cap2}$ be the selection matrix associated with $\bmu_1\cap \bmu_2$.
We use the first argument of Lemma \ref{lem:lemmaoflejeune1} to get
\begin{align*}
& \mathbb{E}_{\bX, \bS, \bT}\left[\operatorname{tr}\left(\bU_1 (\bU_1^{\top} \bX \bV_1)^{\dagger \top } \bV_1^{\top} \bV_2(\bU_2^{\top} \bX \bV_2)^{\dagger} \bU_2^{\top}\right)\right] \\
= &
\mathbb{E}_{\bX, \bS, \bT}\left[\operatorname{tr}\left( \bV_2(\bU_2^{\top} \bX \bV_2)^{\dagger} \bU_2^{\top} \bU_{1\cap2} \bU_{1\cap2}^{\top} \bU_1 (\bU_1^{\top} \bX \bV_1)^{\dagger \top } \bV_1^{\top} \right)\right] \\
= & \mathbb{E}_{\bX, \bS, \bT}\left[\operatorname{tr}\left( \bV_2(\bU_{1\cap2}^{\top} \bX \bV_2)^{\dagger}  (\bU_{1\cap2}^{\top} \bX \bV_1)^{\dagger \top } \bV_1^{\top} \right)\right] \\
= & \mathbb{E}_{\bX, \bS, \bT}\left[\operatorname{tr}\left( (\bU_{1\cap2}^{\top} \bX \bV_1)^{\dagger \top } \bV_1^{\top} \bV_2(\bU_{1\cap2}^{\top} \bX \bV_2)^{\dagger}   \right)\right] \\
= & \mathbb{E}_{\bX, \bS, \bT}\left[\operatorname{tr}\left( (\bU_{1\cap2}^{\top} \bX \bV_1)^{\dagger \top } \bV_1^{\top} \left(\Pi_{\bX^{\vee\top} \bU_{1\cap2} } + \bX^{\vee\top} \bU_{1\cap2} (\bU_{1\cap2}^{\top}\bX^{\vee} \bX^{\vee\top} \bU_{1\cap2})^{-1} \bU_{1\cap2}^{\top}\bX^{\vee} \right) \bV_2(\bU_{1\cap2}^{\top} \bX \bV_2)^{\dagger}   \right)\right],
\end{align*}
where $\bX^{\vee} = \bX\Theta^{\vee}$.
Following the same reasoning in the first circumstance, we get
\begin{align*}
    \mathbb{E}_{\bX, \bS, \bT}\left[\operatorname{tr}\left(\bU_1 (\bU_1^{\top} \bX \bV_1)^{\dagger \top } \bV_1^{\top} \bV_2(\bU_2^{\top} \bX \bV_2)^{\dagger} \bU_2^{\top}\right)\right] = &
    \mathbb{E}_{\bX, \bS, \bT}\left[\operatorname{tr}\left( (\bU^{\top}_{1\cap2}\bX^{\vee} \bX^{\vee\top} \bU_{1\cap2})^{-1} \right)\right] \\
    = &  \frac{|\bmu_1\cap \bmu_2|}{ \tr(\Theta^{\vee}) - |\bmu_1\cap \bmu_2| - 1}.
\end{align*}
To show the second condition, we notice that the operations between
\begin{align*}
    \mathbb{E}_{\bX, \bS, \bT}\left[\operatorname{tr}\left(\bU_1 (\bU_1^{\top} \bX \bV_1)^{\dagger \top } \bV_1^{\top} \bV_2(\bU_2^{\top} \bX \bV_2)^{\dagger} \bU_2^{\top}\right)\right] =
   \mathbb{E}_{\bX, \bS, \bT}\left[\operatorname{tr}\left( (\bU_{1\cap2}^{\top} \bX \bV_1)^{\dagger \top } \bV_1^{\top} \bV_2(\bU_{1\cap2}^{\top} \bX \bV_2)^{\dagger}   \right)\right]
\end{align*}
are still lawful, where we continue to have
\begin{align*}
& \mathbb{E}_{\bX, \bS, \bT}\left[\operatorname{tr}\left( (\bU_{1\cap2}^{\top} \bX \bV_1)^{\dagger \top } \bV_1^{\top} \bV_2(\bU_{1\cap2}^{\top} \bX \bV_2)^{\dagger}   \right)\right] \\
= & \mathbb{E}_{\bX, \bS, \bT}\left[\operatorname{tr}\left( (\bU_{1\cap2}^{\top} \bX \bV_1 \bV_1^{\top} \bX^{\top} \bU_{1\cap2})^{-1} \bU_{1\cap2}^{\top} \bX \bV_1  \bV_1^{\top} \bV_2  \bV_2^{\top} \bX^{\top} \bU_{1\cap2} (\bU_{1\cap2}^{\top} \bX \bV_2 \bV_2^{\top} \bX^{\top} \bU_{1\cap2})^{-1}  \right)\right]\\
= & \mathbb{E}_{\bX, \bS, \bT}\left[\operatorname{tr}\left( (\bU_{1\cap2}^{\top} \bX \bV_1 \bV_1^{\top} \bX^{\top} \bU_{1\cap2})^{-1} \bU^{\top}_{1\cap2} \bX \bV_1  \bV_1^{\top}  \bX^{\top} \bU_{1\cap2} (\bU_{1\cap2}^{\top} \bX \bV_2 \bV_2^{\top} \bX^{\top} \bU_{1\cap2})^{-1}  \right)\right]\\
= & \mathbb{E}_{\bX, \bS, \bT}\left[\operatorname{tr}\left( (\bU_{1\cap2}^{\top} \bX \bV_2 \bV_2^{\top} \bX^{\top} \bU_{1\cap2})^{-1}\right)\right] = \frac{|\bmu_1|}{|\bnu_2| - |\bmu_1| - 1}.
\end{align*}

\end{proof}

\begin{lemma}[Bias]\label{lem:bias-resorted-lemma}
For $\bmu_1, \bmu_2, \bnu_1, \bnu_2$ sampled independently, let $\bU_1,\bU_2,\bV_1,\bV_2$ be their selection matrix.
Moreover, let $\bU_1^c$ be the selection matrix for $\bmu_1^{c}$, and so on for $\bU_2^{c},\bV_1^c,\bV_2^c$.
Suppose $\Lambda_1, \Lambda_2$ and $\Theta_1, \Theta_2$ are diagonal matrices with only $1$s and $0$s on their diagonal.
Then let $\mathbb{E}_T\left[ \bU_1 \bU_1^{\top}\right] = \frac{|\bmu_1|}{\tr(\Lambda_1)} \Lambda_1 $, $ \mathbb{E}_T\left[ \bU_2 \bU_2^{\top}\right] = \frac{|\bmu_2|}{\tr(\Lambda_2)} \Lambda_2 $, $\mathbb{E}_S\left[ \bV_1 \bV_1^{\top}\right] = \frac{|\bnu_1|}{\tr(\Theta_1)} \Theta_1 $, $\mathbb{E}_S\left[ \bV_2 \bV_2^{\top}\right] = \frac{|\bnu_2|}{\tr(\Theta_2)} \Theta_2$.
We want to calculate
\begin{align*}
    (\%) = \mathbb{E}_{\bX, \bS, \bT} \left[ \tr\left(
 \bX^{\top}  \bU_1   (\bU_1^{\top} \bX \bV_1)^{\dagger\top} \bV_1^{\top}
 \bV_2(\bU_2^{\top} \bX \bV_2)^{\dagger} \bU_2^{\top}X\right)\right].
\end{align*}
Assume we always have $\Lambda_1 \Lambda_2 \in \{\Lambda_1, \Lambda_2\}$ and $\Theta_1 \Theta \in  \{\Theta_1, \Theta_2\}$.
Define
\begin{align*}
     \Lambda^{\wedge} = \Lambda_1 \Lambda_2,  \;\;  \Lambda^{\vee} =  \Lambda_1 + \Lambda_2 - \Lambda^{\wedge}, \text{ and } \;\; \Theta^{\wedge} = \Theta_1 \Theta_2,  \;\;  \Theta^{\vee} =  \Theta_1 + \Theta_2 - \Theta^{\wedge}.
\end{align*}
Then there holds
\begin{align}\label{equ:lem-bias-diff-b}
(\%) = \begin{cases}
    \frac{ |\bnu_1\cap \bnu_2| |\bnu^c_1 \cap \bnu^c_2| }{\tr(\Lambda^{\vee}) - |\bnu_1\cap \bnu_2| - 1} + |\bnu_1\cap \bnu_2| & \text{ if } |\bnu_1| < |\bmu_1|, |\bnu_2| < |\bmu_2|,\\
    \frac{|\bmu_{1}/\bmu_2||\bmu_2 / \bmu_1|}{\tr(\Theta^{\vee})  - |\bmu_1 \cap \bmu_2|} +  |\bmu_1 \cap \bmu_2| \frac{d - |\bmu_1 \cap \bmu_2| - 1}{\tr(\Theta^{\vee}) - |\bmu_1 \cap \bmu_2| - 1} & \text{ if } |\bnu_1| > |\bmu_1|, |\bnu_2| > |\bmu_2|.
\end{cases}
\end{align}
Moreover, even if $\bmu_1$ and $\bmu_2$ are not independent, and $\bnu_1$ and $\bnu_2$ are not independent, yet we know $\bmu_1 \subset \bmu_2$, $\bnu_1 \subset \bnu_2$,
the quantity becomes
\begin{align}\label{equ:lem-bias-same-b}
(\%) = \begin{cases}
     \frac{|\bnu_1| (d - |\bnu_2|)}{|\bmu_2| - |\bnu_1| - 1} + |\bnu_1| & \text{ if } |\bnu_1| < |\bmu_1|, |\bnu_2| < |\bmu_2|,\\
     \frac{|\bmu_1|(d- |\bnu_2|)}{|\bnu_2| - |\bmu_1| - 1} + |\bmu_1|  & \text{ if } |\bnu_1| > |\bmu_1|, |\bnu_2| > |\bmu_2|.
\end{cases}
\end{align}

\end{lemma}

\begin{proof}[Proof of Lemma \ref{lem:bias-resorted-lemma}]
By definition, $\{\Lambda_1 ,\Lambda_2\} = \{\Lambda^{\wedge}, \Lambda^{\vee}\}$.
W.o.l.g., we assume $(\Lambda_1 ,\Lambda_2) = (\Lambda^{\wedge}, \Lambda^{\vee})$.
Similarly, w.o.l.g., we assume $(\Theta_1,\Theta_2) = (\Theta^{\wedge},\Theta^{\vee} )$.

\textbf{The first circumstance}\;\;
We first write
\begin{align*}
    (\%) = &  \underbrace{\mathbb{E}_{\bX, \bS, \bT} \left[ \tr\left(
  \bV_1^{c}\bV_1^{c\top} \bX^{\top}  \bU_1   (\bU_1^{\top} \bX \bV_1)^{\dagger\top} \bV_1^{\top}
 \bV_2(\bU_2^{\top} \bX \bV_2)^{\dagger} \bU_2^{\top}\bX \bV^c_2 \bV^{c\top}_2\right)\right]}_{(*)} \\
 + & \underbrace{\mathbb{E}_{\bX, \bS, \bT} \left[ \tr\left(
  \bV_1\bV_1^{\top} \bX^{\top}  \bU_1   (\bU_1^{\top} \bX \bV_1)^{\dagger\top} \bV_1^{\top}
 \bV_2(\bU_2^{\top} \bX \bV_2)^{\dagger} \bU_2^{\top}\bX \bV^c_2 \bV^{c\top}_2\right)\right]}_{(**)} \\
 + & \underbrace{\mathbb{E}_{\bX, \bS, \bT} \left[ \tr\left(
  \bV_1^{c}\bV_1^{c\top} \bX^{\top}  \bU_1   (\bU_1^{\top} \bX \bV_1)^{\dagger\top} \bV_1^{\top}
 \bV_2(\bU_2^{\top} \bX \bV_2)^{\dagger} \bU_2^{\top}\bX\bV_2\bV^{\top}_2\right)\right]}_{(***)} \\
 + & \underbrace{\mathbb{E}_{\bX, \bS, \bT} \left[ \tr\left(
  \bV_1\bV_1^{\top} \bX^{\top}  \bU_1   (\bU_1^{\top} \bX \bV_1)^{\dagger\top} \bV_1^{\top}
 \bV_2(\bU_2^{\top} \bX \bV_2)^{\dagger} \bU_2^{\top}\bX\bV_2\bV^{\top}_2\right)\right]}_{(****)}.
\end{align*}
We notice that
\begin{align*}
    (**) = \mathbb{E}_{\bX, \bS, \bT} \left[ \tr\left(
  \bV_1 \bV_1^{\top}
 \bV_2(\bU_2^{\top} \bX \bV_2)^{\dagger} \bU_2^{\top}\bX \bV^c_2\bV^{c\top}_2\right)\right]  = \mathbf{0},
\end{align*}
where we used the fact that $\bX \bV^c_2\bV^{c\top}_2$ is independent of the rest part and has zero mean.
Similarly, $(***) = \mathbf{0}$.
It is also straightforward to see $(****) = |\bnu_1\cap \bnu_2|$.
It remains to evaluate $(*)$.
We continue to use the notation $\bX^{\vee} = \Lambda^{\vee} \bX $ and $\bX^{\wedge} = \Lambda^{\wedge} \bX $.
We first notice that
\begin{align*}
     &  \mathbb{E}_{ T} \left[
     \bU_1   (\bU_1^{\top} \bX \bV_1)^{\dagger\top} \bV_1^{\top}
 \bV_2(\bU_2^{\top} \bX \bV_2)^{\dagger} \bU_2^{\top}\right] \\
 = & \mathbb{E}_{ T} \left[
   (\Pi_{\bX^{\wedge} \bV_1} + \bX^{\wedge} \bV_1 (\bV_1^{\top} \bX^{\wedge\top} \bX^{\wedge} \bV_1)^{-1} \bV_1^{\top} \bX^{\wedge\top})  \bU_1   (\bU_1^{\top} \bX \bV_1)^{\dagger\top} \bV_1^{\top}
 \bV_2(\bU_2^{\top} \bX \bV_2)^{\dagger} \bU_2^{\top}\right].
\end{align*}
The first term is a zero matrix since
\begin{align*}
   \mathbb{E}_{ T} \left[  \Pi_{\bX^{\wedge} \bV_1} \bU_1 (\bU_1^{\top} \bX \bV_1)^{\dagger\top}  \right]  =
   &
   \mathbb{E}_{ T} \left[  \Pi_{\bX^{\wedge} \bV_1} \bU_1 \bU_1^{\top} \bX \bV_1 (\bV_1^{\top} \bX^{\top} \bU_1 \bU_1^{\top} \bX \bV_1)^{-1}  \right]\\
   = &
    \Pi_{\bX^{\wedge} \bV_1}  \bX^{\wedge}  \bV_1 (\bV_1^{\top} \bX^{\top}   \bU_1 \bU_1^{\top} \bX \bV_1)^{-1}  = \mathbf{0},
\end{align*}
since the distribution of $\bU_1\bU_1^{\top}$ is independent of the rest conditioned on we know $|\bmu_1|$ and $|\bmu_1\cap \bmu_2|$.
The second term can be computed as
\begin{align*}
  &  \mathbb{E}_{ T} \left[
   ( \bX^{\wedge} \bV_1 (\bV_1^{\top} \bX^{\wedge\top} \bX^{\wedge} \bV_1)^{-1} \bV_1^{\top} \bX^{\wedge\top}  \bU_1   (\bU_1^{\top} \bX \bV_1)^{\dagger\top} \bV_1^{\top}
 \bV_2(\bU_2^{\top} \bX \bV_2)^{\dagger} \bU_2^{\top}\right] \\
 = &
 \mathbb{E}_{ T} \left[
   ( \bX^{\wedge} \bV_1 (\bV_1^{\top} \bX^{\wedge\top} \bX^{\wedge} \bV_1)^{-1} \bV_1^{\top} \bX^{\top}  \bU_1   (\bU_1^{\top} \bX \bV_1)^{\dagger\top} \bV_1^{\top}
 \bV_2(\bU_2^{\top} \bX \bV_2)^{\dagger} \bU_2^{\top}\right]\\
  = &
 \mathbb{E}_{ T} \left[
   ( \bX^{\wedge} \bV_1 (\bV_1^{\top} \bX^{\wedge\top} \bX^{\wedge} \bV_1)^{-1}  \bV_1^{\top}
 \bV_2(\bU_2^{\top} \bX \bV_2)^{\dagger} \bU_2^{\top}\right],
\end{align*}
where we used \eqref{equ:variance-resorted-lemma-2} in the first equality.
Via exactly the same strategy, the above quantity reduces to
\begin{align*}
    & \mathbb{E}_{ T} \left[
   ( \bX^{\wedge} \bV_1 (\bV_1^{\top} \bX^{\wedge\top} \bX^{\wedge} \bV_1)^{-1}  \bV_1^{\top}
 \bV_2(\bU_2^{\top} \bX \bV_2)^{\dagger} \bU_2^{\top}\right] \\
 = &
 \mathbb{E}_{ T} \left[
   ( \bX^{\wedge} \bV_1 (\bV_1^{\top} \bX^{\wedge\top} \bX^{\wedge} \bV_1)^{-1}  \bV_1^{\top}
 \bV_2
 (\bV_2^{\top} \bX^{\vee\top} \bX^{\vee} \bV_2)^{-1}
 \bV_2^{\top} \bX^{\vee\top}\right] =
 \mathbb{E}_{ T}
 \left[ (\bV_1^{\top} \bX^{\wedge \top})^{\dagger} \bV_1^{\top}
 \bV_2 (\bX^{\vee} \bV_2)^{\dagger}
 \right].
\end{align*}
Bringing this into $(*)$, we get
\begin{align*}
    (*) = \mathbb{E}_{\bX, \bS, \bT} \left[ \tr\left(
  \bV_1^{c}\bV_1^{c\top} \bX^{\top}  (\bV_1^{\top} \bX^{\wedge \top})^{\dagger} \bV_1^{\top}
 \bV_2 (\bX^{\vee} \bV_2)^{\dagger} \bX \bV^c_2\bV^{c\top}_2\right)\right]
\end{align*}
Let $\bV_{1 \cap 2}$, $\bV_{1 \cup 2}$, $\bV_{1 / 2}$, and $\bV_{2/ 1}$ be the selection matrix associated to $\bnu_1 \cap \bnu_2$, $\bnu_1 \cup \bnu_2$, $\bnu_1 / \bnu_2$, and $\bnu_2 / \bnu_1$, respectively.
Since order-invariant permutation is allowed,  we arrive
\begin{align*}
    (*) =  & \mathbb{E}_{\bX,S,T} \left[ \tr( \bV_{1\cup 2}^{c} \bV_{1\cup 2}^{c\top} \bX^{\top}  (\bV_1^{\top} \bX^{\wedge \top})^{\dagger} \bV_1^{\top}
 \bV_2 (\bX^{\vee} \bV_2)^{\dagger} \bX \bV^c_{1\cup 2} \bV^{c\top}_{1\cup 2} )\right]\\
 = & \mathbb{E}_{\bX,S,T} \left[ \tr( \bV_{1\cup 2}^{c} \bV_{1\cup 2}^{c\top} \bX^{\top}  (\bV_1^{\top} \bX^{\wedge \top})^{\dagger} \bV_1^{\top} \bV_{1\cap 2}
 \bV_{1\cap 2}^{\top} \bV_2 (\bX^{\vee} \bV_2)^{\dagger} \bX \bV^c_{1\cup 2} \bV^{c\top}_{1\cup 2} )\right]\\
  = & \mathbb{E}_{\bX,S,T} \left[ \tr( \bV_{1\cup 2}^{c} \bV_{1\cup 2}^{c\top} \bX^{\top}  (\bV_{1\cap 2}^{\top} \bX^{\wedge \top})^{\dagger}  (\bX^{\vee} \bV_{1\cap 2})^{\dagger} \bX \bV^c_{1\cup 2} \bV^{c\top}_{1\cup 2} )\right].
\end{align*}
where in the last step we used the first argument of Lemma \ref{lem:lemmaoflejeune1}.
Note that this is applicable since $\bX \bV^c_{1\cup 2} \bV^{c\top}_{1\cup 2}$ is independent of both $\bX \bV_{1 / 2}$ and $\bX \bV_{2/ 1}$.
We first note that
\begin{align*}
  \mathbb{E}_{\bX \bV^c_{1\cup 2} \bV^{c\top}_{1\cup 2}}\left[ \bX \bV^c_{1\cup 2} \bV^{c\top}_{1\cup 2} \bV_{1\cup 2}^{c} \bV_{1\cup 2}^{c\top} \bX^{\top}\right] = |\bnu^c_1 \cap \bnu^c_2| \cdot I_n.
\end{align*}
Also, there holds
\begin{align*}
     \tr((\bV_{1\cap 2}^{\top} \bX^{\wedge \top})^{\dagger}  (\bX^{\vee} \bV_{1\cap 2})^{\dagger} )= &  \tr((\bV_{1\cap 2}^{\top} \bX^{\wedge \top} \bX^{\wedge }\bV_{1\cap 2} )^{-1} \bV_{1\cap 2}^{\top} \bX^{\wedge \top} \bX^{\vee} \bV_{1\cap 2} (\bV_{1\cap 2}^{\top} \bX^{\vee\top} \bX^{\vee} \bV_{1\cap 2})^{-1}) \\
     = & \tr((\bV_{1\cap 2}^{\top} \bX^{\wedge \top} \bX^{\wedge }\bV_{1\cap 2} )^{-1} \bV_{1\cap 2}^{\top} \bX^{\wedge \top} \bX^{\wedge} \bV_{1\cap 2} (\bV_{1\cap 2}^{\top} \bX^{\vee\top} \bX^{\vee} \bV_{1\cap 2})^{-1})\\
     = & \tr( (\bV_{1\cap 2}^{\top} \bX^{\vee\top} \bX^{\vee} \bV_{1\cap 2})^{-1}) = \frac{|\bnu_1\cap \bnu_2|}{\tr(\Lambda^{\vee}) - |\bnu_1\cap \bnu_2| - 1}.
\end{align*}
They together yield
\begin{align*}
   (*) =  \frac{|\bnu^c_1 \cap \bnu^c_2| \cdot |\bnu_1\cap \bnu_2|}{\tr(\Lambda^{\vee}) - |\bnu_1\cap \bnu_2| - 1}.
\end{align*}
Thus,
\begin{align*}
    (\%) = (*) + ( **** )  = \frac{|\bnu^c_1 \cap \bnu^c_2| \cdot |\bnu_1\cap \bnu_2|}{\tr(\Lambda^{\vee}) - |\bnu_1\cap \bnu_2| - 1} + |\bnu_1\cap \bnu_2|.
\end{align*}

As for the second conclusion, it is still straightforward that $(**) = (***) = 0$. Also, $(****) = |\gamma_1|$.
For $(****)$, we have
\begin{align*}
    (*) = & \mathbb{E}_{\bX, \bS, \bT} \left[ \tr\left(
  \bV_1^{c}\bV_1^{c\top} \bX^{\top}  \bU_1   (\bU_1^{\top} \bX \bV_1)^{\dagger\top} \bV_1^{\top}
 \bV_2(\bU_2^{\top} \bX \bV_2)^{\dagger} \bU_2^{\top}\bX \bV^c_2\bV^{c\top}_2\right)\right] \\
 = & \mathbb{E}_{\bX, \bS, \bT} \left[ \tr\left(
  \bV_2^{c}\bV_2^{c\top} \bX^{\top}  \bU_1   (\bU_1^{\top} \bX \bV_1)^{\dagger\top} \bV_1^{\top}
 \bV_2(\bU_2^{\top} \bX \bV_2)^{\dagger} \bU_2^{\top}\bX \bV^c_2\bV^{c\top}_2\right)\right] \\
  = & \mathbb{E}_{\bX, \bS, \bT} \left[ \tr\left(
  \bV_2^{c}\bV_2^{c\top} \bX^{\top}  \bU_1   (\bU_1^{\top} \bX \bV_1)^{\dagger\top} \bV_1^{\top}
 \bV_1(\bU_2^{\top} \bX \bV_1)^{\dagger} \bU_2^{\top}\bX \bV^c_2\bV^{c\top}_2\right)\right] ,
\end{align*}
where the last follows from Lemma \ref{lem:lemmaoflejeune1}.
Lemma \ref{lem:mplaw} yields
\begin{align*}
   \mathbb{E}_{\bX \bV^c_2}\left[ \bX \bV^c_2\bV^{c\top}_2\bV_2^{c}\bV_2^{c\top} \bX^{\top}\right] = (d - |\bnu_2|) I_n.
\end{align*}
Then,
\begin{align*}
    (*) = & (d - |\bnu_2|) \mathbb{E}_{\bX, \bS, \bT} \left[ \tr\left(
    \bU_1   (\bU_1^{\top} \bX \bV_1)^{\dagger\top} \bV_1^{\top}
 \bV_1(\bU_2^{\top} \bX \bV_1)^{\dagger} \bU_2^{\top}\right)\right] \\
 = & (d - |\bnu_2|) \mathbb{E}_{\bX, \bS, \bT} \left[ \tr\left(
    \bU_1   (\bU_1^{\top} \bX \bV_1)^{\dagger\top} \bV_1^{\top}
 \bV_1(\bU_2^{\top} \bX \bV_1)^{\dagger} \bU_1^{\top}\right)\right] \\
 =&  (d - |\bnu_2|) \mathbb{E}_{\bX, \bS, \bT} \left[ \tr\left(
    \bU_1   \bU_1^{\top} \bX \bV_1 (\bV_1^{\top}\bX^{\top} \bU_1 \bU_1^{\top} \bX \bV_1)^{- 1} \bV_1^{\top}
 \bV_1(\bV_1^{\top}\bX^{\top} \bU_2 \bU_2^{\top} \bX \bV_1)^{- 1} \bV_1^{\top}\bX^{\top}\right)\right] \\
 = & (d - |\bnu_2|) \mathbb{E}_{\bX, \bS, \bT} \left[ \tr\left(
(\bV_1^{\top}\bX^{\top} \bU_2 \bU_2^{\top} \bX \bV_1)^{- 1} \right)\right] =
 \frac{|\bnu_1| (d - |\bnu_2|)}{|\bmu_2| - |\bnu_1| - 1}.
\end{align*}

\textbf{The second circumstance}\;\;
Then we proceed to the second circumstance where $|\bnu_1| > |\bmu_1|, |\bnu_2| > |\bmu_2|$.
This is an analog to the first circumstance, yet we operate on the dual dimension.
We switch the notation to use the notation $\bX^{\vee} =  \bX  \Theta^{\vee}$ and $\bX^{\wedge} =  \bX  \Theta^{\wedge}$.
Let $\bU_{1 \cap 2}$, $\bU_{1 \cup 2}$, $\bU_{1 / 2}$, and $\bU_{2/ 1}$ be the selection matrix associated to $\bmu_1 \cap \bmu_2$, $\bmu_1 \cup \bmu_2$, $\bmu_1 / \bmu_2$, and $\bmu_2 / \bmu_1$, respectively.
We first write
\begin{align*}
    (\%) = &  \underbrace{\mathbb{E}_{\bX, \bS, \bT} \left[ \tr\left(
   \bX^{\top}\bU_2^c \bU_2^{c\top}  \bU_1   (\bU_1^{\top} \bX \bV_1)^{\dagger\top} \bV_1^{\top}
 \bV_2(\bU_2^{\top} \bX \bV_2)^{\dagger} \bU_2^{\top} \bU_1^c \bU_1^{c\top} \bX \right)\right]}_{(*)} \\
 + & \underbrace{\mathbb{E}_{\bX, \bS, \bT} \left[ \tr\left(
   \bX^{\top} \bU_2^c \bU_2^{c\top}  \bU_1   (\bU_1^{\top} \bX \bV_1)^{\dagger\top} \bV_1^{\top}
 \bV_2(\bU_2^{\top} \bX \bV_2)^{\dagger} \bU_2^{\top} \bU_1 \bU_1^{\top}X
 \right)\right]}_{(**)} \\
 + &
 \underbrace{\mathbb{E}_{\bX, \bS, \bT} \left[ \tr\left(
   \bX^{\top} \bU_2 \bU_2^{\top}  \bU_1   (\bU_1^{\top} \bX \bV_1)^{\dagger\top} \bV_1^{\top}
 \bV_2(\bU_2^{\top} \bX \bV_2)^{\dagger} \bU_2^{\top} \bU_1^c \bU_1^{c\top}X
 \right)\right]}_{(***)} \\
 + &\underbrace{\mathbb{E}_{\bX, \bS, \bT} \left[ \tr\left(
   \bX^{\top}\bU_2 \bU_2^{\top}  \bU_1   (\bU_1^{\top} \bX \bV_1)^{\dagger\top} \bV_1^{\top}
 \bV_2(\bU_2^{\top} \bX \bV_2)^{\dagger} \bU_2^{\top} \bU_1 \bU_1^{\top}\bX \right)\right]}_{(****)}
\end{align*}

$(i)$
We begin with $(*)$.
Due to a similar reasoning in the first circumstance, we know that $\mathbb{E}_{S}\left[(\bU_1^{\top} \bX \bV_1)^{\dagger\top} \bV_1^{\top} \Pi_{\bX^{\vee\top} \bU_1}\right] = \mathbf{0}$.
Thus, we can do the following transformation
\begin{align*}
    (*) = &  \mathbb{E}_{\bX, \bS, \bT} \left[ \tr\left(
   \bX^{\top}\bU_2^c \bU_2^{c\top}  \bU_1   (\bU_1^{\top} \bX \bV_1)^{\dagger\top} \bV_1^{\top}
  \bX^{\vee\top} \bU_1 (\bU_1^{\top} \bX^\vee \bX^{\vee\top} \bU_1)^{-1} \bU_1^{\top} \bX^\vee
 \bV_2(\bU_2^{\top} \bX \bV_2)^{\dagger} \bU_2^{\top}\bU_1^c \bU_1^{c\top} \bX \right)\right]  \\
 = & \mathbb{E}_{\bX, \bS, \bT} \left[ \tr\left(
   \bX^{\top}\bU_2^c \bU_2^{c\top}  \bU_1    (\bU_1^{\top} \bX^\vee \bX^{\vee\top} \bU_1)^{-1} \bU_1^{\top} \bX^\vee
 \bV_2(\bU_2^{\top} \bX \bV_2)^{\dagger} \bU_2^{\top} \bU_1^c \bU_1^{c\top} \bX \right)\right] .
\end{align*}
Via exactly the same strategy, the above quantity reduces to
\begin{align}\nonumber
   (*) = &\mathbb{E}_{\bX, \bS, \bT} \left[ \tr\left(
   \bX^{\top}\bU_2^c \bU_2^{c\top}  \bU_1    (\bU_1^{\top} \bX^\vee \bX^{\vee\top} \bU_1)^{-1} \bU_1^{\top} \bX^\vee
 \bX^{\wedge\top} \bU_2 (\bU_2^{\top} \bX^\wedge \bX^{\wedge\top} \bU_2)^{-1} \bU_2^{\top} \bU_1^c \bU_1^{c\top} \bX \right)\right] \\
 = & \mathbb{E}_{\bX, \bS, \bT} \left[ \tr\left(
   \bX^{\top}\bU_2^c \bU_2^{c\top}  \bU_1    (\bU_1^{\top} \bX^\vee \bX^{\vee\top} \bU_1)^{-1} \bU_1^{\top} \bX^\wedge
 \bX^{\wedge\top} \bU_2 (\bU_2^{\top} \bX^\wedge \bX^{\wedge\top} \bU_2)^{-1} \bU_2^{\top} \bU_1^c \bU_1^{c\top} \bX \right)\right].  \label{equ:bias-resorted-lemma-4}
\end{align}
We also note that
\begin{align*}
    \mathbb{E}_{\bX \Theta^{\vee c}} \left[\bU_2^{\top} \bX \bX^{\top}\bU_2^c \bU_2^{c\top} \right] = & \bU_2^{\top} \bX^{\vee} \bX^{\vee\top}\bU_2^c \bU_2^{c\top}  +
\mathbb{E}_{\bX \Theta^{\vee c}} \left[\bU_2^{\top} \bX \Theta^{\vee c} \Theta^{\vee c}  \bX^{\top} \bU_2^c \bU_2^{c\top} \right] \\
= & \bU_2^{\top} \bX^{\vee} \bX^{\vee\top}\bU_2^c \bU_2^{c\top}  + \tr(\Theta^{\vee c})
\mathbb{E}_{\bX \Theta^{\vee c}} \left[\bU_2^{\top}  \bU_2^c \bU_2^{c\top} \right] = \bU_2^{\top} \bX^{\vee} \bX^{\vee\top}\bU_2^c \bU_2^{c\top}.
\end{align*}
This means we can turn the $\bX$s on the side into
\begin{align*}
   (*) =   \mathbb{E}_{\bX, \bS, \bT} \left[ \tr\left(
   \bX^{\vee\top}\bU_2^c \bU_2^{c\top}  \bU_1    (\bU_1^{\top} \bX^\vee \bX^{\vee\top} \bU_1)^{-1} \bU_1^{\top} \bX^\wedge
 \bX^{\wedge\top} \bU_2 (\bU_2^{\top} \bX^\wedge \bX^{\wedge\top} \bU_2)^{-1} \bU_2^{\top} \bU_1^c \bU_1^{c\top} \bX^{\vee} \right)\right].
\end{align*}
We further decompose $I_n = \bU_2 \bU_2^{\top} + \bU_2^c \bU_2^{c\top}$ and have
\begin{align*}
       (*) =   & \mathbb{E}_{\bX, \bS, \bT} \left[ \tr\left(
   \bX^{\vee\top}\bU_2^c \bU_2^{c\top}  \bU_1    (\bU_1^{\top} \bX^\vee \bX^{\vee\top} \bU_1)^{-1} \bU_1^{\top}  \bU_2 \bU_2^{\top} \bX^\wedge
 \bX^{\wedge\top} \bU_2 (\bU_2^{\top} \bX^\wedge \bX^{\wedge\top} \bU_2)^{-1} \bU_2^{\top} \bU_1^c \bU_1^{c\top} \bX^{\vee} \right)\right] \\
 + &   \mathbb{E}_{\bX, \bS, \bT} \left[ \tr\left(
   \bX^{\vee\top}\bU_2^c \bU_2^{c\top}  \bU_1    (\bU_1^{\top} \bX^\vee \bX^{\vee\top} \bU_1)^{-1} \bU_1^{\top} \bU_2^c \bU_2^{c\top} \bX^\wedge
 \bX^{\wedge\top} \bU_2 (\bU_2^{\top} \bX^\wedge \bX^{\wedge\top} \bU_2)^{-1} \bU_2^{\top} \bU_1^c \bU_1^{c\top} \bX^{\vee} \right)\right].
\end{align*}
The first term becomes
\begin{align}\nonumber
    & \mathbb{E}_{\bX, \bS, \bT} \left[ \tr\left(
   \bX^{\vee\top}\bU_2^c \bU_2^{c\top}  \bU_1    (\bU_1^{\top} \bX^\vee \bX^{\vee\top} \bU_1)^{-1} \bU_1^{\top}  \bU_2 \bU_2^{\top} \bX^\wedge
 \bX^{\wedge\top} \bU_2 (\bU_2^{\top} \bX^\wedge \bX^{\wedge\top} \bU_2)^{-1} \bU_2^{\top} \bU_1^c \bU_1^{c\top} \bX^{\vee} \right)\right] \\
 = & \mathbb{E}_{\bX, \bS, \bT} \left[ \tr\left(
   \bX^{\vee\top}\bU_2^c \bU_2^{c\top}  \bU_1    (\bU_1^{\top} \bX^\vee \bX^{\vee\top} \bU_1)^{-1} \bU_1^{\top}  \bU_2  \bU_2^{\top} \bU_1^c \bU_1^{c\top} \bX^{\vee} \right)\right]  = \mathbf{0}.
   \label{equ:bias-resorted-lemma-1}
\end{align}
For the second term, we write $\bU_{1\cap 2}^{\top}\bX^{\vee} =  \bX^{\vee}_{1\cap 2} $, $\bU_{2/ 1}^{\top}\bX^{\vee} =  \bX^{\vee}_{2/1} $, $\bU_{1/ 2}^{\top}\bX^{\vee} =  \bX^{\vee}_{1/ 2} $, $\bU_{1\cup 2}^{c\top}\bX^{\vee} =  \bX^{\vee}_{1^c\cap 2^c} $.
W.o.l.g., we assume $(\bmu_1\cap \bmu_2, \bmu_1 / \bmu_2, \bmu_2 / \bmu_1,  \bmu_1^c\cap \bmu_2^c) = (1,\ldots,d)$, which means $\bX = (\bX^{\vee\top}_{1\cap 2}, \bX^{\vee\top}_{1/ 2}, \bX^{\vee\top}_{2/1}, \bX^{\vee\top}_{1^c\cap 2^c})^{\top}$.
Then we have
\begin{align*}
\bU_1^{\top} \bX^\vee \bX^{\vee\top} \bU_1 = \left(\begin{matrix}
X^{\vee}_{1\cap 2} \bX^{\vee\top}_{1\cap 2} &  \bX^{\vee}_{1\cap 2} \bX^{\vee\top}_{1/ 2} \\
X^{\vee}_{1/ 2} \bX^{\vee\top}_{1\cap 2}& \bX^{\vee}_{1/ 2} \bX^{\vee\top}_{1/ 2}\\
\end{matrix}\right).
\end{align*}
We can compute
\begin{align}\nonumber
 & \bX^{\vee\top}\bU_2^c \bU_2^{c\top}  \bU_1    (\bU_1^{\top} \bX^{\vee} \bX^{\vee\top} \bU_1)^{-1} \bU_1^{\top} \bU_2^c \bU_2^{c\top} \bX^{\vee} \\
 = & \bX^{\vee\top}_{1/ 2}    (\bU_1^{\top} \bX^{\vee} \bX^{\vee\top} \bU_1)^{-1}  \bX^{\vee}_{1/ 2}\nonumber \\
 = & \bX^{\vee\top}_{1/ 2}    (\bX^{\vee}_{1/ 2} \bX^{\vee\top}_{1/ 2} - \bX^{\vee}_{1/ 2} \bX^{\vee\top}_{1\cap 2} (\bX^{\vee}_{1\cap 2} \bX^{\vee\top}_{1\cap 2})^{-1} \bX^{\vee}_{1\cap 2} \bX^{\vee\top}_{1/ 2})^{-1}  \bX^{\vee}_{1/ 2} = \bX^{\vee\top}_{1/ 2}    (\bX^{\vee}_{1/ 2} \Pi_{\bX^{\vee\top}_{1\cap 2}}\bX^{\vee\top}_{1/ 2} )^{-1}  \bX^{\vee}_{1/ 2}. \label{equ:bias-resorted-lemma-3}
\end{align}
Since $\bX^{\vee}_{1/ 2}$ has i.i.d. standard Gaussian entries,
decomposing it into $\bX^{\vee\top}_{1/ 2} = \Pi_{\bX^{\vee\top}_{1\cap 2}}\bX^{\vee\top}_{1/ 2} + \Pi_{\bX^{\vee c }_{1\cap 2}}\bX^{\vee\top}_{1/ 2}$ yields the following, where the cross terms disappear.
\begin{align*}
    & \mathbb{E}_{\bX^{\vee\top}_{1/ 2}} \left[X^{\vee\top}_{1/ 2}    (\bX^{\vee}_{1/ 2} \Pi_{\bX^{\vee\top}_{1\cap 2}}\bX^{\vee\top}_{1/ 2} )^{-1}  \bX^{\vee}_{1/ 2}\right]\\
    = & \mathbb{E}_{\bX^{\vee\top}_{1/ 2}} \left[ \Pi_{\bX^{\vee }_{1\cap 2}} \bX^{\vee\top}_{1/ 2}    (\bX^{\vee}_{1/ 2} \Pi_{\bX^{\vee\top}_{1\cap 2}}\bX^{\vee\top}_{1/ 2} )^{-1}  \bX^{\vee}_{1/ 2} \Pi_{\bX^{\vee}_{1\cap 2}}\right]  + \mathbb{E}_{\bX^{\vee\top}_{1/ 2}} \left[ \Pi_{\bX^{\vee c}_{1\cap 2}} \bX^{\vee\top}_{1/ 2}    (\bX^{\vee}_{1/ 2} \Pi_{\bX^{\vee\top}_{1\cap 2}}\bX^{\vee\top}_{1/ 2} )^{-1}  \bX^{\vee}_{1/ 2} \Pi_{\bX^{\vee c}_{1\cap 2}}\right].
\end{align*}
The first term has
\begin{align*}
    \mathbb{E}_{\bX^{\vee\top}_{1/ 2}} \left[ \Pi_{\Pi_{\bX^{\vee\top}_{1\cap 2}}\bX^{\vee\top}_{1/ 2}} \right]  = \frac{|\bmu_{1}/\bmu_2|}{\tr(\Theta^{\vee}) - |\bmu_1 \cap \bmu_2|} \Pi_{\bX^{\vee }_{1\cap 2}},
\end{align*}
which is by the fact that $\Pi_{\bX^{\vee }_{1\cap 2}} \bX^{\vee\top}_{1/ 2}$ is distributed as $\mathcal{N}(0, \Pi_{\bX^{\vee }_{1\cap 2}} )$.
The second term has
\begin{align*}
    & \mathbb{E}_{\bX^{\vee\top}_{1/ 2}} \left[ \Pi_{\bX^{\vee c}_{1\cap 2}} \bX^{\vee\top}_{1/ 2}    (\bX^{\vee}_{1/ 2} \Pi_{\bX^{\vee\top}_{1\cap 2}}\bX^{\vee\top}_{1/ 2} )^{-1}  \bX^{\vee}_{1/ 2} \Pi_{\bX^{\vee c}_{1\cap 2}}\right]\\
    = & \mathbb{E}_{\bX^{\vee\top}_{1/ 2}} \left[ \Pi_{\bX^{\vee c}_{1\cap 2}} \bX^{\vee\top}_{1/ 2}    \frac{I_{|\bmu_1/\bmu_2|}}{ \tr(\Theta^{\vee})  - |\bmu_1\cap \bmu_2|- |\bmu_1/\bmu_2| - 1} \bX^{\vee}_{1/ 2} \Pi_{\bX^{\vee c}_{1\cap 2}}\right]\\
    = & \Pi_{\bX^{\vee c}_{1\cap 2}} \frac{ |\bmu_1/\bmu_2|}{\tr(\Theta^{\vee})  - |\bmu_1\cap \bmu_2| - |\bmu_1/\bmu_2| - 1}
\end{align*}
due to Lemma \ref{lem:mplaw}.
This means the second term of $(*)$ is
\begin{align}\nonumber
  (\cdot):= &\frac{|\bmu_{1}/\bmu_2|}{\tr(\Theta^{\vee})  - |\bmu_1 \cap \bmu_2|} \Pi_{\bX^{\vee }_{1\cap 2}} + \Pi_{\bX^{\vee c}_{1\cap 2}} \frac{ |\bmu_1/\bmu_2|}{\tr(\Theta^{\vee})  - |\bmu_1\cap \bmu_2| - |\bmu_1/\bmu_2| - 1} \\
   = & \frac{|\bmu_{1}/\bmu_2|}{\tr(\Theta^{\vee})  - |\bmu_1 \cap \bmu_2|} I_{\tr(\Theta^{\vee}) } + \Pi_{\bX^{\vee c}_{1\cap 2}} \frac{ |\bmu_1/\bmu_2|}{(\tr(\Theta^{\vee})  - |\bmu_1\cap \bmu_2| - |\bmu_1/\bmu_2| - 1)(\tr(\Theta^{\vee})  - |\bmu_1 \cap \bmu_2|)}.
   \label{equ:bias-resorted-lemma-2}
\end{align}
Bringing this into $(*)$, we get
\begin{align*}
    (*) =   \mathbb{E}_{\bX, \bS, \bT} \left[ \tr\left( (\cdot)
 \bX^{\wedge\top} \bU_2 (\bU_2^{\top} \bX^\wedge \bX^{\wedge\top} \bU_2)^{-1} \bU_2^{\top} \bU_1^c \bU_1^{c\top} \bX^{\vee} \right)\right].
\end{align*}
Note that $\bU_1^c \bU_1^{c\top} \bX^{\vee}\Pi_{\bX^{\vee c}_{1\cap 2}} = \mathbf{0}$, and thus
\begin{align*}
    (*) =  \frac{|\bmu_{1}/\bmu_2|}{\tr(\Theta^{\vee})  - |\bmu_1 \cap \bmu_2|}  \mathbb{E}_{\bX, \bS, \bT} \left[ \tr\left(
 \bX^{\wedge\top} \bU_2 (\bU_2^{\top} \bX^\wedge \bX^{\wedge\top} \bU_2)^{-1} \bU_2^{\top} \bU_1^c \bU_1^{c\top} \bX^{\vee} \right)\right].
\end{align*}
Note that we can change the last $\bX^{\vee}$ into $\bX^{\wedge}$ since trace operator can permutate and $\bX^{\vee} \bX^{\wedge\top} = \bX^{\wedge} \bX^{\wedge\top}$.
Following the same notations, we have
\begin{align*}
&  \bX^{\wedge\top} \bU_2 (\bU_2^{\top} \bX^\wedge \bX^{\wedge\top} \bU_2)^{-1} \bU_2^{\top} \bU_1^c \bU_1^{c\top} \bX^{\wedge} \\
=  & -  \bX^{\wedge\top}_{1 \cap 2} (\bX^{\wedge}_{1 \cap 2} \bX^{\wedge\top}_{1 \cap 2})^{-1}  \bX^{\wedge}_{1 \cap 2} \bX^{\wedge\top}_{2 / 1}    (\bX^{\wedge}_{2 / 1} \bX^{\wedge\top}_{2 / 1} - \bX^{\wedge}_{2 / 1} \bX^{\wedge\top}_{1\cap 2} (\bX^{\wedge}_{1\cap 2} \bX^{\wedge\top}_{1\cap 2})^{-1} \bX^{\wedge}_{1\cap 2} \bX^{\wedge\top}_{2 / 1})^{-1}  \bX^{\wedge}_{2 / 1} \\
 & +  \bX^{\wedge\top}_{2 / 1}    (\bX^{\wedge}_{2 / 1} \bX^{\wedge\top}_{2 / 1} - \bX^{\wedge}_{2 / 1} \bX^{\wedge\top}_{1\cap 2} (\bX^{\wedge}_{1\cap 2} \bX^{\wedge\top}_{1\cap 2})^{-1} \bX^{\wedge}_{1\cap 2} \bX^{\wedge\top}_{2 / 1})^{-1}  \bX^{\wedge}_{2 / 1}\\
= & \Pi_{\bX^{\wedge\top}_{1 \cap 2}} \bX^{\wedge\top}_{2 / 1}    (\bX^{\wedge}_{2 / 1} \Pi_{\bX^{\wedge\top}_{1 \cap 2}} \bX^{\wedge\top}_{2 / 1})^{-1}  \bX^{\wedge}_{2 / 1},
\end{align*}
whose trace is
\begin{align*}
& \mathbb{E}_{\bX, \bS, \bT} \left[ \tr\left(
X^{\wedge\top} \bU_2 (\bU_2^{\top} \bX^\wedge \bX^{\wedge\top} \bU_2)^{-1} \bU_2^{\top} \bU_1^c \bU_1^{c\top} \bX^{\wedge} \right)\right]\\
= &
\mathbb{E}_{\bX, \bS, \bT} \left[ \tr\left(
\Pi_{\bX^{\wedge\top}_{1 \cap 2}} \bX^{\wedge\top}_{2 / 1}    (\bX^{\wedge}_{2 / 1} \Pi_{\bX^{\wedge\top}_{1 \cap 2}} \bX^{\wedge\top}_{2 / 1})^{-1}  \bX^{\wedge}_{2 / 1} \Pi_{\bX^{\wedge\top}_{1 \cap 2}}\right)\right] = |\bmu_2 / \bmu_1|.
\end{align*}
Thus, we reach
\begin{align*}
(*) = \frac{|\bmu_{1}/\bmu_2||\bmu_2 / \bmu_1|}{\tr(\Theta^{\vee})  - |\bmu_1 \cap \bmu_2|}.
\end{align*}

$(ii)$
For $(**)$, notice that the derivations till \eqref{equ:bias-resorted-lemma-1} and \eqref{equ:bias-resorted-lemma-2} are still valid, except we need to replace $\bU_1^c \bU_1^{c\top}$ by $\bU_1 \bU_1^{\top}$, namely
\begin{align*}
    (**) = & \mathbb{E}_{\bX, \bS, \bT} \left[ \tr\left(
   \bX^{\vee\top}\bU_2^c \bU_2^{c\top}  \bU_1    (\bU_1^{\top} \bX^\vee \bX^{\vee\top} \bU_1)^{-1} \bU_1^{\top}  \bU_2  \bU_2^{\top} \bU_1 \bU_1^{\top} \bX^{\vee} \right)\right]\\
    + &  \frac{ |\bmu_1/\bmu_2|}{\tr(\Theta^{\vee})  - |\bmu_1\cap \bmu_2| - |\bmu_1/\bmu_2| - 1}  \mathbb{E}_{\bX, \bS, \bT} \left[ \tr\left(
 \bX^{\wedge\top} \bU_2 (\bU_2^{\top} \bX^\wedge \bX^{\wedge\top} \bU_2)^{-1} \bU_2^{\top} \bU_1 \bU_1^{\top} \bX^{\vee} \right)\right].
\end{align*}
Under the same notation, we can compute
\begin{align*}
 \bX^{\vee\top}\bU_2^c \bU_2^{c\top}  \bU_1    (\bU_1^{\top} \bX^\vee \bX^{\vee\top} \bU_1)^{-1} \bU_1^{\top}  \bU_2  \bU_2^{\top} \bX^{\vee}
=   - \bX^{\vee\top}_{1 / 2}    (\bX^{\vee}_{1 / 2} \Pi_{\bX^{\vee\top}_{1 \cap 2}}\bX^{\vee\top}_{1 / 2})^{-1}  \bX^{\vee}_{1 / 2} \Pi_{\bX^{\vee c\top}_{1 \cap 2}}
\end{align*}
and
\begin{align*}
 & \bX^{\wedge\top} \bU_2 (\bU_2^{\top} \bX^\wedge \bX^{\wedge\top} \bU_2)^{-1} \bU_2^{\top} \bU_1 \bU_1^{\top} \bX^{\vee} \\
=  &  \bX^{\wedge\top}_{1 \cap 2} (\bX^{\wedge}_{1 \cap 2} \bX^{\wedge\top}_{1 \cap 2})^{-1}  \bX^{\wedge}_{1 \cap 2}  \\
+ &   \bX^{\wedge\top}_{1 \cap 2} (\bX^{\wedge}_{1 \cap 2} \bX^{\wedge\top}_{1 \cap 2})^{-1}  \bX^{\wedge}_{1 \cap 2} \bX^{\wedge\top}_{2 / 1}    (\bX^{\wedge}_{2 / 1} \Pi_{\bX^{\wedge\top}_{1\cap 2}} \bX^{\wedge\top}_{2 / 1} )^{-1}  \bX^{\wedge}_{2 / 1}  \bX^{\wedge \top}_{1 \cap 2} ( \bX^{\wedge }_{1 \cap 2}  \bX^{\wedge \top}_{1 \cap 2})^{-1}  \bX^{\wedge }_{1 \cap 2} \\
 - &   \bX^{\wedge\top}_{2 / 1}    (\bX^{\wedge}_{2 / 1} \Pi_{\bX^{\wedge\top}_{1\cap 2}} \bX^{\wedge\top}_{2 / 1} )^{-1}  \bX^{\wedge}_{2 / 1}  \bX^{\wedge \top}_{1 \cap 2} ( \bX^{\wedge }_{1 \cap 2}  \bX^{\wedge \top}_{1 \cap 2})^{-1}  \bX^{\wedge }_{1 \cap 2}\\
= & \bX^{\wedge\top}_{1 \cap 2} (\bX^{\wedge}_{1 \cap 2} \bX^{\wedge\top}_{1 \cap 2})^{-1}  \bX^{\wedge}_{1 \cap 2} + \Pi_{\bX^{\wedge\top}_{1\cap 2}}
X^{\wedge\top}_{2 / 1}    (\bX^{\wedge}_{2 / 1} \Pi_{\bX^{\wedge\top}_{1\cap 2}} \bX^{\wedge\top}_{2 / 1} )^{-1}  \bX^{\wedge}_{2 / 1}  \Pi_{\bX^{\wedge c\top}_{1\cap 2}}.
\end{align*}
Bringing in the same calculations as in \eqref{equ:bias-resorted-lemma-3} yields $(**) = 0$.

$(iii)$
The calculation of $(***)$ follows symmetrically to $(*)$.

$(iv)$
We directly apply Lemma \ref{lem:lemmaoflejeune1} to have
\begin{align*}
   &  \mathbb{E}_{\bX, \bS, \bT} \left[ \tr\left(
   \bX^{\top}\bU_2 \bU_2^{\top}  \bU_1   (\bU_1^{\top} \bX \bV_1)^{\dagger\top} \bV_1^{\top}
 \bV_2(\bU_2^{\top} \bX \bV_2)^{\dagger} \bU_2^{\top} \bU_1 \bU_1^{\top}\bX \right)\right] \\
 = & \mathbb{E}_{\bX, \bS, \bT} \left[ \tr\left(
   \bX^{\top}\bU_{1\cap 2}   (\bU_{1\cap 2}^{\top} \bX \bV_1)^{\dagger\top} \bV_1^{\top}
 \bV_2(\bU_{1\cap 2}^{\top} \bX \bV_2)^{\dagger}  \bU_{1\cap 2}^{\top}\bX \right)\right],
\end{align*}
which is followed by the transformations similar to \eqref{equ:bias-resorted-lemma-4} to have
\begin{align*}
    (****)
 =  & \mathbb{E}_{\bX, \bS, \bT} \left[ \tr\left(
   \bX^{\top}\bU_{1\cap 2} (\bU_{1\cap 2}^{\top} \bX^{\vee} \bX^{\vee\top} \bU_{1\cap 2})^{-1}  \bU_{1\cap 2}^{\top}\bX \right)\right] \\
   = & \mathbb{E}_{\bX, \bS, \bT} \left[ \tr\left(
   \bX^{\vee\top}\bU_{1\cap 2} (\bU_{1\cap 2}^{\top} \bX^{\vee} \bX^{\vee\top} \bU_{1\cap 2})^{-1}  \bU_{1\cap 2}^{\top}\bX^{\vee} \right)\right]\\
   + & \mathbb{E}_{\bX, \bS, \bT} \left[ \tr\left(
   \bX^{\top}\bU_{1\cap 2} (\bU_{1\cap 2}^{\top} \bX^{\vee} \bX^{\vee\top} \bU_{1\cap 2})^{-1}  \bU_{1\cap 2}^{\top}\bX (I_d - \Theta^{\vee} )\right)\right].
\end{align*}
The first term is $|\bmu_1 \cap \bmu_2|$, while the second is
\begin{align*}
    \frac{(d - \tr(\Theta^{\vee})) |\bmu_1 \cap \bmu_2|}{\tr(\Theta^{\vee}) - |\bmu_1 \cap \bmu_2| - 1 }.
\end{align*}
They together yield
\begin{align*}
    (****) = |\bmu_1 \cap \bmu_2| \frac{d - |\bmu_1 \cap \bmu_2| - 1}{\tr(\Theta^{\vee}) - |\bmu_1 \cap \bmu_2| - 1}.
\end{align*}
Thus,
\begin{align*}
    (\%) = (*) + (****) = \frac{|\bmu_{1}/\bmu_2||\bmu_2 / \bmu_1|}{\tr(\Theta^{\vee})  - |\bmu_1 \cap \bmu_2|} +  |\bmu_1 \cap \bmu_2| \frac{d - |\bmu_1 \cap \bmu_2| - 1}{\tr(\Theta^{\vee}) - |\bmu_1 \cap \bmu_2| - 1}.
\end{align*}

For the second conclusion, we notice that $(*) = (**) = (***) = 0$ is obvious.
For $(****)$, we have
\begin{align*}
    (****) = & \mathbb{E}_{\bX, \bS, \bT} \left[ \tr\left(
   \bX^{\top}  \bU_1   (\bU_1^{\top} \bX \bV_1)^{\dagger\top} \bV_1^{\top}
 \bV_1(\bU_2^{\top} \bX \bV_2)^{\dagger} \bU_2^{\top} \bU_1  \bU_1^{\top}\bX \right)\right] \\
 = &  \mathbb{E}_{\bX, \bS, \bT} \left[ \tr\left(
   \bX^{\top}  \bU_1   (\bU_1^{\top} \bX \bV_1)^{\dagger\top} \bV_1^{\top}
 \bV_1(\bU_1^{\top} \bX \bV_2)^{\dagger} \bU_1^{\top}\bX \right)\right]
\end{align*}
by Lemma \ref{lem:lemmaoflejeune1}, from where we proceed to have
\begin{align*}
     (****) = & \mathbb{E}_{\bX, \bS, \bT} \left[ \tr\left(
   \bX^{\top}  \bU_1   (\bU_1^{\top} \bX \bV_1 \bV_1^{\top} \bX^{\top} \bU_1)^{-1} \bU_1^{\top} \bX \bV_1 \bV_1^{\top} \bX^{\top} \bU_1
 (\bU_1^{\top} \bX \bV_2\bV_2^{\top} \bX^{\top} \bU_1 )^{-1} \bU_1^{\top}\bX \right)\right] \\
 = & \mathbb{E}_{\bX, \bS, \bT} \left[ \tr\left(
   \bX^{\top}  \bU_1
 (\bU_1^{\top} \bX \bV_2\bV_2^{\top} \bX^{\top} \bU_1 )^{-1} \bU_1^{\top}\bX \right)\right] \\
 = & |\bmu_1| + \frac{|\bmu_1|(d- |\bnu_2|)}{|\bnu_2| - |\bmu_1| - 1}.
\end{align*}
\end{proof}

\subsection{Proofs for Appendix~\ref{sec:error-analysis}}\label{sec:prooflinearstability}

\begin{proof}[Proof of Proposition \ref{prop:varianceofstability}]
    Recall that
\begin{align*}
& \tr\left(\left(g(\bx) - g'(\bx)\right)^{\top} \left(g(\bx) - g'(\bx)\right) \right) \\
= & \frac{1}{B^2}  \sum_{b_1, b_2=1}^B \underbrace{\tr\left(\left( \bV_{b_1}(\bU_{b_1}^{\top} \bX \bV_{b_1})^{\dagger} \bU_{b_1}^{\top} -
 \bV_{b_1}^{\prime}(\bU_{b_1}^{\prime\top} \bX \bV_{b_1}^{\prime})^{\dagger} \bU_{b_1}^{\prime\top}\right)^{\top}
 \left( \bV_{b_2}(\bU_{b_2}^{\top} \bX \bV_{b_2})^{\dagger} \bU_{b_2}^{\top} -
 \bV_{b_2}^{\prime}(\bU_{b_2}^{\prime\top} \bX \bV_{b_2}^{\prime})^{\dagger} \bU_{b_2}^{\prime\top}\right)\right)}_{V_{b_1,b_2}}.
\end{align*}

There are $B(B-1)$ cross terms which have $b_1\neq b_2$, and $B$ squared terms with $b_1 = b_2$.

\textbf{Underparameterized case} \;\;
For instance removal, the squared terms become
\begin{align*}
V_{b,b} =     \tr\left(\left( \bV_{b}(\bU_{b}^{\top} \bX \bV_{b})^{\dagger} \bU_{b}^{\top} -
 \bV_{b}(\bU_{b}^{\prime\top} \bX \bV_{b})^{\dagger} \bU_{b}^{\prime\top}\right)^{\top}
 \left( \bV_{b}(\bU_{b}^{\top} \bX \bV_{b})^{\dagger} \bU_{b}^{\top} -
 \bV_{b}(\bU_{b}^{\prime\top} \bX \bV_{b})^{\dagger} \bU_{b}^{\prime\top}\right)\right),
\end{align*}
where $\bU_b$ and $\bU_b'$ are independent.
We apply \eqref{equ:lem-variance-same-b} in Lemma \ref{lem:variance-resorted-lemma} to have
\begin{align*}
    &\mathbb{E}_{\bX, \bS, \bT} \left[\tr\left(\left( \bV_{b}(\bU_{b}^{\top} \bX \bV_{b})^{\dagger} \bU_{b}^{\top} \right)^{\top}
 \left( \bV_{b}(\bU_{b}^{\top} \bX \bV_{b})^{\dagger} \bU_{b}^{\top} \right)\right)\right] \\
 = &
 \mathbb{E}_{\bX, \bS, \bT} \left[\tr\left(\left(
 \bV_{b}(\bU_{b}^{\prime\top} \bX \bV_{b})^{\dagger} \bU_{b}^{\prime\top}\right)^{\top}
 \left(
 \bV_{b}(\bU_{b}^{\prime\top} \bX \bV_{b})^{\dagger} \bU_{b}^{\prime\top}\right)\right)\right]
 = \frac{|\bnu_{b}|}{ |\bmu_{b}|-  |\bnu_{b} | - 1}.
\end{align*}
For the other two cross terms, we treat them as only have $\bnu_b$ columns, and apply \eqref{equ:lem-variance-diff-b} in Lemma \ref{lem:variance-resorted-lemma} to have
\begin{align*}
    \mathbb{E}_{\bX, \bS, \bT} \left[\tr\left(\left( \bV_{b}(\bU_{b}^{\top} \bX \bV_{b})^{\dagger} \bU_{b}^{\top} \right)^{\top}
 \left( \bV_{b}(\bU_{b}^{\prime\top} \bX \bV_{b})^{\dagger} \bU_{b}^{\prime\top} \right)\right)\right]
 =  \frac{|\bnu_b|}{ n-|\bnu_b| - 1}.
\end{align*}
Combined together, this is
\begin{align*}
     \mathbb{E}_{\bX, \bS, \bT} \left[V_{b,b}\right]  = 2 \frac{|\bnu_b|}{ |\bmu_b|-  |\bnu_b | - 1} - 2 \frac{|\bnu_b|}{ n-|\bnu_b| - 1}.
\end{align*}
Taking the limit, this is
\begin{align*}
 \lim_{n,d\to \infty} \mathbb{E}_{\bX, \bS, \bT} \left[V_{b,b}\right] = 2 \frac{ \gamma q}{p - \gamma q} - 2 \frac{ \gamma q}{1 - \gamma q} = \frac{2 \gamma q (1 - p)}{(p - \gamma q)(1 - \gamma q)}.
\end{align*}

For feature removal, the squared term becomes
\begin{align*}
V_{b,b} =     \tr\left(\left( \bV_{b}(\bU_{b}^{\top} \bX \bV_{b})^{\dagger} \bU_{b}^{\top} -
 \bV_{b}^\prime(\bU_{b}^{\top} \bX \bV_{b}^\prime)^{\dagger} \bU_{b}^{\top}\right)^{\top}
 \left( \bV_{b}(\bU_{b}^{\top} \bX \bV_{b})^{\dagger} \bU_{b}^{\top} -
 \bV_{b}^\prime(\bU_{b}^{\top} \bX \bV_{b}^{\prime})^{\dagger} \bU_{b}^{\top}\right)\right),
\end{align*}
where $\bV_b$ and $\bV_b'$ are independent.
We apply \eqref{equ:lem-variance-same-b} in Lemma \ref{lem:variance-resorted-lemma} to have
\begin{align*}
    &\mathbb{E}_{\bX, \bS, \bT} \left[\tr\left(\left( \bV_{b}(\bU_{b}^{\top} \bX \bV_{b})^{\dagger} \bU_{b}^{\top} \right)^{\top}
 \left( \bV_{b}(\bU_{b}^{\top} \bX \bV_{b})^{\dagger} \bU_{b}^{\top} \right)\right)\right] \\
 = &
 \mathbb{E}_{\bX, \bS, \bT} \left[\tr\left(\left(
 \bV_{b}^\prime(\bU_{b}^{\top} \bX \bV_{b}^\prime)^{\dagger} \bU_{b}^{\top}\right)^{\top}
 \left(
 \bV_{b}^\prime(\bU_{b}^{\top} \bX \bV_{b}^\prime)^{\dagger} \bU_{b}^{\top}\right)\right)\right]
 = \frac{|\bnu_{b}|}{ |\bmu_{b}|-  |\bnu_{b} | - 1}.
\end{align*}
For the other two cross terms, we treat them as only have $\bmu_b$ rows, and apply \eqref{equ:lem-variance-diff-b} in Lemma \ref{lem:variance-resorted-lemma} to have
\begin{align*}
    \mathbb{E}_{\bX, \bS, \bT} \left[\tr\left(\left( \bV_{b}(\bU_{b}^{\top} \bX \bV_{b})^{\dagger} \bU_{b}^{\top} \right)^{\top}
 \left( \bV_{b}^\prime(\bU_{b}^{\top} \bX \bV_{b}^\prime)^{\dagger} \bU_{b}^{\top} \right)\right)\right]
 =  \frac{|\bnu_b\cap \bnu_b'|}{ n p -|\bnu_b \cap \bnu_b'| - 1}.
\end{align*}
Combined together, this is
\begin{align*}
     \mathbb{E}_{\bX, \bS, \bT} \left[V_{b,b}\right]  = 2 \frac{|\bnu_b|}{ |\bmu_b|-  |\bnu_b | - 1} - 2 \frac{|\bnu_b\cap \bnu_b'|}{ n p -|\bnu_b \cap \bnu_b'| - 1}.
\end{align*}
Taking the limit, this is
\begin{align*}
 \lim_{n,d\to \infty} \mathbb{E}_{\bX, \bS, \bT} \left[V_{b,b}\right] = 2 \frac{ \gamma q}{p - \gamma q} - 2 \frac{ \gamma q^2 }{p - \gamma q^2} = \frac{2 \gamma pq (1 - q)}{(p - \gamma q)(p - \gamma q^2)}.
\end{align*}

For the cross terms ($b_1\neq b_2$) for instance removal, we have
\begin{align*}
     \lim_{n,d\to \infty}n \cdot  \mathbb{E}_{\bX, \bS, \bT} \left[V_{b_1,b_2}\right] =& \lim_{n,d\to \infty}n \cdot  \mathbb{E}_{\bX, \bS, \bT} \left[\frac{|\bnu_1 \cap \bnu_2|}{n - 1 -  |\bnu_1 \cap \bnu_2| - 1}- \frac{|\bnu_1 \cap \bnu_2|}{n -  |\bnu_1 \cap \bnu_2| - 1}\right] \\
     = & \lim_{n,d\to \infty}n \cdot   \left[\frac{\frac{s^2}{d}}{n - \frac{s^2}{d} - 2} - \frac{\frac{s^2}{d}}{n - \frac{s^2}{d} - 1}\right] = \frac{\gamma q^2}{(1 - \gamma q^2)^2}.
\end{align*}

As for the cross terms ($b_1\neq b_2$) for feature removal, we have
\begin{align*}
     \lim_{n,d\to \infty}n \cdot  \mathbb{E}_{\bX, \bS, \bT} \left[V_{b_1,b_2}\right] =& \lim_{n,d\to \infty}n \cdot  \mathbb{E}_{\bX, \bS, \bT} \left[\frac{|\bnu_1' \cap \bnu_2'|}{n -  |\bnu_1' \cap \bnu_2'| - 1} - \frac{|\bnu_1 \cap \bnu_2|}{n -  |\bnu_1 \cap \bnu_2| - 1}\right] \\
     = & \lim_{n,d\to \infty}n \cdot   \left[\frac{\frac{s^2}{d-1}}{n - \frac{s^2}{d-1} - 1} - \frac{\frac{s^2}{d}}{n - \frac{s^2}{d} - 1}\right] = \frac{q^2}{(1 - \gamma q^2)^2}.
\end{align*}

\textbf{Overparameterized case} \;\;
For instance removal, the squared term becomes
\begin{align*}
V_{b,b} =     \tr\left(\left( \bV_{b}(\bU_{b}^{\top} \bX \bV_{b})^{\dagger} \bU_{b}^{\top} -
 \bV_{b}(\bU_{b}^{\prime\top} \bX \bV_{b})^{\dagger} \bU_{b}^{\prime\top}\right)^{\top}
 \left( \bV_{b}(\bU_{b}^{\top} \bX \bV_{b})^{\dagger} \bU_{b}^{\top} -
 \bV_{b}(\bU_{b}^{\prime\top} \bX \bV_{b})^{\dagger} \bU_{b}^{\prime\top}\right)\right),
\end{align*}
where $\bU_b$ and $\bU_b'$ are independent.
We apply \eqref{equ:lem-variance-same-b} in Lemma \ref{lem:variance-resorted-lemma} to have
\begin{align*}
    &\mathbb{E}_{\bX, \bS, \bT} \left[\tr\left(\left( \bV_{b}(\bU_{b}^{\top} \bX \bV_{b})^{\dagger} \bU_{b}^{\top} \right)^{\top}
 \left( \bV_{b}(\bU_{b}^{\top} \bX \bV_{b})^{\dagger} \bU_{b}^{\top} \right)\right)\right] \\
 = &
 \mathbb{E}_{\bX, \bS, \bT} \left[\tr\left(\left(
 \bV_{b}(\bU_{b}^{\prime\top} \bX \bV_{b})^{\dagger} \bU_{b}^{\prime\top}\right)^{\top}
 \left(
 \bV_{b}(\bU_{b}^{\prime\top} \bX \bV_{b})^{\dagger} \bU_{b}^{\prime\top}\right)\right)\right]
 = \frac{|\bmu_{b}|}{ |\bnu_{b}|-  |\bmu_{b} | - 1}.
\end{align*}
For the other two cross terms, we treat them as only have $\bnu_b$ columns, and apply \eqref{equ:lem-variance-diff-b} in Lemma \ref{lem:variance-resorted-lemma} to have
\begin{align*}
    \mathbb{E}_{\bX, \bS, \bT} \left[\tr\left(\left( \bV_{b}(\bU_{b}^{\top} \bX \bV_{b})^{\dagger} \bU_{b}^{\top} \right)^{\top}
 \left( \bV_{b}^\prime(\bU_{b}^{\top} \bX \bV_{b}^\prime)^{\dagger} \bU_{b}^{\top} \right)\right)\right]
 =  \frac{|\bmu_b \cap \bmu_b'|}{qd - |\bmu_b \cap \bmu_b'| - 1}.
\end{align*}
Combined together, this is
\begin{align*}
     \mathbb{E}_{\bX, \bS, \bT} \left[V_{b,b}\right]  = 2 \frac{|\bmu_{b}|}{ |\bnu_{b}|-  |\bmu_{b} | - 1} - 2  \frac{|\bmu_b \cap \bmu_b'|}{qd - |\bmu_b \cap \bmu_b'| - 1}.
\end{align*}
Taking the limit, this is
\begin{align*}
 \lim_{n,d\to \infty} \mathbb{E}_{\bX, \bS, \bT} \left[V_{b,b}\right] = 2 \frac{ p}{ \gamma q - p} - 2 \frac{  p^2 }{\gamma q - p^2} = \frac{2 \gamma q (p - p^2)}{( \gamma q - p)( \gamma q - p^2)}.
\end{align*}

For feature removal, the squared terms become
\begin{align*}
V_{b,b} =     \tr\left(\left( \bV_{b}(\bU_{b}^{\top} \bX \bV_{b})^{\dagger} \bU_{b}^{\top} -
 \bV_{b}^\prime(\bU_{b}^{\top} \bX \bV_{b}^\prime)^{\dagger} \bU_{b}^{\top}\right)^{\top}
 \left( \bV_{b}(\bU_{b}^{\top} \bX \bV_{b})^{\dagger} \bU_{b}^{\top} -
 \bV_{b}^\prime(\bU_{b}^{\top} \bX \bV_{b}^{\prime})^{\dagger} \bU_{b}^{\top}\right)\right),
\end{align*}
where $\bV_b$ and $\bV_b'$ are independent.
We apply \eqref{equ:lem-variance-same-b} in Lemma \ref{lem:variance-resorted-lemma} to have
\begin{align*}
    &\mathbb{E}_{\bX, \bS, \bT} \left[\tr\left(\left( \bV_{b}(\bU_{b}^{\top} \bX \bV_{b})^{\dagger} \bU_{b}^{\top} \right)^{\top}
 \left( \bV_{b}(\bU_{b}^{\top} \bX \bV_{b})^{\dagger} \bU_{b}^{\top} \right)\right)\right] \\
 = &
 \mathbb{E}_{\bX, \bS, \bT} \left[\tr\left(\left(
 \bV_{b}^\prime(\bU_{b}^{\top} \bX \bV_{b}^\prime)^{\dagger} \bU_{b}^{\top}\right)^{\top}
 \left(
 \bV_{b}^\prime(\bU_{b}^{\top} \bX \bV_{b}^\prime)^{\dagger} \bU_{b}^{\top}\right)\right)\right]
 = \frac{|\bmu_{b}|}{ |\bnu_{b}|-  |\bmu_{b} | - 1}.
\end{align*}
For the other two cross terms, we treat them as only have $\bmu_b$ rows, and apply \eqref{equ:lem-variance-diff-b} in Lemma \ref{lem:variance-resorted-lemma} to have
\begin{align*}
    \mathbb{E}_{\bX, \bS, \bT} \left[\tr\left(\left( \bV_{b}(\bU_{b}^{\top} \bX \bV_{b})^{\dagger} \bU_{b}^{\top} \right)^{\top}
 \left( \bV_{b}^\prime(\bU_{b}^{\top} \bX \bV_{b}^\prime)^{\dagger} \bU_{b}^{\top} \right)\right)\right]
 =  \frac{|\bmu_b |}{d - |\bmu_b | - 1}.
\end{align*}
Combined together, this is
\begin{align*}
     \mathbb{E}_{\bX, \bS, \bT} \left[V_{b,b}\right]  = 2 \frac{|\bmu_{b}|}{ |\bnu_{b}|-  |\bmu_{b} | - 1} - 2   \frac{|\bmu_b |}{d - |\bmu_b | - 1}.
\end{align*}
Taking the limit, this is
\begin{align*}
 \lim_{n,d\to \infty} \mathbb{E}_{\bX, \bS, \bT} \left[V_{b,b}\right] = 2 \frac{ p}{ \gamma q - p} - 2 \frac{  p }{\gamma  - p} = \frac{2 \gamma p (1 - q)}{( \gamma q - p)( \gamma - p)}.
\end{align*}

For the cross terms ($b_1 \neq b_2$) with instance removal, we have
\begin{align*}
     \lim_{n,d\to \infty}n \cdot  \mathbb{E}_{\bX, \bS, \bT} \left[V_{b_1,b_2}\right] =& \lim_{n,d\to \infty}n \cdot  \mathbb{E}_{\bX, \bS, \bT} \left[\frac{|\bmu_1' \cap \bmu_2'|}{d -  |\bmu_1' \cap \bmu_2'| - 1} - \frac{|\bmu_1 \cap \bmu_2|}{d -  |\bmu_1 \cap \bmu_2| - 1}\right] \\
     = & \lim_{n,d\to \infty}n \cdot   \left[\frac{\frac{m^2}{n-1}}{d - \frac{m^2}{n-1} - 1} - \frac{\frac{m^2}{n}}{d - \frac{m^2}{n} - 1}\right] = \frac{\gamma p^2}{(\gamma - p^2)^2}.
\end{align*}
For the cross terms ($b_1 \neq b_2$) with feature removal, we have
\begin{align*}
     \lim_{n,d\to \infty}n \cdot  \mathbb{E}_{\bX, \bS, \bT} \left[V_{b_1,b_2}\right] =& \lim_{n,d\to \infty}n \cdot  \mathbb{E}_{\bX, \bS, \bT} \left[\frac{|\bmu_1 \cap \bmu_2|}{d - 1 -  |\bmu_1 \cap \bmu_2| - 1} - \frac{|\bmu_1 \cap \bmu_2|}{d -  |\bmu_1 \cap \bmu_2| - 1}\right] \\
     = & \lim_{n,d\to \infty}n \cdot   \left[\frac{\frac{m^2}{n}}{d -1 - \frac{m^2}{n} - 1} - \frac{\frac{m^2}{n}}{d - \frac{m^2}{n} - 1}\right] = \frac{ p^2}{(\gamma - p^2)^2}.
\end{align*}

\end{proof}

\begin{proof}[Proof of Proposition \ref{prop:biasofstability}]
    Recall that
\begin{align*}
& \mathbb{E}_{\bbeta^*}\left[\bbeta^{*\top}\bX^{\top} \left(g(\bx) - g'(\bx)\right)^{\top} \left(g(\bx) - g'(\bx)\right) \bX \bbeta^*  \right]\\
= & \frac{1}{d B^2}  \sum_{b_1, b_2=1}^B \underbrace{ \tr\left(\left( \bV_{b_1}(\bU_{b_1}^{\top} \bX \bV_{b_1})^{\dagger} \bU_{b_1}^{\top} \bX -
 \bV_{b_1}^{\prime}(\bU_{b_1}^{\prime\top} \bX \bV_{b_1}^{\prime})^{\dagger} \bU_{b_1}^{\prime\top} \bX\right)^{\top}
 \left( \bV_{b_2}(\bU_{b_2}^{\top} \bX \bV_{b_2})^{\dagger} \bU_{b_2}^{\top} \bX -
 \bV_{b_2}^{\prime}(\bU_{b_2}^{\prime\top} \bX \bV_{b_2}^{\prime})^{\dagger} \bU_{b_2}^{\prime\top} \bX\right)\right) }_{B_{b_1,b_2}}.
\end{align*}
There are $B(B-1)$ terms which have $b_1\neq b_2$, and $B$ terms with $b_1 = b_2$.

\textbf{Underparameterized case} \;\;
For instance removal, the squared terms become
\begin{align*}
B_{b,b} =   \tr\left(\left( \bV_{b}(\bU_{b}^{\top} \bX \bV_{b})^{\dagger} \bU_{b}^{\top} \bX -
 \bV_{b}(\bU_{b}^{\prime\top} \bX \bV_{b})^{\dagger} \bU_{b}^{\prime\top} \bX\right)^{\top}
 \left( \bV_{b}(\bU_{b}^{\top} \bX \bV_{b})^{\dagger} \bU_{b}^{\top} \bX -
 \bV_{b}(\bU_{b}^{\prime\top} \bX \bV_{b})^{\dagger} \bU_{b}^{\prime\top} \bX\right)\right),
\end{align*}
where $\bU_b$ and $\bU_b'$ are independent.
We apply \eqref{equ:lem-bias-same-b} in Lemma \ref{lem:bias-resorted-lemma} to have
\begin{align*}
    &\mathbb{E}_{\bX, \bS, \bT} \left[\tr\left(\left( \bV_{b}(\bU_{b}^{\top} \bX \bV_{b})^{\dagger} \bU_{b}^{\top} \bX\right)^{\top}
 \left( \bV_{b}(\bU_{b}^{\top} \bX \bV_{b})^{\dagger} \bU_{b}^{\top} \bX\right)\right)\right] \\
 = &
 \mathbb{E}_{\bX, \bS, \bT} \left[\tr\left(\left(
 \bV_{b}(\bU_{b}^{\prime\top} \bX \bV_{b})^{\dagger} \bU_{b}^{\prime\top}\bX\right)^{\top}
 \left(
 \bV_{b}(\bU_{b}^{\prime\top} \bX \bV_{b})^{\dagger} \bU_{b}^{\prime\top}\bX\right)\right)\right]
 = \frac{|\bnu_{b}| (d - |\bnu_b|)}{ |\bmu_{b}|-  |\bnu_{b} | - 1} + |\bnu_{b} |.
\end{align*}
For the other two cross terms, we treat them as only have $\bnu_b$ columns, and apply \eqref{equ:lem-bias-diff-b} in Lemma \ref{lem:bias-resorted-lemma} to have
\begin{align*}
   \mathbb{E}_{\bX, \bS, \bT} \left[\tr\left(\left(
 \bV_{b}(\bU_{b}^{\top} \bX \bV_{b})^{\dagger} \bU_{b}^{\top}\bX\right)^{\top}
 \left(
 \bV_{b}(\bU_{b}^{\prime\top} \bX \bV_{b})^{\dagger} \bU_{b}^{\prime\top}\bX\right)\right)\right]
 =  \frac{|\bnu_b||\bnu_b^c|}{ n-|\bnu_b| - 1} + |\bnu_b|.
\end{align*}
Combined together, this is
\begin{align*}
     \mathbb{E}_{\bX, \bS, \bT} \left[B_{b,b}\right]  = 2 ( \frac{|\bnu_{b}| (d - |\bnu_b|)}{ |\bmu_{b}|-  |\bnu_{b} | - 1} + |\bnu_{b} |) - 2 ( \frac{|\bnu_b||\bnu_b^c|}{ n-|\bnu_b| - 1} + |\bnu_b|).
\end{align*}
Taking the limit, this is
\begin{align*}
 \lim_{n,d\to \infty} \frac{1}{d}\mathbb{E}_{\bX, \bS, \bT} \left[B_{b,b}\right] = 2 \frac{ \gamma q (1- q)}{p - \gamma q} - 2 \frac{ \gamma q (1- q)}{1 - \gamma q} = \frac{2 \gamma q (1- q)(1 - p)}{(p - \gamma q)(1 - \gamma q)}.
\end{align*}

For feature removal, the squared term becomes
\begin{align*}
B_{b,b} =   \tr\left(\left( \bV_{b}(\bU_{b}^{\top} \bX \bV_{b})^{\dagger} \bU_{b}^{\top} \bX -
 \bV_{b}^{\prime}(\bU_{b}^{\top} \bX \bV_{b}^{\prime})^{\dagger} \bU_{b}^{\top} \bX\right)^{\top}
 \left( \bV_{b}(\bU_{b}^{\top} \bX \bV_{b})^{\dagger} \bU_{b}^{\top} \bX -
 \bV_{b}^{\prime}(\bU_{b}^{\top} \bX \bV_{b}^{\prime})^{\dagger} \bU_{b}^{\top} \bX\right)\right),
\end{align*}
where $\bV_b$ and $\bV_b'$ are independent.
We apply \eqref{equ:lem-bias-same-b} in Lemma \ref{lem:bias-resorted-lemma} to have
\begin{align*}
    &\mathbb{E}_{\bX, \bS, \bT} \left[\tr\left(\left( \bV_{b}(\bU_{b}^{\top} \bX \bV_{b})^{\dagger} \bU_{b}^{\top} \bX\right)^{\top}
 \left( \bV_{b}(\bU_{b}^{\top} \bX \bV_{b})^{\dagger} \bU_{b}^{\top} \bX\right)\right)\right] \\
 = &
 \mathbb{E}_{\bX, \bS, \bT} \left[\tr\left(\left(
 \bV_{b}^\prime(\bU_{b}^{\top} \bX \bV_{b}^\prime)^{\dagger} \bU_{b}^{\top}\bX\right)^{\top}
 \left(
 \bV_{b}^\prime(\bU_{b}^{\top} \bX \bV_{b}^\prime)^{\dagger} \bU_{b}^{\top}\bX\right)\right)\right]
 = \frac{|\bnu_{b}| (d - |\bnu_b|)}{ |\bmu_{b}|-  |\bnu_{b} | - 1} + |\bnu_{b} |.
\end{align*}
For the other two cross terms, we treat them as only have $\bnu_b$ rows, and apply \eqref{equ:lem-bias-diff-b} in Lemma \ref{lem:bias-resorted-lemma} to have
\begin{align*}
   \mathbb{E}_{\bX, \bS, \bT} \left[\tr\left(\left(
 \bV_{b}(\bU_{b}^{\top} \bX \bV_{b})^{\dagger} \bU_{b}^{\top}\bX\right)^{\top}
 \left(
 \bV_{b}^\prime(\bU_{b}^{\top} \bX \bV_{b}^\prime)^{\dagger} \bU_{b}^{\top}\bX\right)\right)\right]
 =  \frac{|\bnu_b\cap \bnu_b'||\bnu_b^c\cap \bnu_b^{\prime c}|}{ np-|\bnu_b\cap \bnu_b'| - 1} + |\bnu_b\cap \bnu_b'|.
\end{align*}
Combined together, this is
\begin{align*}
     \mathbb{E}_{\bX, \bS, \bT} \left[B_{b,b}\right]  = 2 \left(\frac{|\bnu_{b}| (d - |\bnu_b|)}{ |\bmu_{b}|-  |\bnu_{b} | - 1} + |\bnu_{b} |\right) - 2 \left(\frac{|\bnu_b\cap \bnu_b'||\bnu_b^c\cap \bnu_b^{\prime c}|}{ np-|\bnu_b\cap \bnu_b'| - 1} + |\bnu_b\cap \bnu_b'|\right).
\end{align*}
Taking the limit, this is
\begin{align*}
 \lim_{n,d\to \infty} \frac{1}{d}\mathbb{E}_{\bX, \bS, \bT} \left[B_{b,b}\right] = 2 \frac{ \gamma q (1 - q) }{p - \gamma q} - 2 \frac{ \gamma q^2 (1-q)^2 }{p - \gamma q^2}  + 2 q (1- q) = \frac{2 \gamma pq (1 - q)^2 }{(p - \gamma q)(p - \gamma q^2)} + \frac{2  pq (1 - q)}{p - \gamma q^2}.
\end{align*}

For the cross terms ($b_1\neq b_2$) for instance removal, we have
\begin{align*}
     & \lim_{n,d\to \infty}\frac{n}{d} \cdot  \mathbb{E}_{\bX, \bS, \bT} \left[B_{b_1,b_2}\right] \\
     =& \lim_{n,d\to \infty} \frac{n}{d} \cdot  \mathbb{E}_{\bX, \bS, \bT} \left[ \frac{ |\bnu_{b_1}\cap \bnu_{b_2}| |\bnu^c_{b_1} \cap \bnu^c_{b_2}| }{\tr(\Lambda^{\vee}) - |\bnu_{b_1}\cap \bnu_{b_2}| - 1} + |\bnu_{b_1}\cap \bnu_{b_2}| -  \frac{ |\bnu_{b_1}\cap \bnu_{b_2}| |\bnu^c_{b_1} \cap \bnu^c_{b_2}| }{\tr(\Lambda^{\vee}) - |\bnu_{b_1}\cap \bnu_{b_2}| - 1} - |\bnu_{b_1}\cap \bnu_{b_2}| \right] \\
     = & \lim_{n,d\to \infty}\frac{n}{d} \cdot   \left[ \frac{ \frac{s^2}{d} \frac{(d-s)^2}{d} }{n - 1 - \frac{s^2}{d} - 1} + \frac{s^2}{d} - \frac{ \frac{s^2}{d} \frac{(d-s)^2}{d} }{n - \frac{s^2}{d} - 1} - \frac{s^2}{d} \right] = \frac{  \gamma q^2 (1 - q)^2 }{(1 - \gamma q^2)^2} .
\end{align*}
As for the cross terms ($b_1\neq b_2$) for feature removal, we have
\begin{align*}
     & \lim_{n,d\to \infty}\frac{n}{d} \cdot  \mathbb{E}_{\bX, \bS, \bT} \left[B_{b_1,b_2}\right] \\
     =& \lim_{n,d\to \infty} \frac{n}{d} \cdot  \mathbb{E}_{\bX, \bS, \bT} \left[ \frac{ |\bnu_{b_1}^{\prime}\cap \bnu_{b_2}^{\prime}| |\bnu^{\prime c}_{b_1} \cap \bnu^{\prime c}_{b_2}| }{\tr(\Lambda^{\vee}) - |\bnu_{b_1}^{\prime}\cap \bnu_{b_2}^{\prime}| - 1} + |\bnu_{b_1}^{\prime}\cap \bnu_{b_2}^{\prime}| -  \frac{ |\bnu_{b_1}\cap \bnu_{b_2}| |\bnu^c_{b_1} \cap \bnu^c_{b_2}| }{\tr(\Lambda^{\vee}) - |\bnu_{b_1}\cap \bnu_{b_2}| - 1} - |\bnu_{b_1}\cap \bnu_{b_2}| \right] \\
     = & \lim_{n,d\to \infty}\frac{n}{d} \cdot   \left[ \frac{ \frac{s^2}{d-1} (\frac{(d-1-s)^2}{d-2} + 1) }{n - \frac{s^2}{d-1} - 1} + \frac{s^2}{d-1} - \frac{ \frac{s^2}{d} \frac{(d-s)^2}{d-1} }{n - \frac{s^2}{d} - 1} - \frac{s^2}{d} \right] \\
     = & \frac{  \gamma q^4 (1 - q^2) - 2  q^3(1 - q) }{ (1 - \gamma q^2)^2} + \frac{q^2}{1 - \gamma q^2} + \gamma^{-1} q^2 .
\end{align*}

\textbf{Overparameterized case}\;\;
For instance removal, the squared terms become
\begin{align*}
B_{b,b} =   \tr\left(\left( \bV_{b}(\bU_{b}^{\top} \bX \bV_{b})^{\dagger} \bU_{b}^{\top} \bX -
 \bV_{b}(\bU_{b}^{\prime\top} \bX \bV_{b})^{\dagger} \bU_{b}^{\prime\top} \bX\right)^{\top}
 \left( \bV_{b}(\bU_{b}^{\top} \bX \bV_{b})^{\dagger} \bU_{b}^{\top} \bX -
 \bV_{b}(\bU_{b}^{\prime\top} \bX \bV_{b})^{\dagger} \bU_{b}^{\prime\top} \bX\right)\right),
\end{align*}
where $\bU_b$ and $\bU_b'$ are independent.
We apply \eqref{equ:lem-bias-same-b} in Lemma \ref{lem:bias-resorted-lemma} to have
\begin{align*}
    &\mathbb{E}_{\bX, \bS, \bT} \left[\tr\left(\left( \bV_{b}(\bU_{b}^{\top} \bX \bV_{b})^{\dagger} \bU_{b}^{\top} \bX\right)^{\top}
 \left( \bV_{b}(\bU_{b}^{\top} \bX \bV_{b})^{\dagger} \bU_{b}^{\top} \bX\right)\right)\right] \\
 = &
 \mathbb{E}_{\bX, \bS, \bT} \left[\tr\left(\left(
 \bV_{b}(\bU_{b}^{\prime\top} \bX \bV_{b})^{\dagger} \bU_{b}^{\prime\top}\bX\right)^{\top}
 \left(
 \bV_{b}(\bU_{b}^{\prime\top} \bX \bV_{b})^{\dagger} \bU_{b}^{\prime\top}\bX\right)\right)\right]
 = \frac{|\bmu_{b}| (d - |\bnu_b|)}{ |\bnu_{b}|-  |\bmu_{b} | - 1} + |\bmu_{b} |.
\end{align*}
For the other two cross terms, we treat them as only have $\bnu_b$ columns, and apply \eqref{equ:lem-bias-diff-b} in Lemma \ref{lem:bias-resorted-lemma} to have
\begin{align*}
  & \mathbb{E}_{\bX, \bS, \bT} \left[\tr\left(\left(
 \bV_{b}(\bU_{b}^{\top} \bX \bV_{b})^{\dagger} \bU_{b}^{\top}\bX\right)^{\top}
 \left(
 \bV_{b}(\bU_{b}^{\prime\top} \bX \bV_{b})^{\dagger} \bU_{b}^{\prime\top}\bX\right)\right)\right]
 \\
 = & \frac{|\bmu_b / \bmu_b'||\bmu_b' / \bmu_b|}{ qd -|\bmu_b \cap \bmu_b'| - 1} + |\bmu_b \cap \bmu_b'| \frac{d - |\bmu_b \cap \bmu_b'| - 1}{qd -   |\bmu_b \cap \bmu_b'|-1}.
\end{align*}
Combined together, this is
\begin{align*}
     \mathbb{E}_{\bX, \bS, \bT} \left[B_{b,b}\right]  = 2 (\frac{|\bmu_{b}| (d - |\bnu_b|)}{ |\bnu_{b}|-  |\bmu_{b} | - 1} + |\bmu_{b} |) - 2 ( \frac{|\bmu_b / \bmu_b'||\bmu_b' / \bmu_b|}{ qd -|\bmu_b \cap \bmu_b'| - 1} + |\bmu_b \cap \bmu_b'| \frac{d - |\bmu_b \cap \bmu_b'| - 1}{qd -   |\bmu_b \cap \bmu_b'|-1}).
\end{align*}
Taking the limit, this is
\begin{align*}
 \lim_{n,d\to \infty} \frac{1}{d}\mathbb{E}_{\bX, \bS, \bT} \left[B_{b,b}\right] = & 2 \frac{  p (1- q)} { \gamma q -p} + 2 \gamma^{-1}p - 2 \frac{ \gamma^{-1} p^2 (1- p)^2 }{\gamma q - p^2} - 2 \gamma^{-1} p^2 \frac{\gamma - p^2 }{\gamma q - p^2}.
\end{align*}

For feature removal, the squared term becomes
\begin{align*}
B_{b,b} =   \tr\left(\left( \bV_{b}(\bU_{b}^{\top} \bX \bV_{b})^{\dagger} \bU_{b}^{\top} \bX -
 \bV_{b}^{\prime}(\bU_{b}^{\top} \bX \bV_{b}^{\prime})^{\dagger} \bU_{b}^{\top} \bX\right)^{\top}
 \left( \bV_{b}(\bU_{b}^{\top} \bX \bV_{b})^{\dagger} \bU_{b}^{\top} \bX -
 \bV_{b}^{\prime}(\bU_{b}^{\top} \bX \bV_{b}^{\prime})^{\dagger} \bU_{b}^{\top} \bX\right)\right),
\end{align*}
where $\bV_b$ and $\bV_b'$ are independent.
We apply \eqref{equ:lem-bias-same-b} in Lemma \ref{lem:bias-resorted-lemma} to have
\begin{align*}
    &\mathbb{E}_{\bX, \bS, \bT} \left[\tr\left(\left( \bV_{b}(\bU_{b}^{\top} \bX \bV_{b})^{\dagger} \bU_{b}^{\top} \bX\right)^{\top}
 \left( \bV_{b}(\bU_{b}^{\top} \bX \bV_{b})^{\dagger} \bU_{b}^{\top} \bX\right)\right)\right] \\
 = &
 \mathbb{E}_{\bX, \bS, \bT} \left[\tr\left(\left(
 \bV_{b}^\prime(\bU_{b}^{\top} \bX \bV_{b}^\prime)^{\dagger} \bU_{b}^{\top}\bX\right)^{\top}
 \left(
 \bV_{b}^\prime(\bU_{b}^{\top} \bX \bV_{b}^\prime)^{\dagger} \bU_{b}^{\top}\bX\right)\right)\right]
 = \frac{|\bmu_{b}| (d - |\bnu_b|)}{ |\bnu_{b}|-  |\bmu_{b} | - 1} + |\bmu_{b} |.
\end{align*}
For the other two cross terms, we treat them as only have $\bmu_b$ rows, and apply \eqref{equ:lem-bias-diff-b} in Lemma \ref{lem:bias-resorted-lemma} to have
\begin{align*}
   \mathbb{E}_{\bX, \bS, \bT} \left[\tr\left(\left(
 \bV_{b}(\bU_{b}^{\top} \bX \bV_{b})^{\dagger} \bU_{b}^{\top}\bX\right)^{\top}
 \left(
 \bV_{b}^\prime(\bU_{b}^{\top} \bX \bV_{b}^\prime)^{\dagger} \bU_{b}^{\top}\bX\right)\right)\right]
 =  |\bmu_b|.
\end{align*}
Combined together, this is
\begin{align*}
     \mathbb{E}_{\bX, \bS, \bT} \left[B_{b,b}\right]  = 2 \left(\frac{|\bmu_{b}| (d - |\bnu_b|)}{ |\bnu_{b}|-  |\bmu_{b} | - 1} + |\bmu_{b} |\right) - 2 |\bmu_b|.
\end{align*}
Taking the limit, this is
\begin{align*}
 \lim_{n,d\to \infty} \frac{1}{d}\mathbb{E}_{\bX, \bS, \bT} \left[B_{b,b}\right] = \frac{ 2 p (1 - q) }{ \gamma  q - p} .
\end{align*}

As for the cross terms ($b_1\neq b_2$) for instance removal, we have
\begin{align*}
     & \lim_{n,d\to \infty}\frac{n}{d} \cdot  \mathbb{E}_{\bX, \bS, \bT} \left[B_{b_1,b_2}\right] \\
     =& \lim_{n,d\to \infty} \frac{n}{d} \cdot  \mathbb{E}_{\bX, \bS, \bT} \left[ \frac{|\bmu_{1}^{\prime}/\bmu_2^{\prime}||\bmu_2^{\prime} / \bmu_1^{\prime}|}{\tr(\Theta^{\vee})  - |\bmu_1^{\prime} \cap \bmu_2^{\prime}|} +  |\bmu_1^{\prime} \cap \bmu_2^{\prime}| \frac{d - |\bmu_1^{\prime} \cap \bmu_2^{\prime}| - 1}{\tr(\Theta^{\vee}) - |\bmu_1^{\prime} \cap \bmu_2^{\prime}| - 1} \right] \\
     - & \lim_{n,d\to \infty} \frac{n}{d} \cdot  \mathbb{E}_{\bX, \bS, \bT} \left[ \frac{|\bmu_{1}/\bmu_2||\bmu_2 / \bmu_1|}{\tr(\Theta^{\vee})  - |\bmu_1 \cap \bmu_2|} +  |\bmu_1 \cap \bmu_2| \frac{d - |\bmu_1 \cap \bmu_2| - 1}{\tr(\Theta^{\vee}) - |\bmu_1 \cap \bmu_2| - 1} \right] \\
     = & \lim_{n,d\to \infty}\frac{n}{d} \cdot   \left[ \frac{ \frac{m^2}{n-1} \frac{(n-1-m)^2}{n-2} }{d - \frac{m^2}{n-1} } + \frac{m^2}{n-1} - \frac{ \frac{m^2}{n} \frac{(n-m)^2}{n-1} }{d - \frac{m^2}{n} } - \frac{m^2}{n} \right] \\
     = & \frac{\gamma^{-1}p^4 (1 - p^2) - 2 p^3 (1 - p)}{(\gamma - p^2)^2} + \gamma^{-1} p^2.
\end{align*}

As for the cross terms ($b_1\neq b_2$) for feature removal, we have
\begin{align*}
     & \lim_{n,d\to \infty}\frac{n}{d} \cdot  \mathbb{E}_{\bX, \bS, \bT} \left[B_{b_1,b_2}\right] \\
     =& \lim_{n,d\to \infty} \frac{n}{d} \cdot  \mathbb{E}_{\bX, \bS, \bT} \left[ \frac{|\bmu_{1}/\bmu_2||\bmu_2 / \bmu_1|}{d  - 1  - |\bmu_1 \cap \bmu_2|} +  |\bmu_1 \cap \bmu_2| \frac{d - |\bmu_1 \cap \bmu_2| - 1}{d  - 1 - |\bmu_1 \cap \bmu_2| - 1} \right] \\
     - & \lim_{n,d\to \infty} \frac{n}{d} \cdot  \mathbb{E}_{\bX, \bS, \bT} \left[ \frac{|\bmu_{1}/\bmu_2||\bmu_2 / \bmu_1|}{d  - |\bmu_1 \cap \bmu_2|} +  |\bmu_1 \cap \bmu_2| \frac{d - |\bmu_1 \cap \bmu_2| - 1}{d - |\bmu_1 \cap \bmu_2| - 1} \right] \\
     = & \frac{\gamma^{-1} p^2(1-p)^2}{(\gamma - p^2)^2} + \frac{\gamma^{-1} p^2 }{\gamma - p^2} .
\end{align*}

\end{proof}

\section{Contents Related to Model Free Stability Guarantees}

This appendix proves stronger stability guarantees in which the expectation is taken only over the algorithmic randomness.

\subsection{Additional Experiments}\label{sec:additional-experiments}

\begin{figure}[!p]
\centering
\subfigure[Weak signal]{
\begin{minipage}{0.3\linewidth}
\centering
\includegraphics[width=\textwidth]{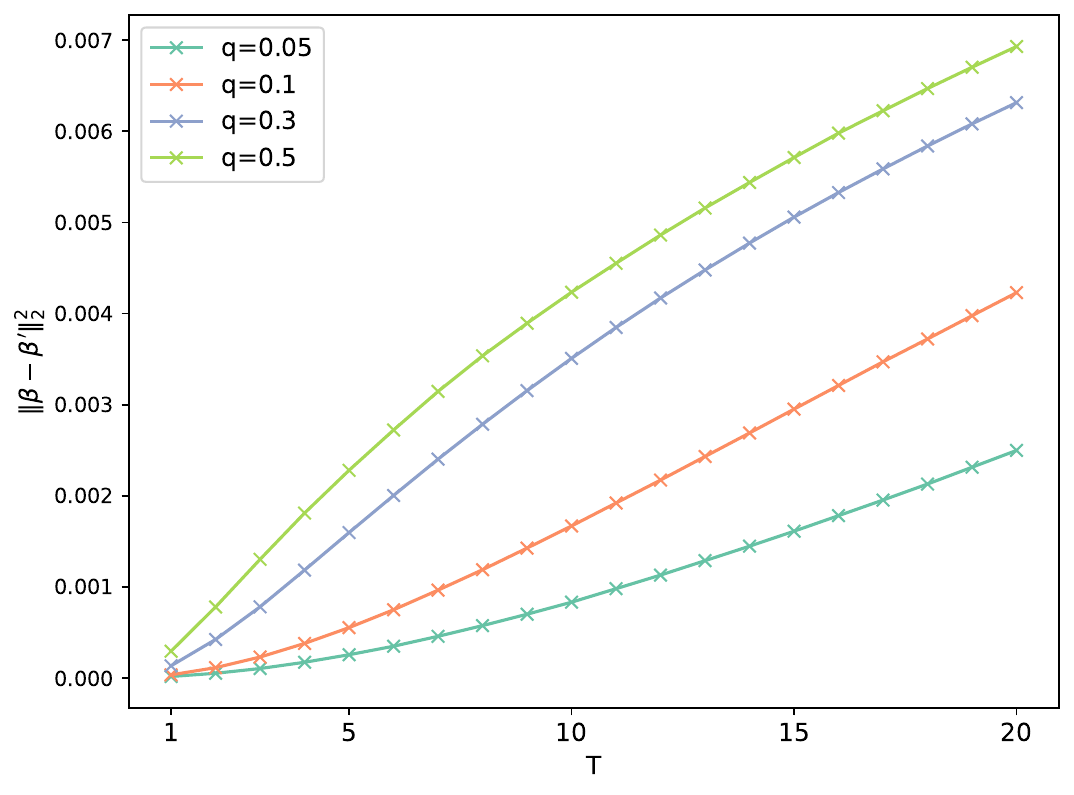}
\end{minipage}
\label{fig:forward_selection_only_stability_0.1_nu_0}
}
\subfigure[Middle signal]{
\begin{minipage}{0.3\linewidth}
\centering
\includegraphics[width=\textwidth]{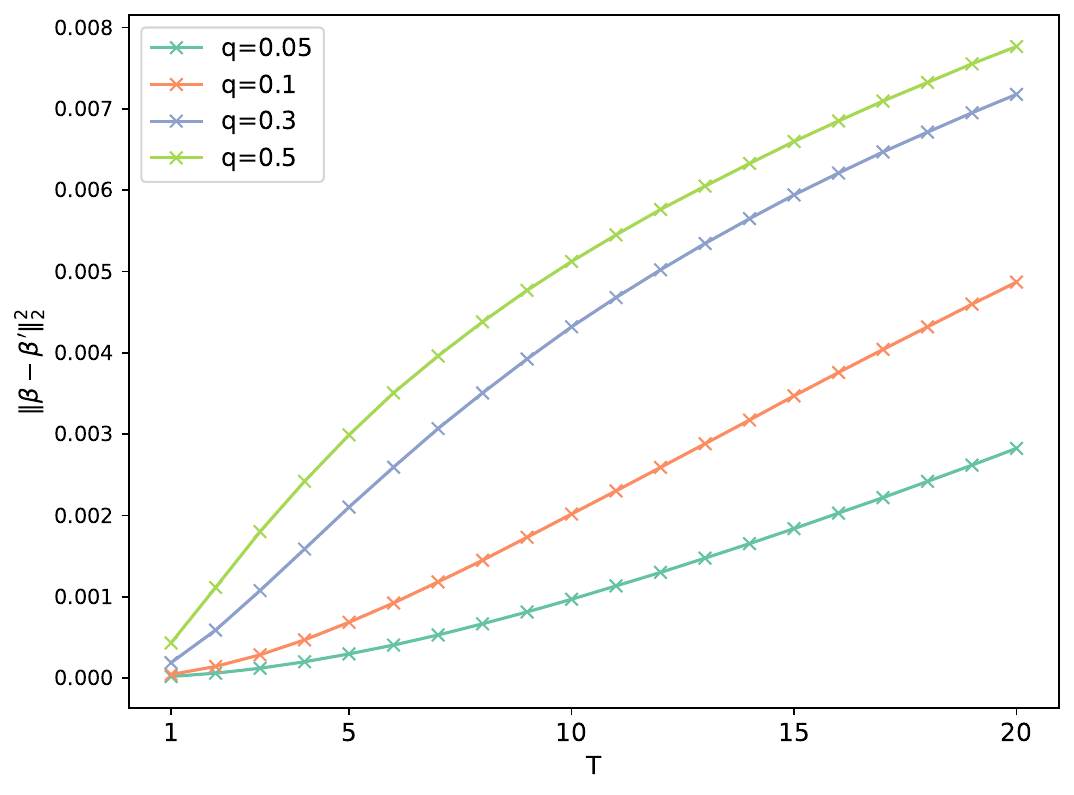}
\end{minipage}
\label{fig:forward_selection_only_stability_0.1_nu_0.5}
}
\subfigure[Strong signal]{
\begin{minipage}{0.3\linewidth}
\centering
\includegraphics[width=\textwidth]{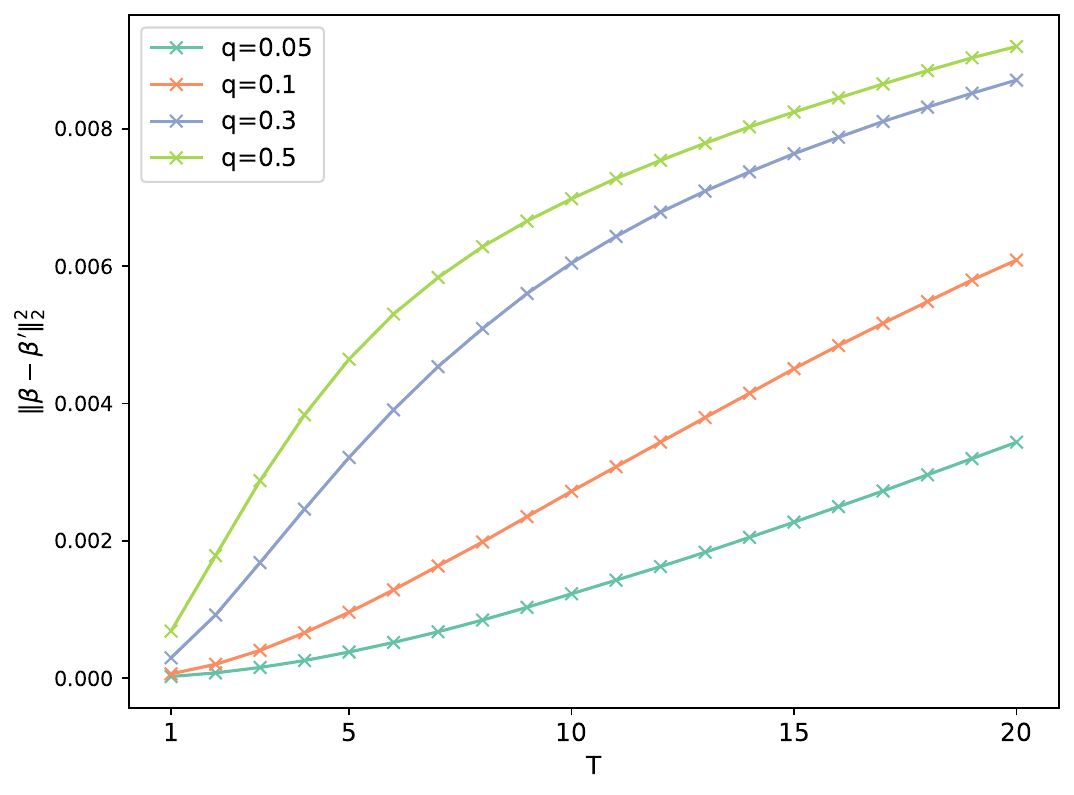}
\end{minipage}
\label{fig:forward_selection_only_stability_0.1_nu_1}
}
\subfigure[Weak signal]{
\begin{minipage}{0.3\linewidth}
\centering
\includegraphics[width=\textwidth]{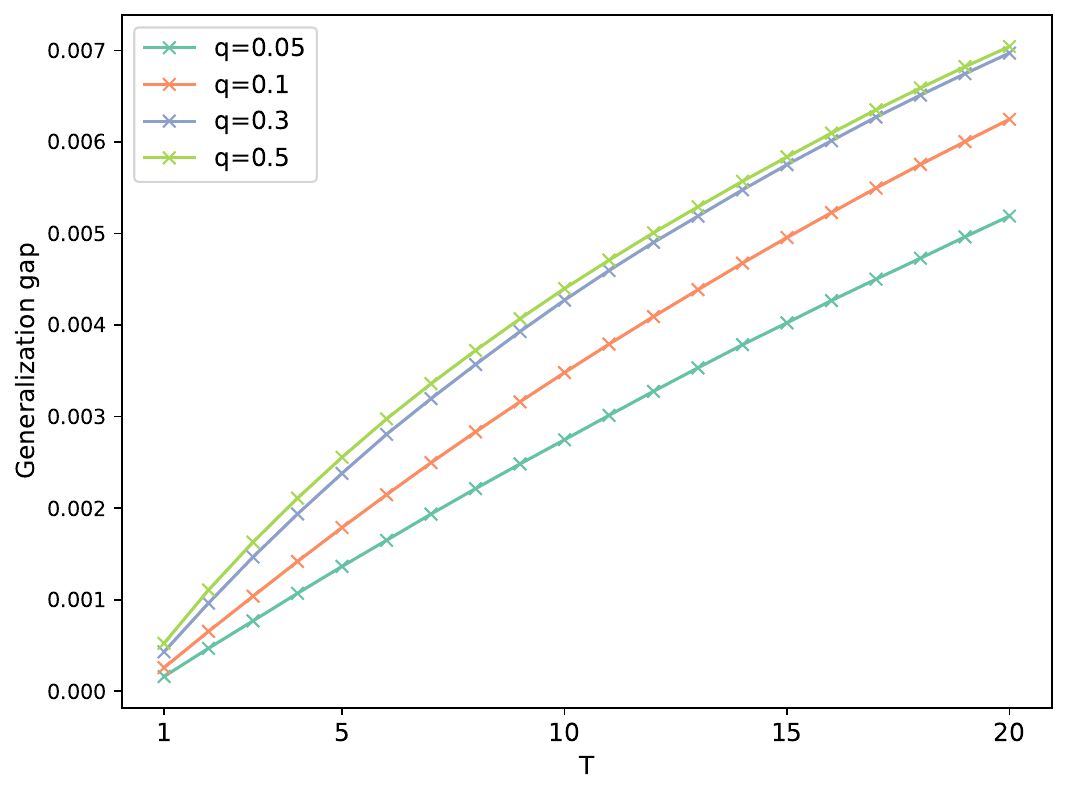}
\end{minipage}
\label{fig:forward_selection_generalization_gap_sigma_0.1_nu_0}
}
\subfigure[Mixed signal]{
\begin{minipage}{0.3\linewidth}
\centering
\includegraphics[width=\textwidth]{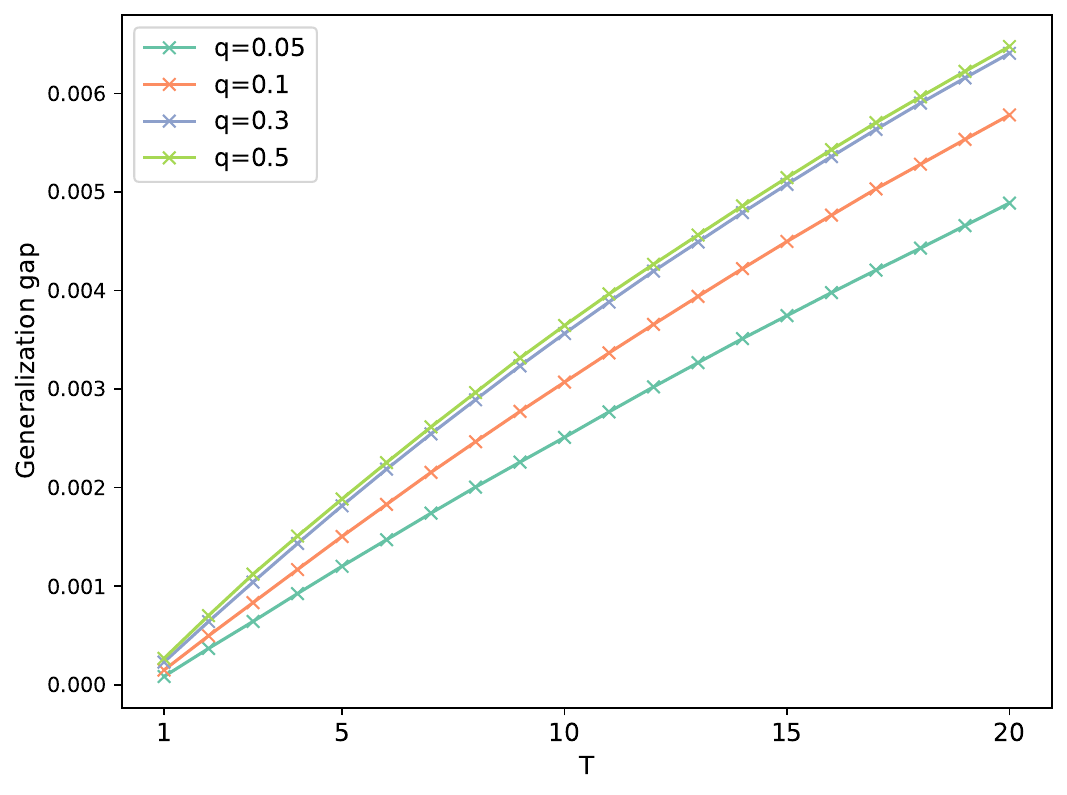}
\end{minipage}
\label{fig:forward_selection_generalization_gap_sigma_0.1_nu_0.5}
}
\subfigure[Strong signal]{
\begin{minipage}{0.3\linewidth}
\centering
\includegraphics[width=\textwidth]{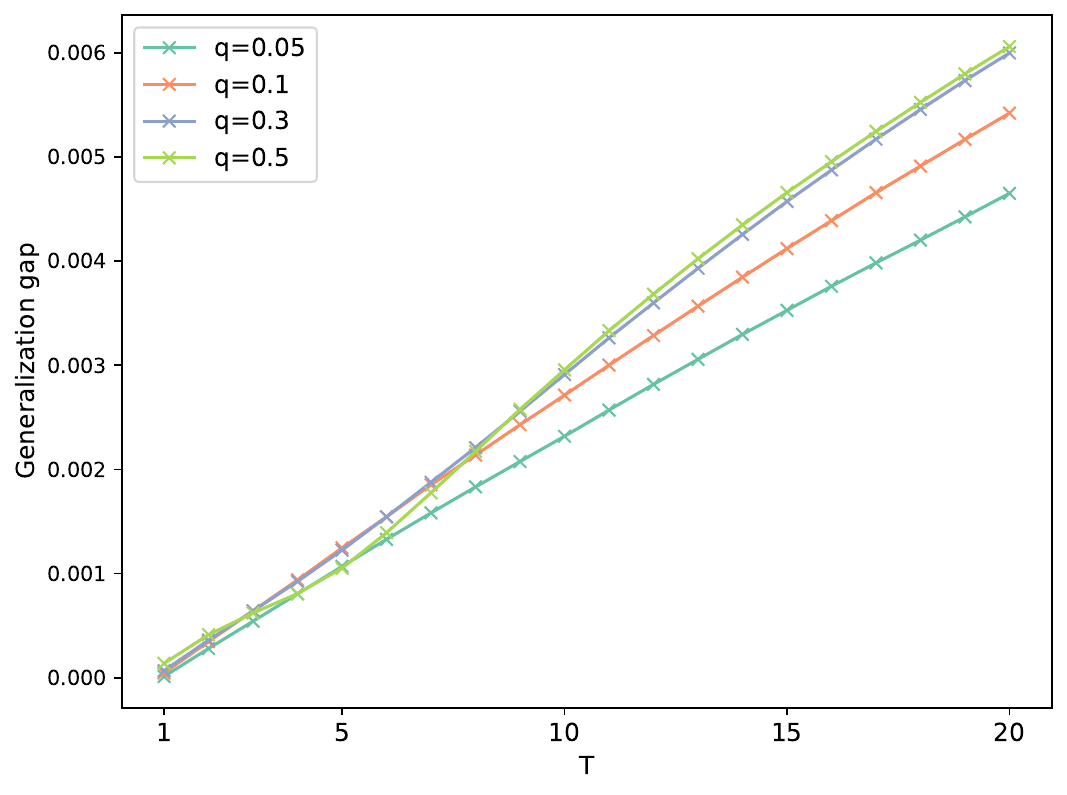}
\end{minipage}
\label{fig:forward_selection_generalization_gap_sigma_0.1_nu_1}
}
\subfigure[Weak signal]{
\begin{minipage}{0.3\linewidth}
\centering
\includegraphics[width=\textwidth]{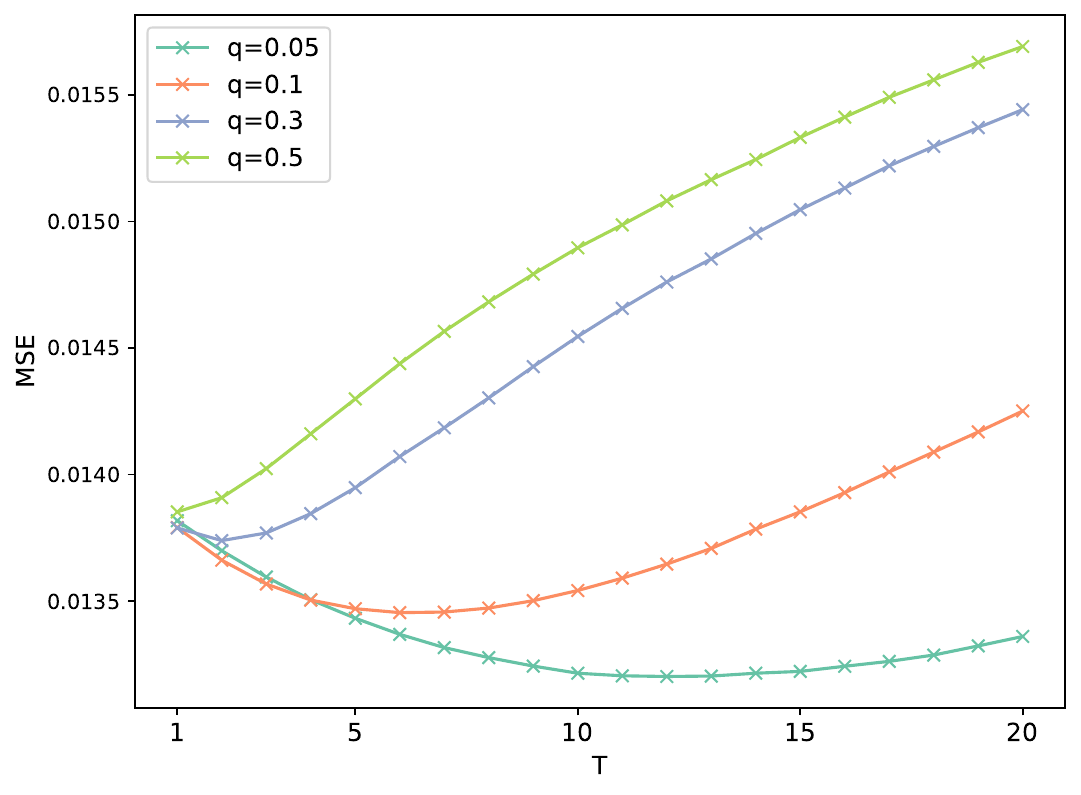}
\end{minipage}
\label{fig:forward_selection_feature_stability_sigma_nu_0}
}
\subfigure[Mixed signal]{
\begin{minipage}{0.3\linewidth}
\centering
\includegraphics[width=\textwidth]{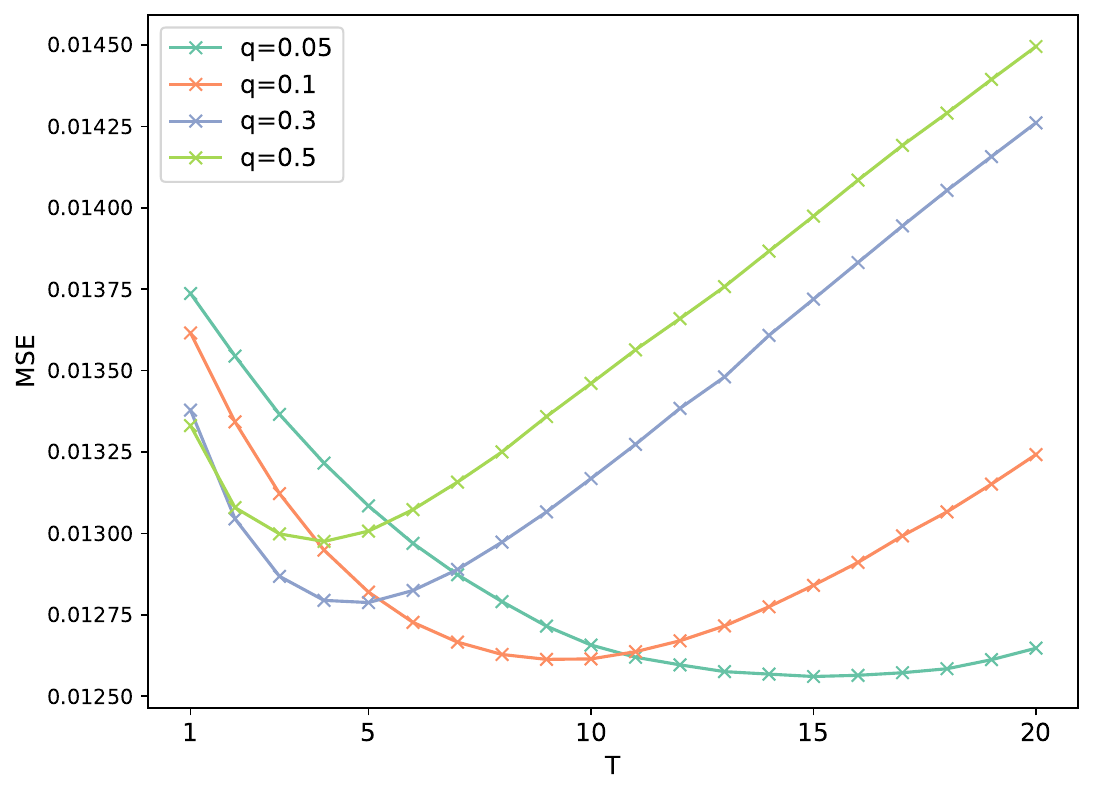}
\end{minipage}
\label{fig:forward_selection_feature_stability_sigma_nu_0.5}
}
\subfigure[Strong signal]{
\begin{minipage}{0.3\linewidth}
\centering
\includegraphics[width=\textwidth]{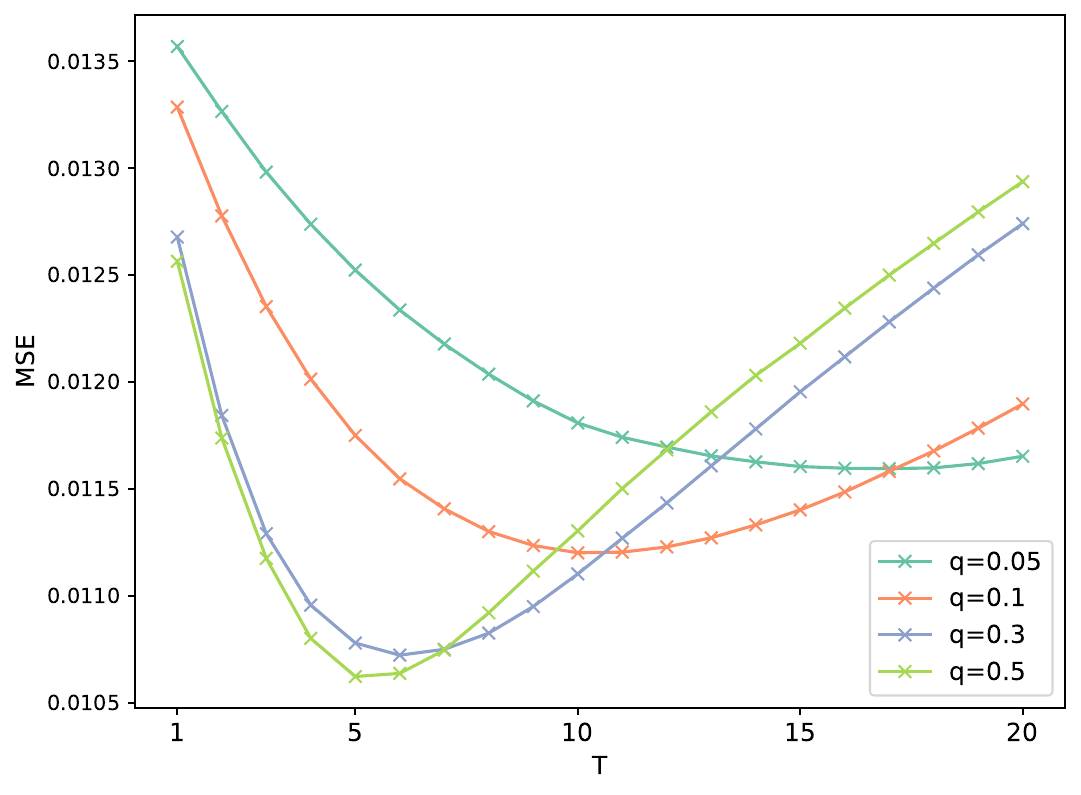}
\end{minipage}
\label{fig:forward_selection_feature_stability_sigma_nu_1}
}
\subfigure[Weak signal]{
\begin{minipage}{0.3\linewidth}
\centering
\includegraphics[width=\textwidth]{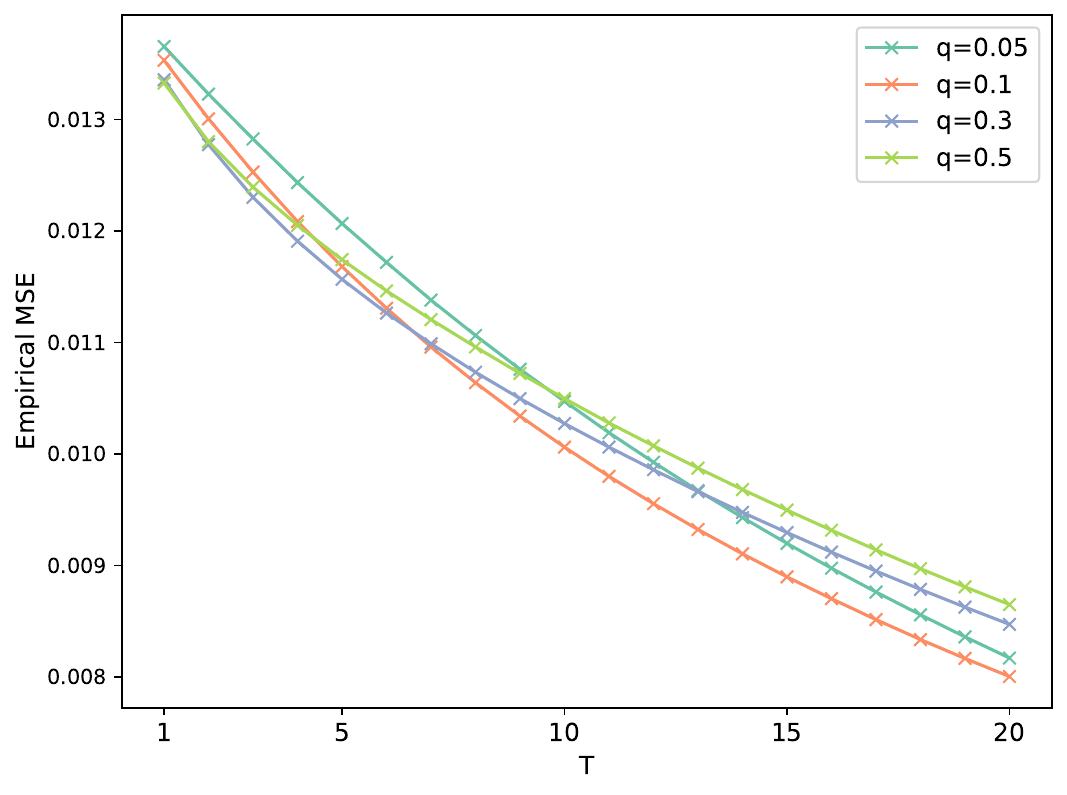}
\end{minipage}
\label{fig:forward_selection_feature_stability_empirical_sigma_nu_0}
}
\subfigure[Mixed signal]{
\begin{minipage}{0.3\linewidth}
\centering
\includegraphics[width=\textwidth]{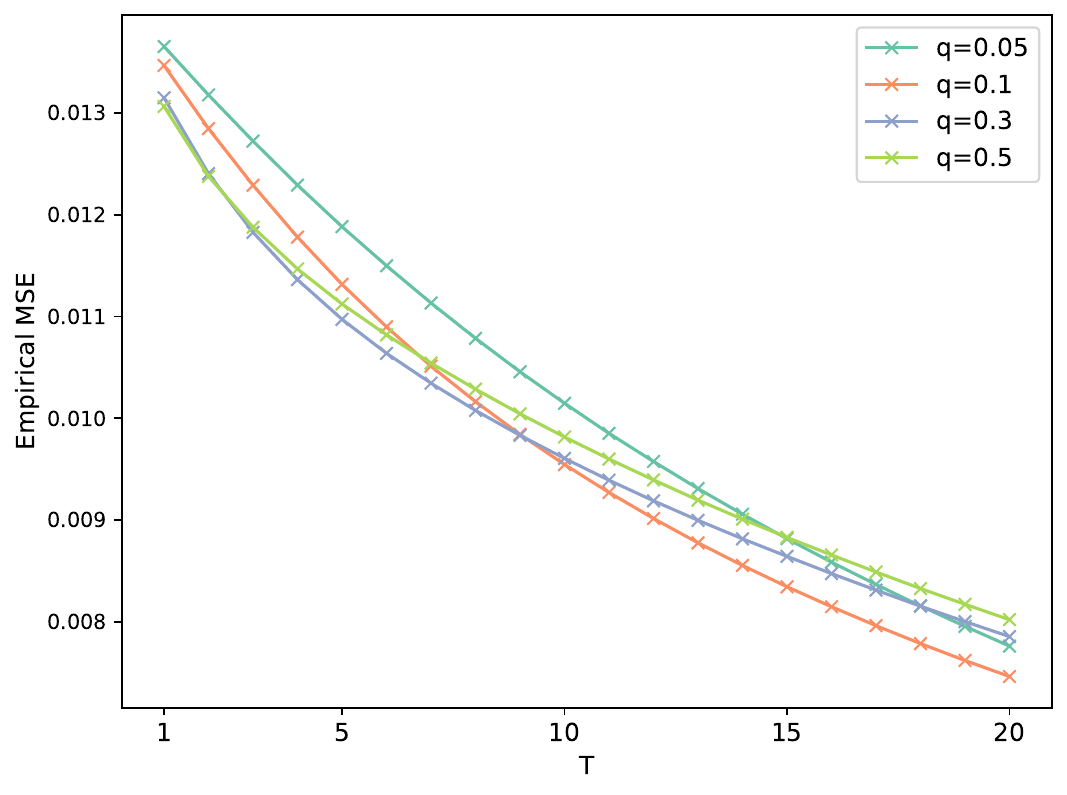}
\end{minipage}
\label{fig:forward_selection_feature_stability_empirical_sigma_nu_0.5}
}
\subfigure[Strong signal]{
\begin{minipage}{0.3\linewidth}
\centering
\includegraphics[width=\textwidth]{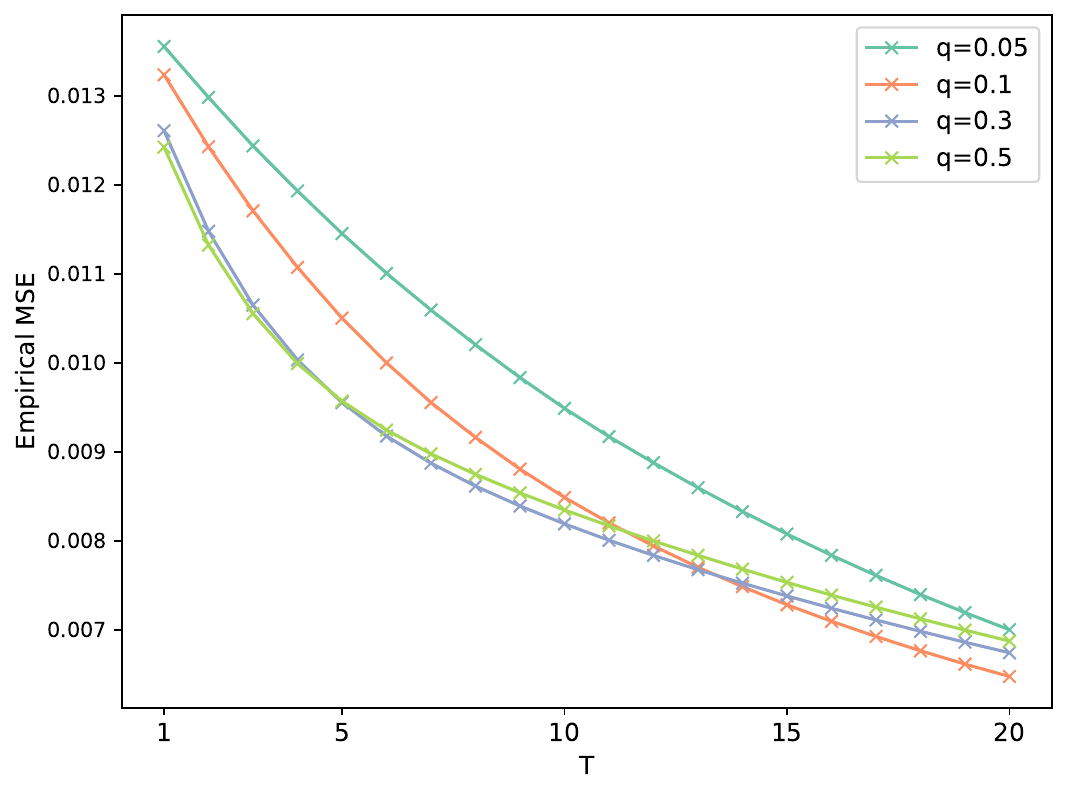}
\end{minipage}
\label{fig:forward_selection_feature_stability_empirical_sigma_nu_1}
}
\caption{The feature instability, generalization gap, generalization error, and empirical risk of RFS under different signal types.
Each cross corresponds to 100 repetitions.
We set $\sigma = 0.1$.
}
\label{fig:forward_selection_generalization_gap_sigma_0.1}
\end{figure}

We adopt the MARSadd setting from \citet{friedman1991multivariate}, now a standard benchmark for random forests \citep{mentch2020randomization, curth2024random}.
We generate $\bX$ from Unif$([0,1]^d)$ and set
\begin{align}\label{equ:model-marsadd}
    y = 0.1 e^{4 X^1}+\frac{4}{1+e^{-20\left(X^2-0.5\right)}}+3 X^3+2 X^4+X^5+\cN (0,\sigma^2).
\end{align}
We choose $n = 250$, $d=10$, and $\sigma = 1$.
We set $B=256$, $q\in \{0.1, 0.2, 0.3, 0.4, 0.5\}$, and $T\in [7]$.
In Figures~\ref{fig:random_forest_feature_stability} and~\ref{fig:random_forest_comparision}, we conduct a parallel set of experiments to those for RF, which brings similar conclusions.

We adopt the following simulation setup.
An orthogonal design matrix $\bX \in \mathbb{R}^{n \times d}$ is generated by applying QR decomposition to a matrix with i.i.d. standard normal entries. The true coefficient vector is set to $\bbeta^{*} = d^{-1/2}(1,\ldots,1) \in \mathbb{R}^d$, and the noise terms are drawn independently as $\varepsilon_i \sim \mathcal{N}(0, \sigma^2)$.  We fix the parameters as follows: $n = 250$, $d = 200$, and $\sigma = 1$. The number of bagging rounds is set to $B = 256$. We vary the feature subsampling ratio $q \in \{0.05, 0.1, 0.3, 0.5\}$ and the number of forward selection steps $T \in [20]$. Larger values of $T$ are not considered, as none of the chosen $q$ values satisfy the regime $qT \lesssim 1$ beyond this range.

We next explore the connection between feature instability and generalization.
Define $\bbeta^*_{\textrm{strong}} = (5^{-1/2},\ldots,5^{-1/2}, 0, \ldots, 0)\in \RR^{200}$ to be a sparse coefficient with $5$ nonzero positions, which represents a strong signal.
Let $\bbeta^*_{\textrm{weak}} = d^{-1/2} (1,\ldots,1)$ and $\bbeta^*_{\textrm{mixed}} = (\bbeta^*_{\textrm{weak}} + \bbeta^*_{\textrm{strong}}) / \|\bbeta^*_{\textrm{weak}} + \bbeta^*_{\textrm{strong}}\|_2$ represent a weak and a mixed signal, respectively.  Note that this sense of strong and weak depicts a different aspect of learning hardness opposed to the SNR \citep{zhou2023trees, liu2025randomization}.
We compare the relationship between feature instability, generalization gap, generalization error, and empirical risk in Figure \ref{fig:forward_selection_generalization_gap_sigma_0.1}, respectively, in four rows.
\begin{itemize}
    \item  Although there is no theoretical result establishing a relationship between the generalization gap and feature instability, we observe a similarity in their trends, regardless of whether the signal is concentrated or spread across the features.
    \item The generalization-risk curves, however, have different patterns under different signal types.
Observing the second last row of Figure \ref{fig:forward_selection_generalization_gap_sigma_0.1}, when the signal is weak, it is suggested to use a smaller subsampling ratio.
In contrast, when the signal is strong, using a larger $q$ and a smaller $T$ is a better choice.
This observation aligns with past experience that the bagging provides a regularization effect, which is more promising when the signal is weak \citep{lejeune2020implicit, mentch2020randomization}.
\item
In practice, the generalization error is unobservable.
Practitioners see only the empirical risk in the last row.
In all signal types, the empirical risk decreases as $T$ grows, i.e., the model capacity becomes larger. Thus, in practice, hyperparameter tuning involves balancing between the empirical risk and the generalization gap, where the latter can take evidence from feature instability.

\end{itemize}

\begin{figure}[h]
\centering
\subfigure[$\frac{1}{d}\sum_{j=1}^d\|\bw - \bw^{-j}\|_2^2$]{
\begin{minipage}{0.45\linewidth}
\centering
\includegraphics[width=\textwidth]{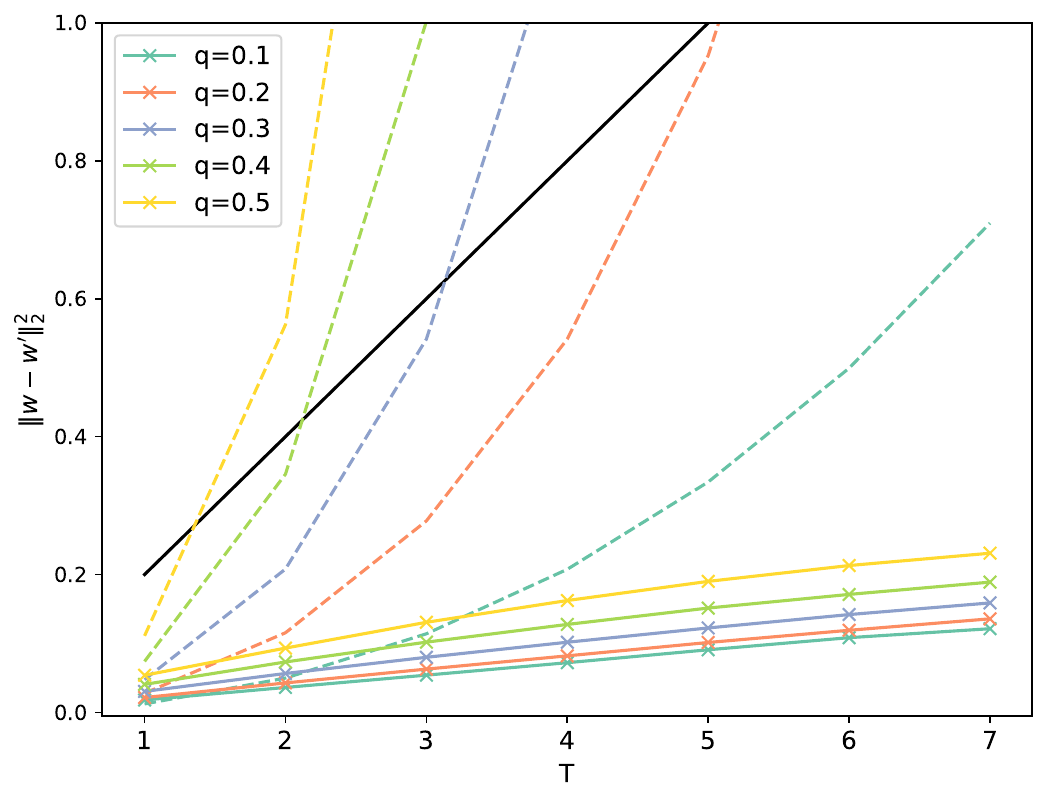}
\end{minipage}
\label{fig:random_forest_feature_stability_l2}
}
\subfigure[$\frac{1}{d}\sum_{j=1}^d|f(\bx) - f^{-j}(\bx)|_2^2$]{
\begin{minipage}{0.45\linewidth}
\centering
\includegraphics[width=\textwidth]{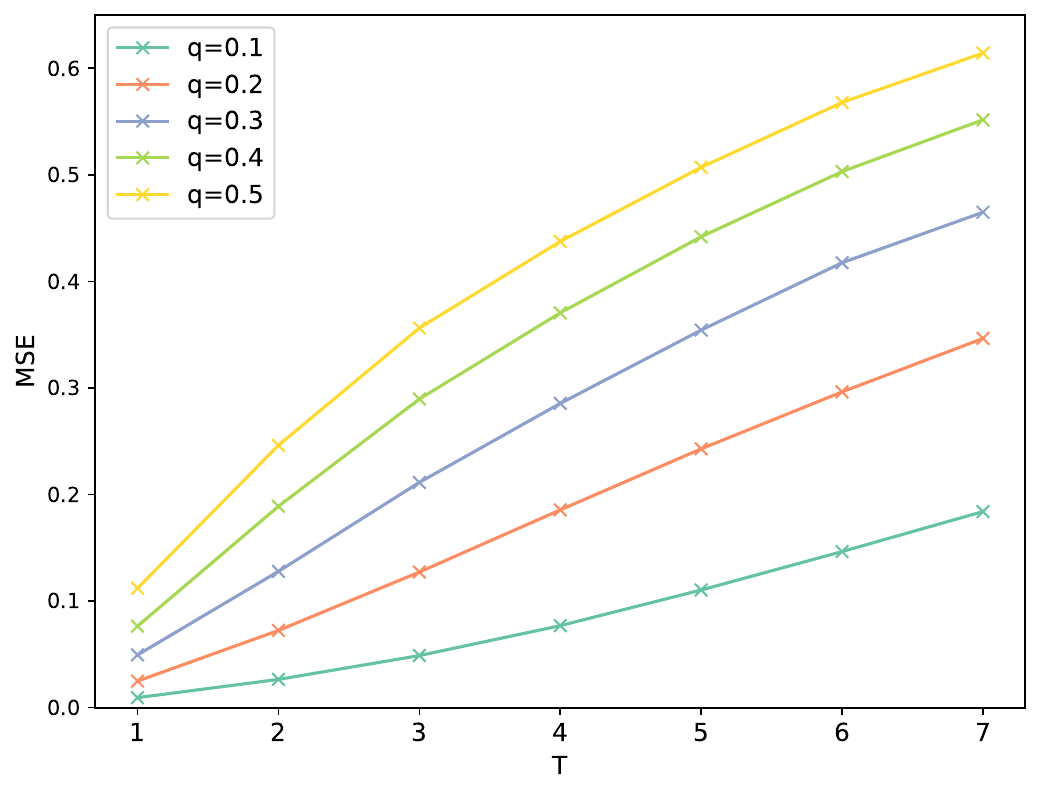}
\end{minipage}
\label{fig:random_forest_feature_stability_prediction_mse}
}
\caption{
Feature-instability metrics for RF.
Left: comparison of the feature importance between the feature-bagged RF and the non-bagged RF.
The real lines with crosses correspond to the simulated values over 100 repetitions under different feature subsampling ratios.
The dashed lines are the theoretical upper bound \eqref{equ:stability-of-rf-final-bound}.
The black solid line is the feature instability of non-bagged RF, i.e., a single tree.
Right:
comparison of the prediction value between the feature-bagged RF with different subsampling ratios.
We compute over an independent test set of size $n$.
}
\label{fig:random_forest_feature_stability}
\end{figure}

\begin{figure}[h]
\centering
\subfigure[Weak signal]{
\begin{minipage}{0.3\linewidth}
\centering
\includegraphics[width=\textwidth]{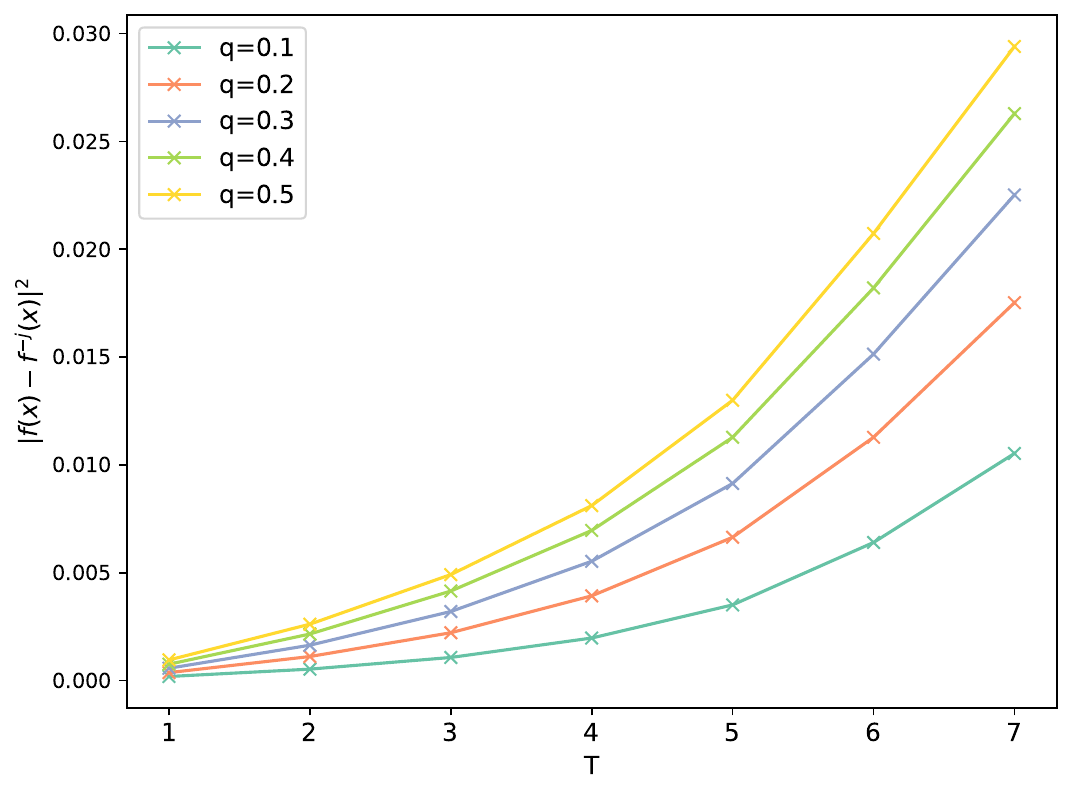}
\end{minipage}
\label{fig:random_forest_prediction_stability_sigma_1_nu_0}
}
\subfigure[Mixed signal]{
\begin{minipage}{0.3\linewidth}
\centering
\includegraphics[width=\textwidth]{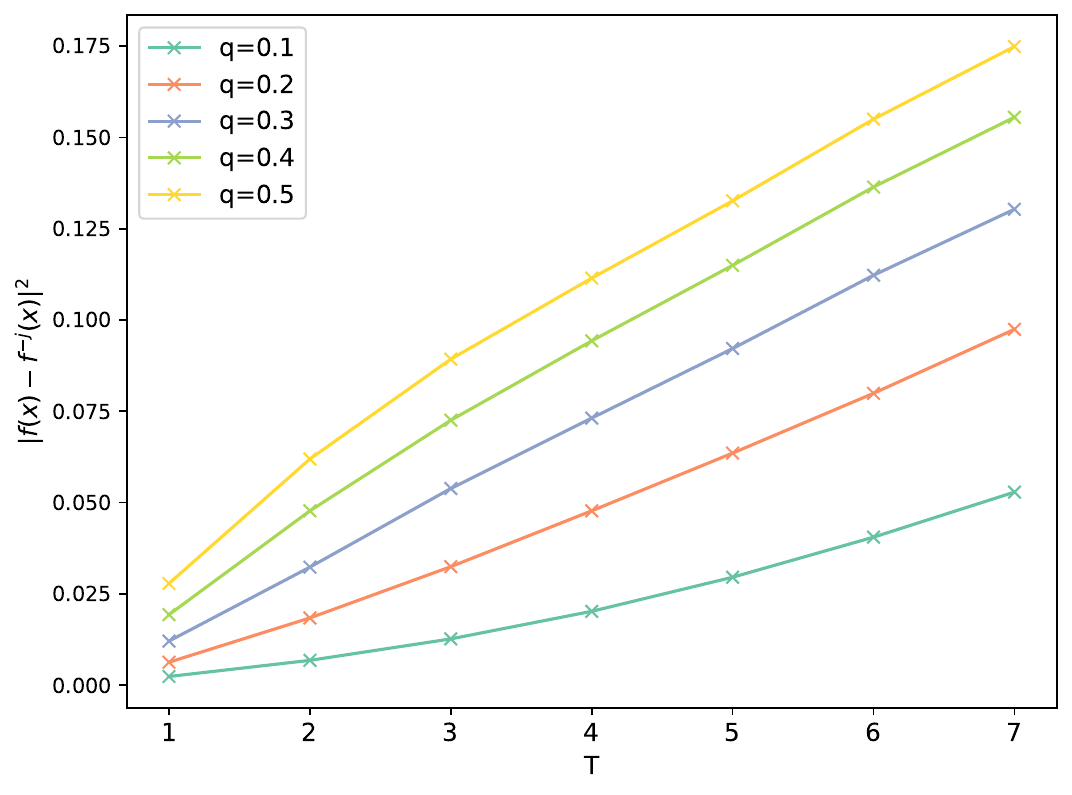}
\end{minipage}
\label{fig:random_forest_prediction_stability_sigma_1_nu_0.5}
}
\subfigure[Strong signal]{
\begin{minipage}{0.3\linewidth}
\centering
\includegraphics[width=\textwidth]{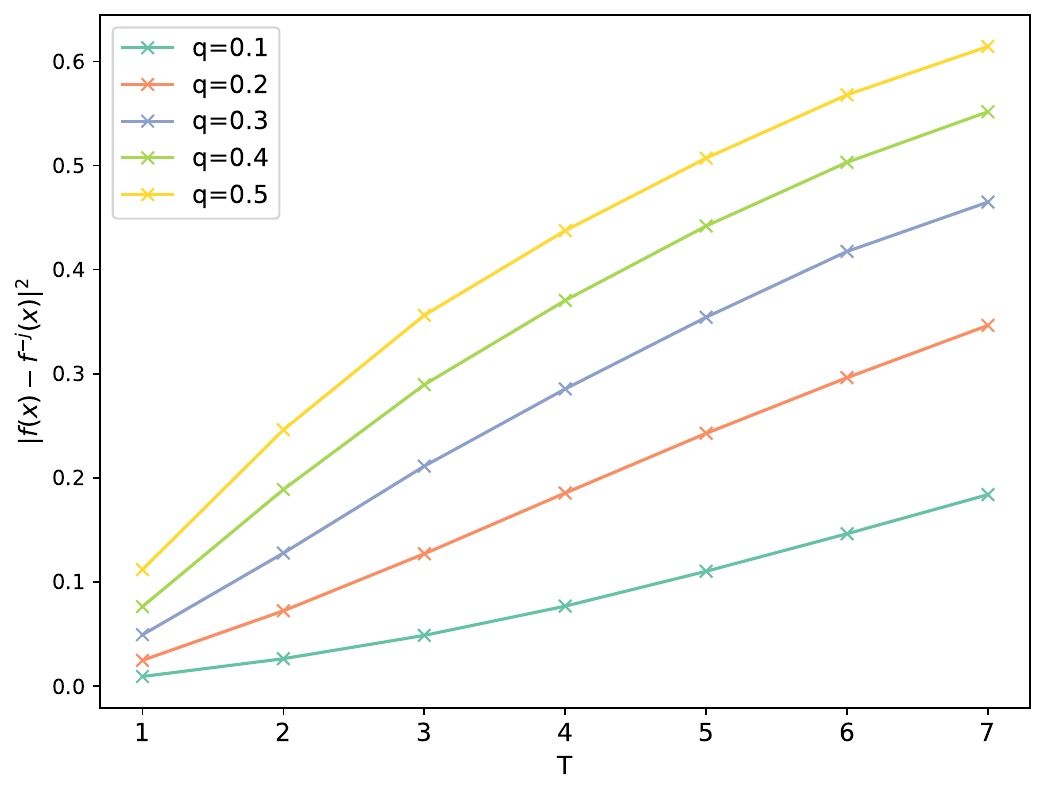}
\end{minipage}
\label{fig:random_forest_prediction_stability_sigma_1_nu_1}
}
\subfigure[Weak signal]{
\begin{minipage}{0.3\linewidth}
\centering
\includegraphics[width=\textwidth]{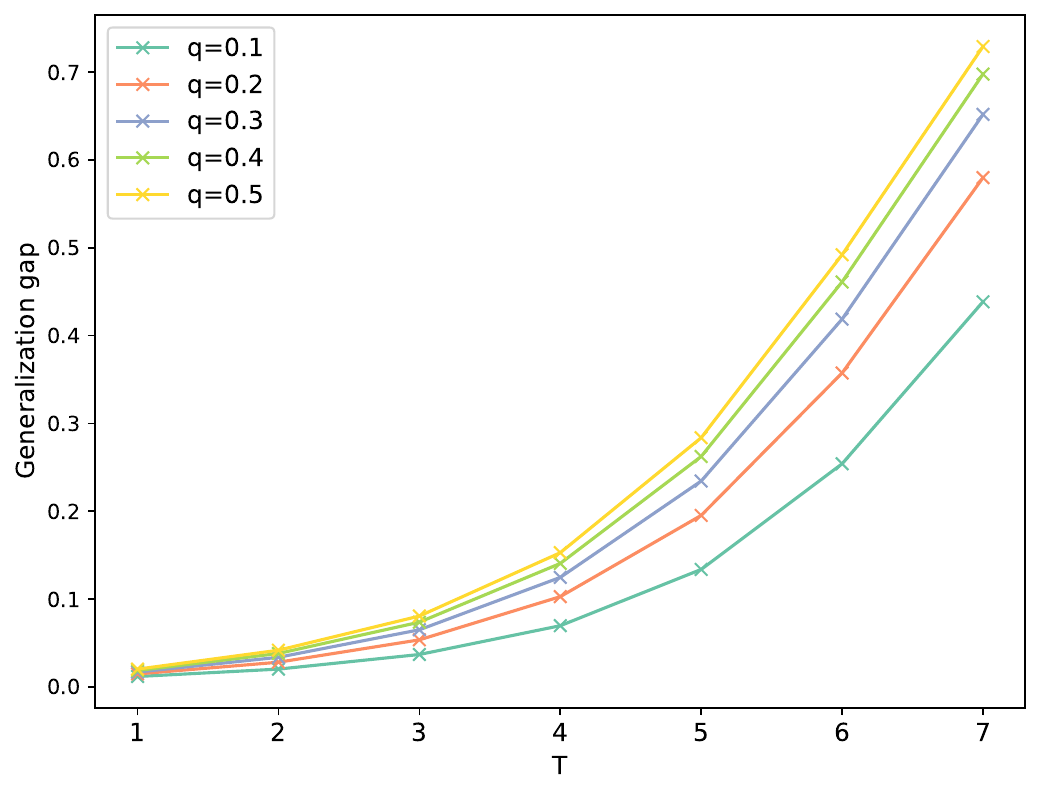}
\end{minipage}
\label{fig:random_forest_generalization_gap_sigma_1_nu_0}
}
\subfigure[Mixed signal]{
\begin{minipage}{0.3\linewidth}
\centering
\includegraphics[width=\textwidth]{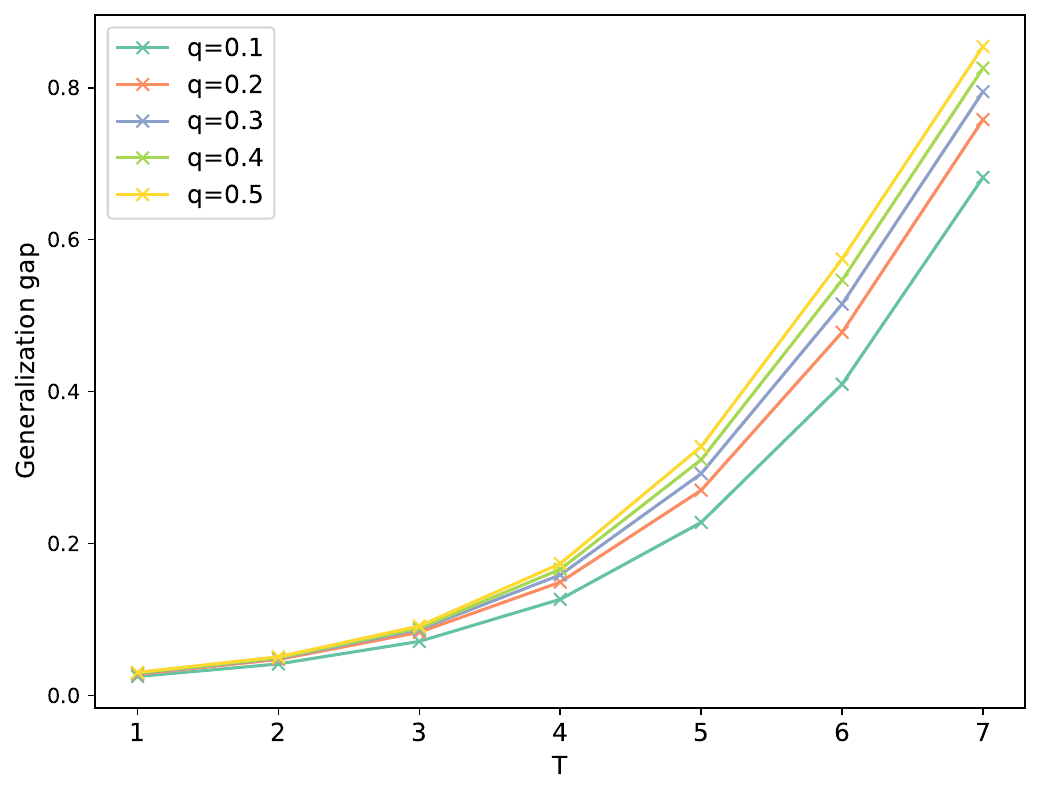}
\end{minipage}
\label{fig:random_forest_generalization_gap_sigma_1_nu_0.5}
}
\subfigure[Strong signal]{
\begin{minipage}{0.3\linewidth}
\centering
\includegraphics[width=\textwidth]{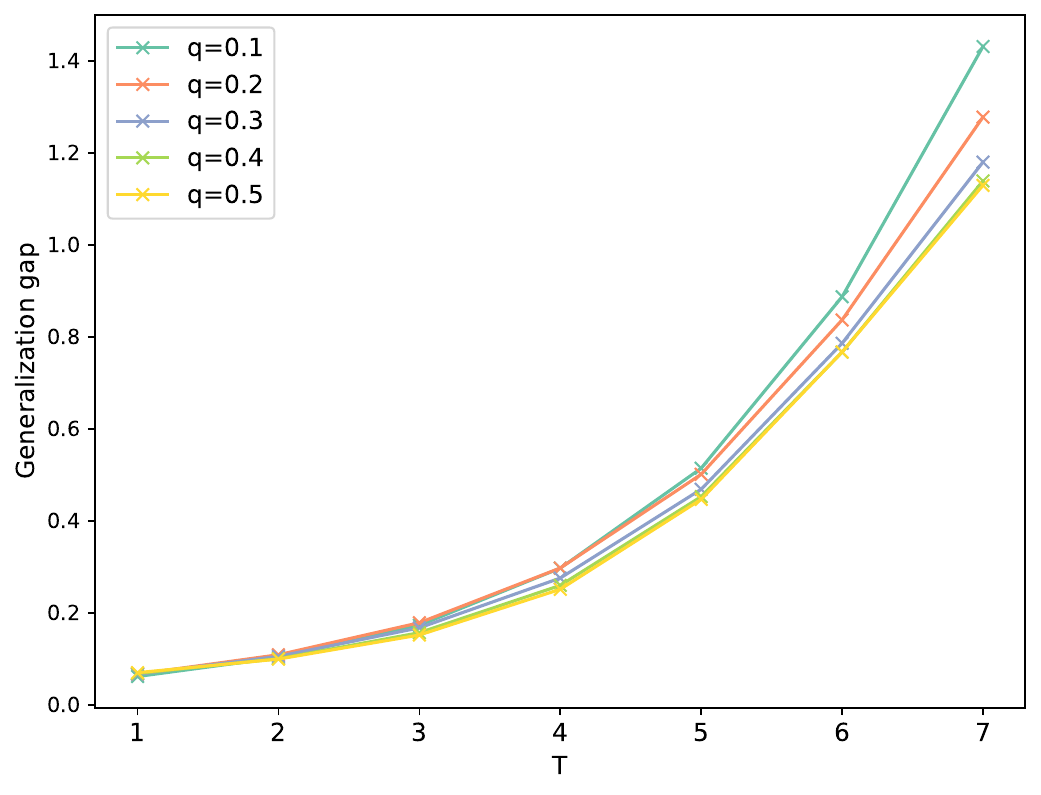}
\end{minipage}
\label{fig:random_forest_generalization_gap_sigma_1_nu_1}
}
\subfigure[Weak signal]{
\begin{minipage}{0.3\linewidth}
\centering
\includegraphics[width=\textwidth]{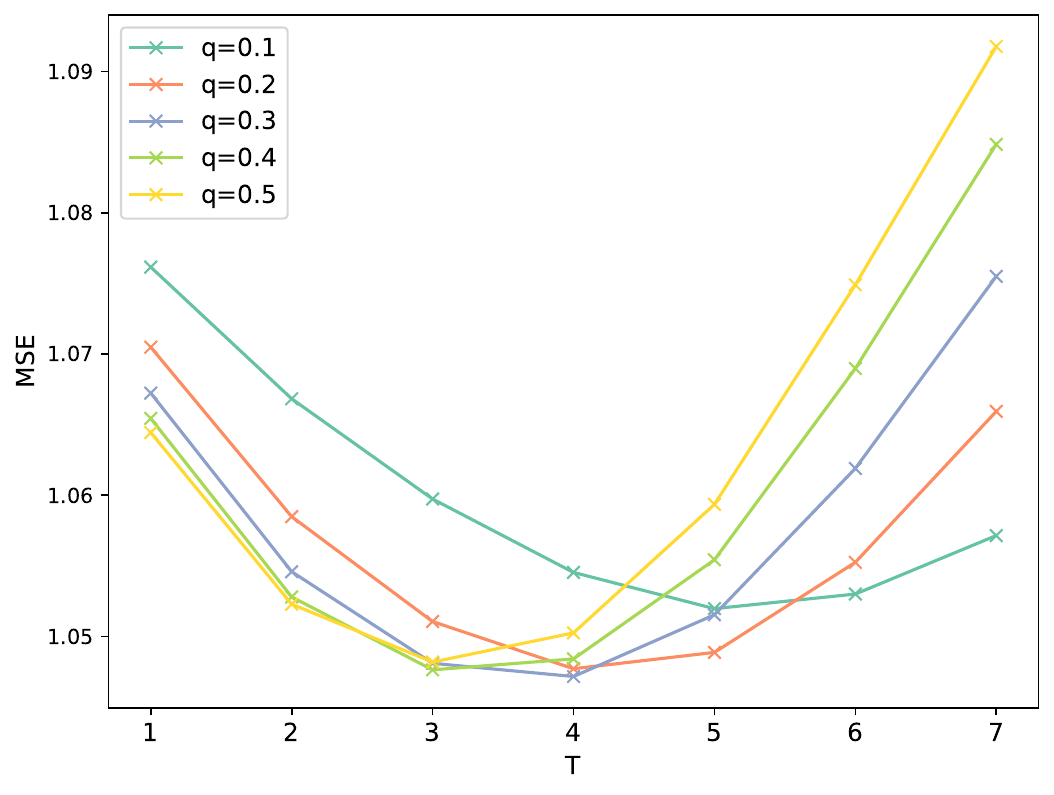}
\end{minipage}
\label{fig:random_forest_feature_stability_nu_0}
}
\subfigure[Mixed signal]{
\begin{minipage}{0.3\linewidth}
\centering
\includegraphics[width=\textwidth]{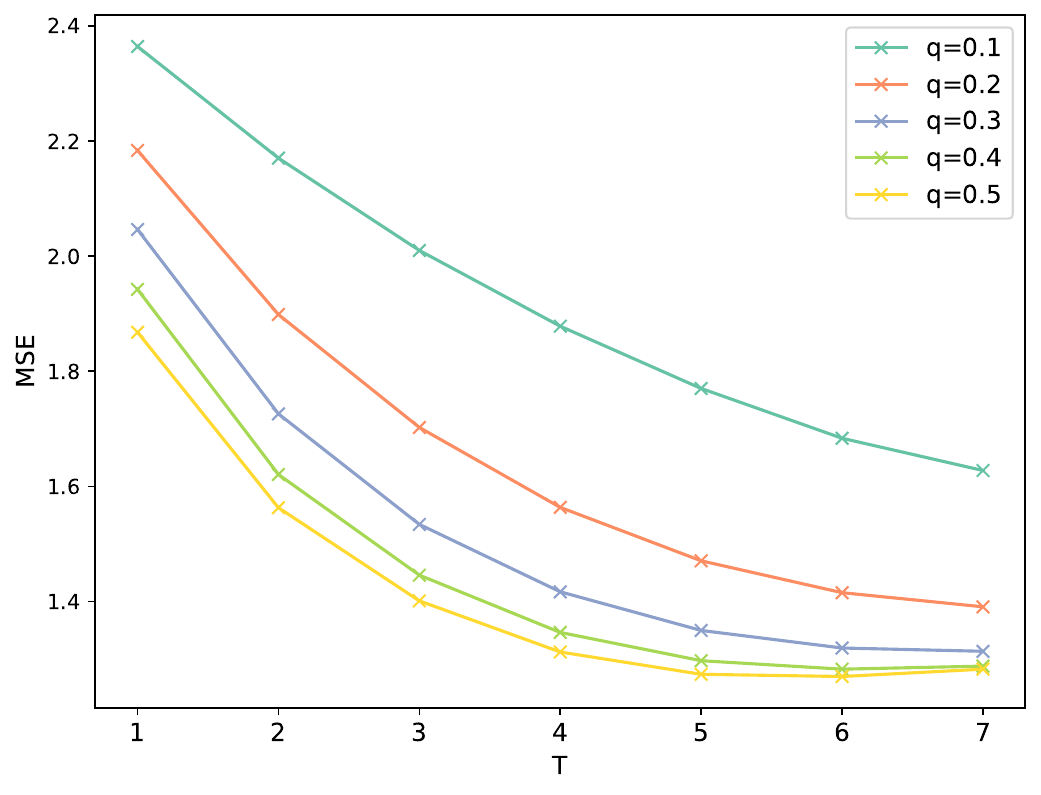}
\end{minipage}
\label{fig:random_forest_feature_stability_nu_0.5}
}
\subfigure[Strong signal]{
\begin{minipage}{0.3\linewidth}
\centering
\includegraphics[width=\textwidth]{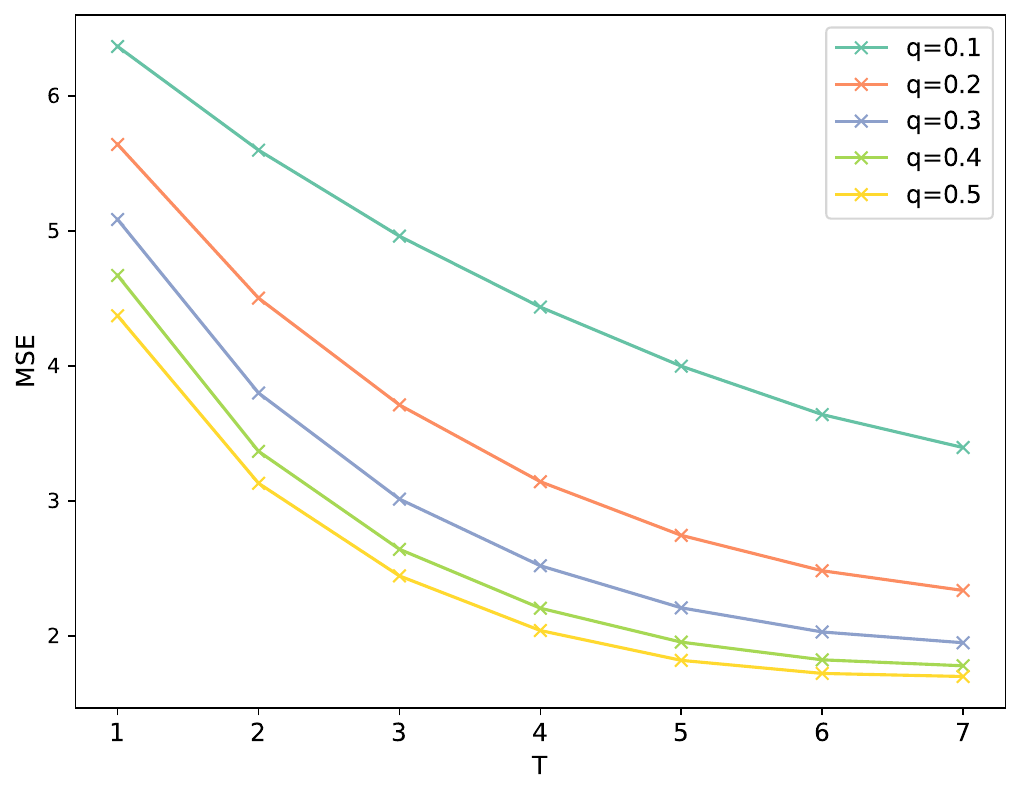}
\end{minipage}
\label{fig:random_forest_feature_stability_nu_1}
}
\subfigure[Weak signal]{
\begin{minipage}{0.3\linewidth}
\centering
\includegraphics[width=\textwidth]{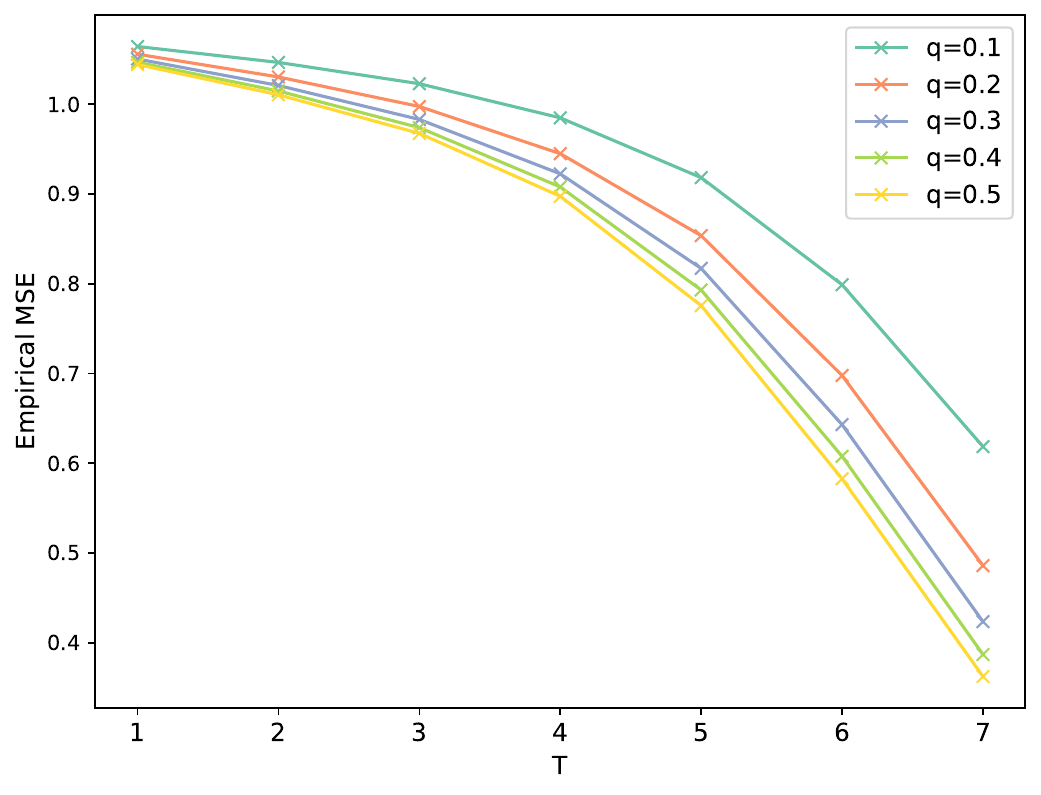}
\end{minipage}
\label{fig:random_forest_empirical_error_sigma_1_nu_0}
}
\subfigure[Mixed signal]{
\begin{minipage}{0.3\linewidth}
\centering
\includegraphics[width=\textwidth]{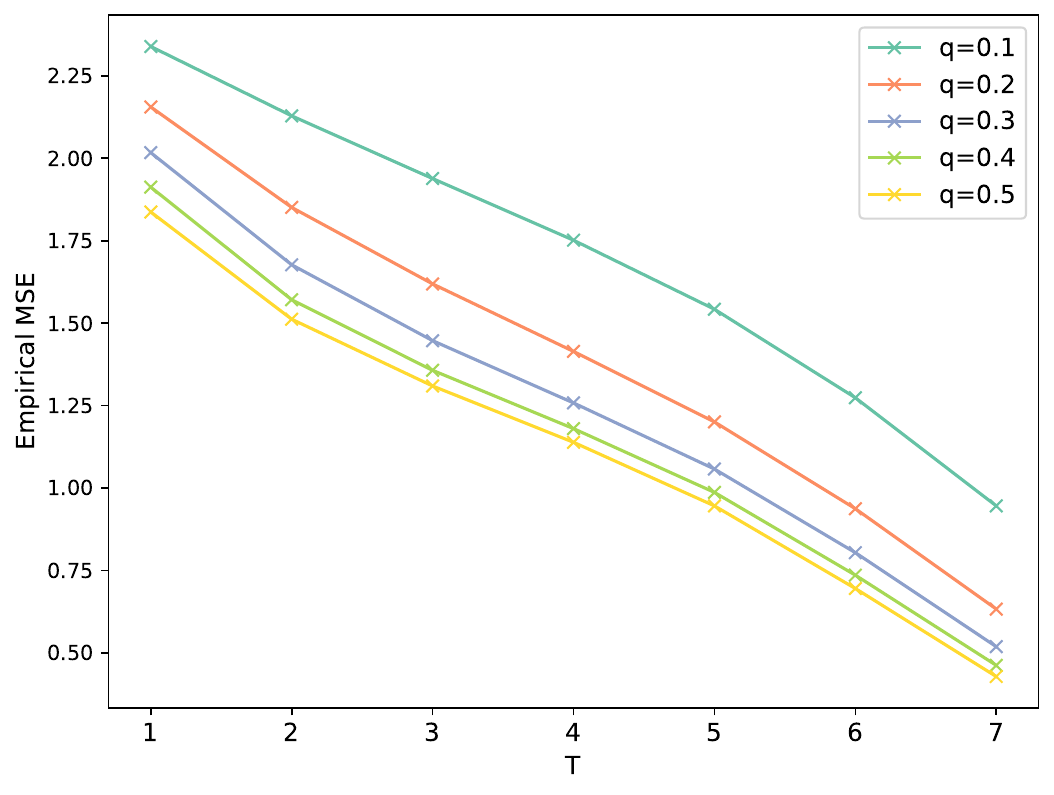}
\end{minipage}
\label{fig:random_forest_empirical_error_sigma_1_nu_0.5}
}
\subfigure[Strong signal]{
\begin{minipage}{0.3\linewidth}
\centering
\includegraphics[width=\textwidth]{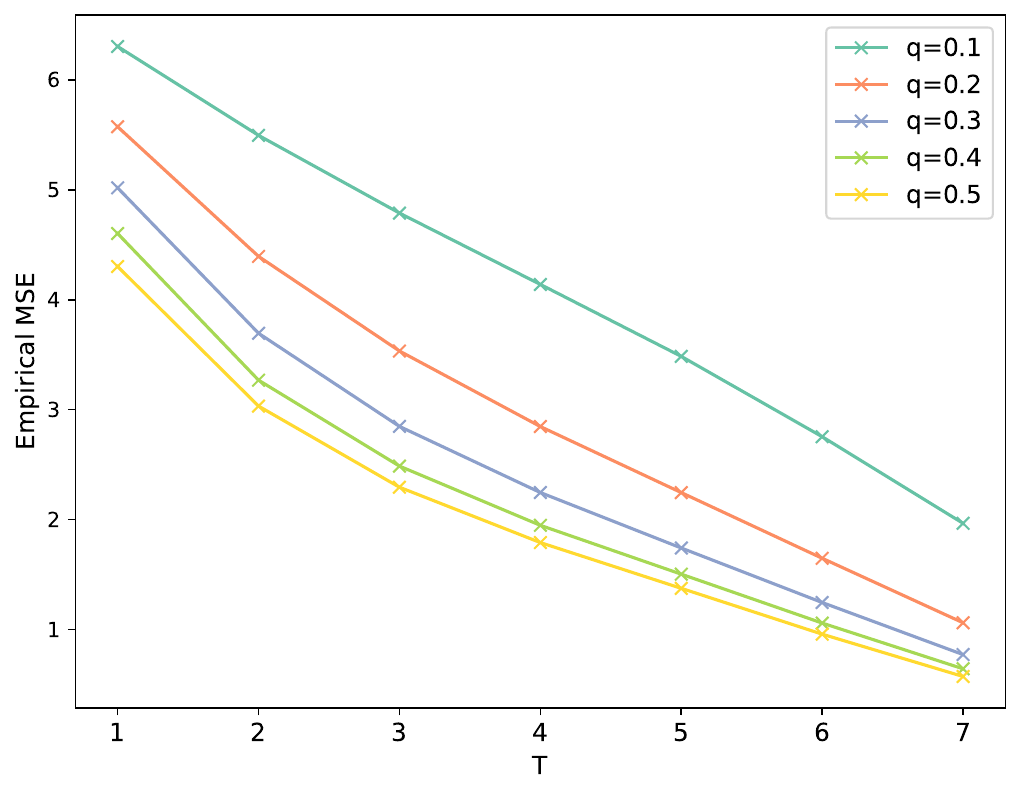}
\end{minipage}
\label{fig:random_forest_empirical_error_sigma_1_nu_1}
}
\caption{The feature instability, generalization gap, generalization error, and empirical risk of RF under different signal types.
Each cross corresponds to 100 repetitions.
}
\label{fig:random_forest_comparision}
\end{figure}

\subsection{Proofs of Results in Section~\ref{sec:extention-metric-space}}

\begin{proof}[\textbf{Proof of Proposition \ref{prop:stability-feature}}]
    The proposition is a special case of Proposition \ref{prop:stability-edit-norm}, where $\operatorname{rad}(\cW) = \operatorname{rad}([-M,M]) = M$.
\end{proof}

\begin{proof}[\textbf{Proof of Proposition \ref{prop:stability-edit-norm}}]
The proof is a direct extension of that of Proposition \ref{prop:stability-feature}, except that here we are considering the norm in general spaces.
We use $\bnu$ to denote a subsampled set of features.
Let event $E_{j, \bnu} = \{ j \notin \bnu\}$.
Denote the quantity $L_{j} = \widehat{\bw}- \widehat{\bw}^{-j}$, for which we similarly define $ \widehat{\bw}^{-j}  = \mathbb{E}_{ \bnu}\left[\cA(\cD^{-j};\bnu) \right]$.
Recall that we have $\widehat{\bw} = \mathbb{E}_{ \bnu}\left[\cA(\cD;\bnu) \right]$.
We also denote $\bw^{\bnu} = \cA(\cD;\bnu)$.
As a result, we can express $L_{j}$ as
\begin{align*}
    L_{j} = & \mathbb{E}_{ \bnu}\left[\cA(\cD;\xi) - \cA(\cD^{-j};\xi) \right] \\
    = & \mathbb{E}_{  \bnu}\left[\mathbb{E}_{\bnu}\left[\bw^{\bnu}\right] - \bw^{\bnu} \mid  E_{j, \bnu} \right]\\
    = & \frac{1}{\mathbb{P}(E_{j, \bnu})}\mathbb{E}_{ \bnu }\left[\left(\mathbb{E}_{\bnu}\left[\bw^{\bnu}\right] - \bw^{\bnu}\right)\mathbf{1} \{ E_{j, \bnu}\} \right].
\end{align*}
Note that we took uniform sampling without replacement.
Thus, $\mathbb{P}(E_{j, \bnu}) =1 - q$.
Since $\mathbb{E}_{  \bnu}\left[\mathbb{E}_{\bnu}\left[\bw^{\bnu}\right] - \bw^{\bnu}\right] = 0$, we have
\begin{align*}
    L_{j} = &\frac{1}{\mathbb{P}(E_{j, \bnu})}\mathbb{E}_{\bnu}\left[\left(\mathbb{E}_{\bnu}\left[\bw^{\bnu}\right] - \bw^{\bnu}\right)\mathbf{1} \{ E_{j, \bnu}\} \right]
    \\
     = & \frac{1}{1 -q}\mathbb{E}_{ \bnu}\left[\left(\mathbb{E}_{\bnu}\left[\bw^{\bnu}\right] - \bw^{\bnu}\right)\left(\mathbf{1} \{ E_{j, \bnu}\}  - (1 - q)\right)\right].
\end{align*}
Then, the sum of norms has
\begin{align*}
    \sum_{j=1}^d \| \widehat{\bw}- \widehat{\bw}^{-j}\|_{\cH}^2
    = &
    \sum_{j=1}^d \|  L_j\|_{\cH}^2 \\
    = & \sum_{j=1}^d \left\langle  L_j, \frac{1}{1 -q}\mathbb{E}_{ \bnu}\left[\left(\mathbb{E}_{\bnu}\left[\bw^{\bnu}\right] - \bw^{\bnu}\right)\left(\mathbf{1} \{ E_{j, \bnu}\}  - (1 - q)\right)\right]\right\rangle_{\cH} \\
    \stackrel{(i)}{=} &
    \sum_{j=1}^d \frac{1}{1 -q} \mathbb{E}_{ \bnu}\left[\left\langle  L_j, \left(\mathbb{E}_{\bnu}\left[\bw^{\bnu}\right] - \bw^{\bnu}\right)\left(\mathbf{1} \{ E_{j, \bnu}\}  - (1 - q)\right)\right\rangle_{\cH}\right]\\
    = &
    \sum_{j=1}^d \frac{1}{1 -q} \mathbb{E}_{ \bnu}\left[\left\langle  L_j \left(\mathbf{1} \{ E_{j, \bnu}\}  - (1 - q)\right), \mathbb{E}_{\bnu}\left[\bw^{\bnu}\right] - \bw^{\bnu}\right\rangle_{\cH}\right]\\
    = & \frac{1}{1 -q}
     \mathbb{E}_{ \bnu}\left[\left\langle  \sum_{j=1}^d  L_j \left(\mathbf{1} \{ E_{j, \bnu}\}  - (1 - q)\right), \mathbb{E}_{\bnu}\left[\bw^{\bnu}\right] - \bw^{\bnu}\right\rangle_{\cH}\right]
\end{align*}
Here, $(i)$ holds since $L_j$ is a constant and thus is independent of $\bnu$.
Applying the Cauchy-Schwarz inequality, we have
\begin{align}\nonumber
    \sum_{j=1}^d \| \widehat{\bw}- \widehat{\bw}^{-j}\|_{\cH}^2  \leq & \frac{1}{1 -q} \sqrt{\mathbb{E}_{ \bnu}\left[\left\|  \sum_{j=1}^d  L_j \left(\mathbf{1} \{ E_{j, \bnu}\}  - (1 - q)\right)\right\|_{\cH}^2\right] \cdot \mathbb{E}_{ \bnu}\left[ \left\|\mathbb{E}_{\bnu}\left[\bw^{\bnu}\right] - \bw^{\bnu}\right\|_{\cH}^2\right]} \\
    \leq & \frac{\operatorname{rad}(\cW)}{1 -q} \sqrt{\mathbb{E}_{ \bnu}\left[\left\|  \sum_{j=1}^d  L_j \left(\mathbf{1} \{ E_{j, \bnu}\}  - (1 - q)\right)\right\|_{\cH}^2\right]}.
    \label{equ:normspace1}
\end{align}
Here, the last step is true because
\begin{align*}
    \mathbb{E}_{ \bnu}\left[ \left\|\mathbb{E}_{\bnu}\left[\bw^{\bnu}\right] - \bw^{\bnu}\right\|_{\cH}^2\right] =
    \inf_{\bw\in\cW}
    \mathbb{E}_{ \bnu}\left[ \left\|\bw - \bw^{\bnu} \right\|_{\cH}^2\right] \leq \inf_{\bw\in \cW } \sup_{\bw'\in \cW}  \left[ \left\|\bw - \bw' \right\|_{\cH}^2\right] \leq \operatorname{rad}^2(\cW),
\end{align*}
where the first equality is due to the fact that expectation is the minimizer of squared loss.
We calculate the remaining term by noticing
\begin{align*}
\mathbb{E}_{\bnu}\left[ \mathbf{1}\{E_{j_1, \bnu}^c\} \mathbf{1}\{E_{j_2, \bnu}^c\}\right] =  \begin{cases}
 \frac{s(s - 1)}{d (d-1)}
& \text{ if }  j_1 \neq j_2, \\
\frac{s}{d}  & \text{ if }  j_1 = j_2
\end{cases}
\end{align*}
and thus
\begin{align*}
    \operatorname{Cov}\left[ \mathbf{1}\{E_{j_1, \bnu}^c\}, \mathbf{1}\{E_{j_2, \bnu}^c\}\right] = \mathbb{E}_{\bnu}\left[ \mathbf{1}\{E_{j_1, \bnu}^c\} \mathbf{1}\{E_{j_2,\bnu}^c\}\right] - q^2 =  \begin{cases}
 \frac{s(s - 1)}{d (d-1)} - \frac{s^2}{d^2}
& \text{ if }  j_1 \neq j_2, \\
\frac{s}{d} - \frac{s^2}{d^2}  & \text{ if }  j_1 = j_2.
\end{cases}
\end{align*}
Then we expand \eqref{equ:normspace1} by
\begin{align*}
    \mathbb{E}_{ \bnu}\left[\left\|  \sum_{j=1}^d  L_j \left(\mathbf{1} \{ E_{j, \bnu}\}  - (1 - q)\right)\right\|_{\cH}^2\right] = & \sum_{j_1, j_2}^{d} \langle L_{j_1},  L_{j_2}  \rangle \operatorname{Cov}\left[ \mathbf{1}\{E_{j_1, \bnu}^c\}, \mathbf{1}\{E_{j_2, \bnu}^c\}\right] \\
    = &
    \sum_{j_1\neq j_2} \langle L_{j_1},  L_{j_2}  \rangle  \left(\frac{s(s - 1)}{d (d-1)} - \frac{s^2}{d^2}\right) + \sum_{j=1}^d \langle L_{j},  L_{j}  \rangle
    \left(\frac{s}{d} - \frac{s^2}{d^2}\right) \\
    = &  \left\|\sum_{j=1}^d L_{j} \right\|_{\cH}^2  \left(\frac{s(s - 1)}{d (d-1)} - \frac{s^2}{d^2}\right) + \sum_{j=1}^d \langle L_{j},  L_{j}  \rangle
    \left(\frac{s}{d} - \frac{s(s - 1)}{d (d-1)}\right) \\
    \leq & \frac{d}{d-1} q (1 - q) \sum_{j=1}^d \|L_j\|^2_{\cH}.
\end{align*}
Bringing back this into \eqref{equ:normspace1}, we have
\begin{align*}
    \sum_{j=1}^d \|  L_j\|_{\cH}^2\leq  \frac{\operatorname{rad}(\cW)}{1 -q} \sqrt{\frac{d}{d-1} q (1 - q) \sum_{j=1}^d \|L_j\|^2_{\cH}},
\end{align*}
which leads to
\begin{align*}
   \frac{1}{d} \sum_{j=1}^d \|  L_j\|_{\cH}^2\leq \frac{\operatorname{rad}^2(\cW)}{d-1} \frac{q}{1 -q}.
\end{align*}

\end{proof}

\subsection{Related Contexts and Proofs of Results in Section \ref{sec:recursivesubsampling}}

\begin{proof}[Proof of Proposition \ref{prop:stabilityofgeneralrecursiveobject}]
The proof is done by recursively applying equation \eqref{equ:recurence} and Proposition \ref{prop:stability-edit-norm}.
\end{proof}

Our final goal is to bound the feature instability measure: 
\begin{align*}
    \frac{1}{d} \sum_{j=1}^d \left\|\EE_{\xi} \left[ \bw_T \right] - \EE_{\xi} \left[ \bw_T^{-j} \right]\right\|_{\cH}^2,
\end{align*}
where $T$ is some preset horizon length.
We first consider a fixed $j\in[d]$ for the convenience of derivation.
Let $\Delta \bw_t^{- j} := \bw_T - \bw_T^{-j}$ and write
\begin{align}
\left\|  \EE_{\xi} \left[\Delta \bw_t^{- j}\right]  \right\|_{\cH}^2 = \left\|\EE_{\xi} \left[S_{\xi_{t}}  \circ \cdots \circ S_{\xi_1 } (\bw_0)\right] -
    \EE_{\xi} \left[
    S^{-j}_{\xi_{t}} \circ  \cdots \circ S^{-j}_{\xi_1 } (\bw_0)\right]\right\|_{\cH}^2 . \label{equ:definitionofrecursivequantity}
\end{align}
In the above quantity, the discrepancy between $\bw_t$ and $\bw_{t}^{-j}$ accumulates as $t$ grows as each step introduces additional discrepancy due to the difference between $S_{\xi_t}$ and $S_{\xi_t}^{-j}$.

The recursive nature of the procedure indicates that we should analyze the instability using a peeling strategy. Specifically, at the $t$-th recursive step, we decompose the instability into two terms.
The first term, referred to as the peel at step $t$, deals with the heterogeneity caused by the removal in the $t$-th step, formally defined as
\begin{align}\label{equ:peel}
\text{peel}_t^j = \EE_{\xi} \left[
S_{\xi_{t}} \circ S_{\xi_{t-1}} \circ \cdots \circ S_{\xi_1 } (\bw_0) \right]- \EE_{\xi} \left[
S^{-j}_{\xi_{t}} \circ S_{\xi_{t-1}} \circ \cdots \circ S_{\xi_1 } (\bw_0) \right].
\end{align}
The second term is the error accumulated along $1,\ldots,t-1$, referred to as the pith at step $t$,  is defined as
\begin{align}\label{equ:pith}
\text{pith}_t^j =\EE_{\xi} \left[
S^{-j}_{\xi_{t}} \circ S_{\xi_{t-1}} \circ \cdots \circ S_{\xi_1 } (\bw_0)\right] -  \EE_{\xi} \left[
S^{-j}_{\xi_{t}} \circ S^{-j}_{\xi_{t-1}} \circ \cdots \circ S^{-j}_{\xi_1 } (\bw_0) \right].
\end{align}

Compared to $\Delta \bw_{t-1}^{-j}$, the pith at step $t$ can be expressed as
\[
\text{pith}_t^j := \EE_{\xi} \left[ S^{-j}_{\xi_{t}} \circ  \Delta \bw_{t-1}^{-j} \right],
\]
which applies the random operator $S^{-j}_{\xi_{t}}$ to the change $\Delta \bw_{t-1}^{-j}$ from the previous step. We will then show the one-step condition~\eqref{equ:contraction}:
\begin{align}
\left\|\operatorname{ pith }_t^j\right\|_{\cH}
\leq
(1 + \delta_t)\left\|\EE_{\xi}\left[\Delta \bw_{t-1}^{-j}\right]\right\|_{\cH},
\end{align}
where the inflation factor $1 + \delta_t$ depends on the algorithm $\cA$ and the geometry of the space $\cH$.

Applying the triangle inequality, we bound the recursive quantity from \eqref{equ:definitionofrecursivequantity} as
\begin{align*}
 \frac{1}{d} \sum_{j=1}^d   \left\|\EE_{\xi}\left[\Delta \bw_t^{-j}\right]\right\|_{\cH}^2 = &
  \frac{1}{d} \sum_{j=1}^d
  \left\|\operatorname{ peel }_t^j +
 \operatorname{ pith }_t^j\right\|_{\cH}^2  \\
 \leq &  \frac{1}{d} \sum_{j=1}^d  \left(\left\|\operatorname{ peel }_t^j \right\|_{\cH} + \left\|
 \operatorname{ pith }_t^j\right\|_{\cH}\right)^2  \\
 \leq &  \frac{1}{d} \sum_{j=1}^d  \left(\left\|\operatorname{ peel }_t^j \right\|_{\cH} +  (1 + \delta_t)\left\|\EE_{\xi}\left[\Delta \bw_{t-1}^{-j}\right]\right\|_{\cH} \right)^2.
\end{align*}
Then, by Minkowski’s inequality, we obtain the recurrence bound:
\begin{align}
\sqrt{\frac{1}{d} \sum_{j=1}^d   \left\|\EE_{\xi}\left[\Delta \bw_t^{-j}\right]\right\|_{\cH}^2}
\leq
\underbrace{\sqrt{\frac{1}{d} \sum_{j=1}^d  \left\|\operatorname{ peel }_j \right\|_{\cH}^2 }}_{\text{Proposition \ref{prop:stability-edit-norm}}}
+   \underbrace{(1 + \delta_t)\sqrt{ \frac{1}{d} \sum_{j=1}^d\left\|\EE_{\xi}\left[\Delta \bw_{t-1}^{-j}\right]\right\|_{\cH}^2}}_{\text{ step } t - 1}.
\label{equ:recurence}
\end{align}
The first term on the right-hand side corresponds to the instability due to one-step fitting, which can be bounded via Proposition~\ref{prop:stability-edit-norm}. The inequality \eqref{equ:recurence} forms an inhomogeneous linear recurrence, which admits a closed-form solution.
This is formalized in the following proposition.

We conclude with a brief remark on the peeling strategy. For instance stability, a closely related analysis exists for stochastic gradient descent (SGD) \citep{hardt2016train, lei2020fine}, which was noted earlier in this section as a key motivating example. That analysis relies on assumptions analogous to Assumption~\ref{asp:hilbert} and the inflation condition~\eqref{equ:contraction}. As a result, the corresponding instability bounds for SGD also exhibit cumulative error growth with the number of update steps.

\subsection{Derivation of Feature Instability for Non-bagged RFS and RF}\label{app:non-bagged-stability}

We derive the feature instability without bagging.

We first consider RFS.
Notice that without bagging, every $\bw_T$ includes $T$ features without doubt.
Thus, $\EE_{\xi}\left[\bw_T\right]$ has $T$ positions being one and others being zero.
Moreover, the selected features are the $T$ features with the largest $\by^{\top}\bX^{j } \bX^{j \top} \by$, as suggested by \eqref{equ:selectionindexoforthogonalfeature}.
Thus, as long as $j$ is in the group of $T$ features with largest $\by^{\top}\bX^{j } \bX^{j \top} \by$, the term
$\EE_{\xi}\left[\bw_T\right]- \EE_{\xi}\left[\bw_T^{-j}\right]$ has one position being 1 and the other being -1.
Thus,
$\left\|\EE_{\xi}\left[\bw_T\right]- \EE_{\xi}\left[\bw_T^{-j}\right]\right\|^2 = 2$.
There are in total $T$ out of $d$ such terms, and thus
\begin{align*}
   \frac{1}{d}  \sum_{j=1}^d\left\|\EE_{\xi}\left[\bw_T\right]- \EE_{\xi}\left[\bw_T^{-j}\right]\right\|^2 = \frac{2T}{d} .
\end{align*}

For random forest, things are the same.
$\EE_{\xi}\left[\bw_T\right]$ has in total $T$ positions being one and others being zero.
Removing $j$-th feature results in a $\EE_{\xi}\left[\bw_T\right]$ with the positions in $j$-th column being zero, and equally many entries in other columns being one.
Thus,
$\left\|\EE_{\xi}\left[\bw_T\right]- \EE_{\xi}\left[\bw_T^{-j}\right]\right\|^2 = 2 \cdot \sum_{i=1}^T(\EE_{\xi}\left[\bw_T\right])_i^j$, which leads to
\begin{align*}
   \frac{1}{d}  \sum_{j=1}^d\left\|\EE_{\xi}\left[\bw_T\right]- \EE_{\xi}\left[\bw_T^{-j}\right]\right\|^2 = \frac{1}{d} \sum_{j=1}^d 2 \cdot \sum_{i=1}^T(\EE_{\xi}\left[\bw_T\right])_i^j = \frac{2T}{d}.
\end{align*}

\subsection{Derivation of the Affine Random Forest Encoding Map}

\begin{proof}[\textbf{Proof of Proposition \ref{prop:linear-map-tree}}]
Fix the query point \(\bx\).  Let \(I_{k,r}(\bx)\) be the dyadic interval in
coordinate \(k\) that contains \(x^k\) after \(r\) midpoint splits along
coordinate \(k\).  Define
\[
A_{k,r}(\bx)
:=
2^r\int_{I_{k,r}(\bx)} f_k^*(u)\,du,
\qquad r=0,\ldots,T.
\]
If \(r_k(\bw)=\|\bw^k\|_1\), then under the additive model and the uniform
design,
\[
\EE_{\by\mid\bX,\bw}[f^{\bw}(\bx)]
=
\sum_{k=1}^d A_{k,r_k(\bw)}(\bx).
\]
Because the encoding is column-filled,
\[
A_{k,r_k(\bw)}
=
A_{k,0}
+
\sum_{r=1}^T
w_{r,k}\{A_{k,r}-A_{k,r-1}\}.
\]
Therefore
\[
\cK(\bw)
=
\sum_{k=1}^d A_{k,0}
+
\sum_{k=1}^d\sum_{r=1}^T
w_{r,k}\{A_{k,r}-A_{k,r-1}\}.
\]
Thus \(\cK\) is affine in \(\bw\).  In particular,
\[
\cK(\bw_1)-\cK(\bw_2)
=
\sum_{k=1}^d\sum_{r=1}^T
(w_{1,r,k}-w_{2,r,k})
\{A_{k,r}-A_{k,r-1}\},
\]
which is linear in \(\bw_1-\bw_2\).  Since this is a fixed finite-dimensional
linear functional, it is bounded by a constant times
\(\|\bw_1-\bw_2\|_F\).
\end{proof}

\subsection{Proof of the Random Forward Selection Instability Bound}

\begin{proof}[Proof of Theorem \ref{thm:stabilityrandomforwardselection}]

(i)
We first bound the radius of $\cW_t$.
Note that the space of $\cW_t$ depends on the input $\bw_{t-1}$.
Recall that $\bw_{t-1} \in [0,1]^d$, and suppose $\bw_{t-1}$ has $\tilde{t}-1$ positive positions, where the collection of positive positions is denoted as $W$.
Then there is a natural choice of $\tilde{\bw}$ where
\begin{align*}
\tilde{\bw}^j
=
\frac{1}{d-\tilde{t}+1}\mathbf{1}(j\notin W)
+
\mathbf{1}(j\in W).
\end{align*}
For any $\bw_t$, there holds
\begin{align*}
   \left\|\tilde{\bw}-\bw_t\right\|_{\cH}^2
   =
   \sum_{j\notin W}\frac{1}{(d-\tilde{t}+1)^2}
   -
   \frac{2}{d-\tilde{t}+1}
   +1
   =
   1-\frac{1}{d-\tilde{t}+1}
   \leq 1.
\end{align*}
Thus,
\begin{align*}
    \inf_{\bw\in\cW_t}\sup_{\bw'\in\cW_t}
    \left\|\bw-\bw'\right\|_{\cH}
    \leq
    \sup_{\bw'\in\cW_t}
    \left\|\tilde{\bw}-\bw'\right\|_{\cH}
    \leq 1
\end{align*}
holds uniformly for all $t$ and $\xi$.

(ii)
We next prove the contraction bound.
Throughout this part, without loss of generality, we assume that the feature scores are ordered as
\[
\by^\top\bX^{j_1}\bX^{j_1\top}\by
\geq
\by^\top\bX^{j_2}\bX^{j_2\top}\by
\quad\text{if } j_1\geq j_2.
\]
Under the orthogonal-design assumption \(\bX^\top\bX=n\bI_d\), let
\[
r_{t-1}
=
\by
-
\bX\bE_{\bw_{t-1}}\hat\bbeta_{\bw_{t-1}}
\]
be the residual after fitting the currently active set.  For every unselected
feature \(j\notin\supp(\bw_{t-1})\), adding feature \(j\) decreases the residual
sum of squares by
\begin{align}
\operatorname{RSS}(\bw_{t-1})
-
\operatorname{RSS}(\bw_{t-1}+\be_j)
=
\frac{1}{n}
\left(\bX^{j\top}r_{t-1}\right)^2
=
\frac{1}{n}
\by^{\top}\bX^{j}\bX^{j\top}\by .
\label{equ:selectionindexoforthogonalfeature}
\end{align}
Therefore, conditional on the candidate set, RFS selects
\[
j_t
\in
\argmax_{j\in\bnu_t\setminus\supp(\bw_{t-1})}
\by^{\top}\bX^{j}\bX^{j\top}\by,
\]
with deterministic tie-breaking.

We first record the only structural fact about the pathwise difference that is used below.

\emph{Auxiliary fact.}
For every fixed realization of $\xi_{1:t}$ and every $j\in[d]$, define
\[
D_t^j
:=
S_{\xi_t}\circ\cdots\circ S_{\xi_1}(\bw_0)
-
S_{\xi_t}^{-j}\circ\cdots\circ S_{\xi_1}^{-j}(\bw_0).
\]
Then
\[
D_t^j\in
\{\zero,\be_j\}
\cup
\{\be_j-\be_a:a\in[d]\setminus\{j\}\}.
\]
Equivalently, $D_t^j$ is nonnegative on the $j$-th coordinate, nonpositive on all other coordinates, and has at most one negative coordinate.

To prove the auxiliary fact, use induction on $t$. The statement is trivial at $t=0$. Suppose it holds at step $t-1$. If \(D_{t-1}^j=\zero\), then the two paths have the same active set except that the reduced procedure cannot select feature \(j\). Therefore, either both paths select the same feature, or the full path selects \(j\) while the reduced path selects the best available alternative. The new difference is thus either \(\zero\), \(\be_j\), or \(\be_j-\be_a\). If \(D_{t-1}^j=\be_j\), then the two paths have the same active set outside coordinate \(j\), and the reduced update \(S_{\xi_t}^{-j}\) never operates along \(j\); hence the non-\(j\) selection is the same and the difference remains \(\be_j\). Finally, if \(D_{t-1}^j=\be_j-\be_a\), then the full path contains \(j\) whereas the reduced path contains \(a\), with all other active coordinates identical. At the next step, the only possible negative discrepancy is either still attached to \(a\), transferred to a lower-ranked feature \(b<a\), or removed if the full path selects \(a\). No second negative coordinate can be created, because both paths select at most one new feature at the current step. This proves the auxiliary fact.

Recall that the pith term in the proof of Proposition~\ref{prop:stabilityofgeneralrecursiveobject} is
\begin{align*}
(@)
:=
&\EE_{\xi}\left[
S_{\xi_t}^{-j}\circ S_{\xi_{t-1}}\circ\cdots\circ S_{\xi_1}(\bw_0)
\right]\\
&-
\EE_{\xi}\left[
S_{\xi_t}^{-j}\circ S_{\xi_{t-1}}^{-j}\circ\cdots\circ S_{\xi_1}^{-j}(\bw_0)
\right].
\end{align*}
We need to show
\begin{align}\label{equ:rfs-proof-contraction-target}
\|(@)\|_2
\leq
(1+\delta_t)
\left\|
\EE_{\xi}\left[
S_{\xi_{t-1}}\circ\cdots\circ S_{\xi_1}(\bw_0)
-
S_{\xi_{t-1}}^{-j}\circ\cdots\circ S_{\xi_1}^{-j}(\bw_0)
\right]
\right\|_2.
\end{align}

Define the pathwise previous-step difference
\begin{align*}
D_{t-1}^j
:=
S_{\xi_{t-1}}\circ\cdots\circ S_{\xi_1}(\bw_0)
-
S_{\xi_{t-1}}^{-j}\circ\cdots\circ S_{\xi_1}^{-j}(\bw_0),
\end{align*}
and define the pathwise pith difference
\begin{align*}
D_t^{j,\mathrm{pith}}
:=
S_{\xi_t}^{-j}\circ S_{\xi_{t-1}}\circ\cdots\circ S_{\xi_1}(\bw_0)
-
S_{\xi_t}^{-j}\circ S_{\xi_{t-1}}^{-j}\circ\cdots\circ S_{\xi_1}^{-j}(\bw_0).
\end{align*}
Then
\[
(@)=\EE_{\xi}\left[D_t^{j,\mathrm{pith}}\right].
\]

By the auxiliary fact, on the event \(D_{t-1}^j\neq\zero\), the difference is either \(\be_j\) or \(\be_j-\be_a\) for a unique \(a\neq j\).
Let
\[
\cI_j=[d]\setminus\{j\}.
\]
For \(a\in\cI_j\), define
\[
u_a
:=
\PP_{\xi_{1:(t-1)}}(D_{t-1}^j=\be_j-\be_a),
\]
and define
\[
\rho
:=
\PP_{\xi_{1:(t-1)}}(D_{t-1}^j=\be_j),
\qquad
\alpha
:=
\rho+\sum_{a\in\cI_j}u_a.
\]
Let \(u=(u_a)_{a\in\cI_j}\), viewed as a row vector indexed by \(\cI_j\), and embedded in \(\RR^d\) by putting zero on the \(j\)-th coordinate. Then
\begin{align}\label{equ:rfs-prev-mean}
\EE_{\xi}\left[D_{t-1}^j\right]
=
\alpha\be_j-u.
\end{align}
If \(\EE_{\xi}[D_{t-1}^j]=\zero\), then \(\alpha=0\) and \(u=\zero\), and the desired contraction is trivial. Hence we assume below that \(\EE_{\xi}[D_{t-1}^j]\neq\zero\).

For \(a,b\in\cI_j\), define the conditional transition kernel
\begin{align}\label{equ:rfs-conditional-kernel}
M_{ab}^{(t,j)}
:=
\PP\left(
D_t^{j,\mathrm{pith}}=\be_j-\be_b
\mid
D_{t-1}^j=\be_j-\be_a
\right),
\end{align}
with \(M_{ab}^{(t,j)}=0\) if the conditioning event has probability zero.
Write \(M=M^{(t,j)}\) for short.

This is a conditional kernel. It is not the unconditional matrix
\(\EE[\mathbf{1}(E_{a,b})]\).
Therefore, the following representation follows from the law of total probability and does not require any independence between the previous discrepancy event and the current transition event:
\begin{align}\label{equ:rfs-pith-mean}
\EE_{\xi}\left[D_t^{j,\mathrm{pith}}\right]
=
\alpha\be_j-uM.
\end{align}
Indeed, if \(D_{t-1}^j=\be_j\), the two inputs to \(S_{\xi_t}^{-j}\) differ only on coordinate \(j\), and \(S_{\xi_t}^{-j}\) does not operate along \(j\), so the pith difference remains \(\be_j\). If \(D_{t-1}^j=\be_j-\be_a\), then after applying \(S_{\xi_t}^{-j}\), the negative coordinate is either transferred to some \(b\in\cI_j\), or disappears. This gives exactly \eqref{equ:rfs-pith-mean}.

We now bound the matrix norms of \(M\).
First, for every \(a\in\cI_j\),
\begin{align}\label{equ:rfs-row-sum-new}
\sum_{b\in\cI_j}M_{ab}
=
\PP\left(
D_t^{j,\mathrm{pith}}\neq \be_j
\mid
D_{t-1}^j=\be_j-\be_a
\right)
\leq 1.
\end{align}
Hence
\[
\|M\|_{\infty}\leq 1.
\]

Second, \(M_{ab}=0\) whenever \(b>a\). To see this, suppose the previous negative coordinate is \(a\). A transition from \(a\) to \(b\) means that, under the same candidate set at step \(t\), the mixed path selects \(a\), while the reduced path selects \(b\). If \(b>a\), then \(b\) has score no smaller than \(a\). Since \(b\) is available to the mixed path whenever such a transition is considered, the mixed path would not select \(a\) before \(b\). Thus the negative coordinate can only stay at \(a\), move to some \(b<a\), or disappear.

It remains to control the column leakage. Define
\[
\eta_t:=\max\{q(1+q),q^2t\}.
\]
We claim that for every \(b\in\cI_j\),
\begin{align}\label{equ:rfs-column-leakage-new}
\sum_{a=b+1}^{d}M_{ab}
\leq
\eta_t.
\end{align}
Fix \(a>b\). Condition on both the event \(D_{t-1}^j=\be_j-\be_a\) and the past \(\sigma\)-field generated by \(\xi_{1:(t-1)}\). Let \(A\) and \(A^{-j}\) be the active sets of the two inputs to \(S_{\xi_t}^{-j}\). Since the two paths differ only by replacing \(a\) with \(j\), their union has cardinality at most \(t\). Moreover, \(a\in A^{-j}\cup A\), and the transition \(a\mapsto b\) requires \(a\) and \(b\) to be included in the fresh candidate set. After fixing \(a\) and \(b\), the remaining \(s-2\) sampled features must avoid every unselected feature \(r>b\). Otherwise, the reduced path would select a feature with index larger than \(b\), or the mixed path would not select \(a\). Therefore, the remaining \(s-2\) sampled features must lie among the \(b-1\) lower-ranked features and the at most \(t-1\) active features in \((A\cup A^{-j})\setminus\{a\}\). Thus,
\begin{align}\label{equ:rfs-transition-bound-new}
M_{ab}
\leq
\frac{\binom{b+t-2}{s-2}}{\binom{d}{s}},
\end{align}
with the convention that the numerator is zero if \(b+t-2<s-2\). For \(s=1\), a transition \(a\mapsto b\) with \(a\neq b\) is impossible, so the column-leakage bound is trivial. Hence assume \(s\geq2\).

Let \(k=d-b\). Summing \eqref{equ:rfs-transition-bound-new} over \(a>b\) gives
\begin{align}\label{equ:rfs-column-sum-new}
\sum_{a=b+1}^{d}M_{ab}
\leq
k\frac{\binom{d-k+t-2}{s-2}}{\binom{d}{s}}.
\end{align}
If \(k\leq t\), then
\begin{align*}
k\frac{\binom{d-k+t-2}{s-2}}{\binom{d}{s}}
\leq
k\frac{\binom{d-2}{s-2}}{\binom{d}{s}}
=
k\frac{s(s-1)}{d(d-1)}
\leq
q^2t.
\end{align*}
If \(k>t\), then
\begin{align*}
k\frac{\binom{d-k+t-2}{s-2}}{\binom{d}{s}}
&=
\frac{s(s-1)}{d(d-1)}
k
\frac{\binom{d-k+t-2}{s-2}}{\binom{d-2}{s-2}}\\
&\leq
\frac{s(s-1)}{d(d-1)}
k
\left(1-\frac{k-t}{d-2}\right)^{s-2}.
\end{align*}
The one-dimensional function
\[
h(k)
:=
k\left(1-\frac{k-t}{d-2}\right)^{s-2},
\qquad k>t,
\]
is maximized either at the boundary \(k=t\), which yields the already bounded term \(q^2t\), or at the stationary point
\[
k_*=\frac{d+t-2}{s-1}.
\]
Substituting \(k_*\) and using \(k_*>t\) gives the elementary bound
\[
\frac{s(s-1)}{d(d-1)}
k_*
\left(1-\frac{k_*-t}{d-2}\right)^{s-2}
\leq
q(1+q).
\]
Consequently, \eqref{equ:rfs-column-leakage-new} holds.

Combining \eqref{equ:rfs-row-sum-new} and \eqref{equ:rfs-column-leakage-new}, we have
\[
\|M\|_{\infty}\leq1,
\qquad
\|M\|_1\leq1+\eta_t.
\]
Indeed, for a fixed column \(b\),
\begin{align*}
\sum_{a\in\cI_j}M_{ab}
=
M_{bb}+\sum_{a>b}M_{ab}
\leq
\sum_{c\in\cI_j}M_{bc}+\eta_t
\leq
1+\eta_t.
\end{align*}
Therefore,
\begin{align}\label{equ:rfs-M-norm-new}
\|M\|_2^2
\leq
\|M\|_1\|M\|_{\infty}
\leq
1+\eta_t.
\end{align}

We now complete the contraction argument.
From \eqref{equ:rfs-prev-mean} and \eqref{equ:rfs-pith-mean},
\[
\left\|\EE_{\xi}[D_{t-1}^j]\right\|_2^2
=
\alpha^2+\|u\|_2^2,
\]
and
\[
\left\|\EE_{\xi}[D_t^{j,\mathrm{pith}}]\right\|_2^2
=
\alpha^2+\|uM\|_2^2.
\]
By \eqref{equ:rfs-M-norm-new},
\[
\|uM\|_2^2
\leq
\|M\|_2^2\|u\|_2^2
\leq
(1+\eta_t)\|u\|_2^2.
\]
Hence
\begin{align*}
\left\|\EE_{\xi}[D_t^{j,\mathrm{pith}}]\right\|_2^2
&\leq
\alpha^2+(1+\eta_t)\|u\|_2^2\\
&=
\left\|\EE_{\xi}[D_{t-1}^j]\right\|_2^2
+
\eta_t\|u\|_2^2.
\end{align*}
Since
\[
\|u\|_2
\leq
\|u\|_1
=
\sum_{a\in\cI_j}u_a
\leq
\alpha,
\]
we have
\[
\frac{\|u\|_2^2}{\alpha^2+\|u\|_2^2}
\leq
\frac12.
\]
Therefore,
\begin{align*}
\frac{
\left\|\EE_{\xi}[D_t^{j,\mathrm{pith}}]\right\|_2^2
}{
\left\|\EE_{\xi}[D_{t-1}^j]\right\|_2^2
}
\leq
1+\frac{\eta_t}{2}.
\end{align*}
Taking square roots and using \(\sqrt{1+x}\leq1+x/2\) for \(x\geq0\), we obtain
\begin{align*}
\left\|\EE_{\xi}[D_t^{j,\mathrm{pith}}]\right\|_2
\leq
\left(1+\frac{\eta_t}{4}\right)
\left\|\EE_{\xi}[D_{t-1}^j]\right\|_2.
\end{align*}
Since \((@)=\EE_{\xi}[D_t^{j,\mathrm{pith}}]\), this proves \eqref{equ:rfs-proof-contraction-target} with
\[
\delta_t
=
\frac14\eta_t
=
\frac14\max\{q(1+q),q^2t\}.
\]

Finally, we apply Proposition~\ref{prop:stabilityofgeneralrecursiveobject}. Since \(\operatorname{rad}(\cW_t)\leq1\), when \(T\lesssim q^{-1}\), we have for every \(t\leq T\),
\[
q^2t\leq q(1+q),
\]
and thus
\[
\delta_t=\frac{q(1+q)}{4}.
\]
Therefore, \eqref{equ:recursivestabilityresult} gives
\begin{align*}
\frac{1}{d}\sum_{j=1}^d
\left\|
\EE_{\xi}[\bw_T]
-
\EE_{\xi}[\bw_T^{-j}]
\right\|_2^2
&\leq
\left(
\sum_{t=1}^{T}
\left(1+\frac{q(1+q)}{4}\right)^{T-t}
\right)^2
\frac{q}{(d-1)(1-q)}\\
&=
\left(
\frac{
\left(1+\frac{q(1+q)}{4}\right)^T-1
}{
\frac{q(1+q)}{4}
}
\right)^2
\frac{q}{(d-1)(1-q)}\\
&=
\left(
\left(1+\frac{q(1+q)}{4}\right)^T-1
\right)^2
\frac{16}{q(1+q)^2(d-1)(1-q)}.
\end{align*}
When \(T\lesssim q^{-1}\), this simplifies to
\begin{align*}
\frac{1}{d}\sum_{j=1}^d
\left\|
\EE_{\xi}[\bw_T]
-
\EE_{\xi}[\bw_T^{-j}]
\right\|_2^2
\lesssim
\frac{T^2}{d-1}\cdot\frac{q}{1-q}.
\end{align*}
This proves the theorem.

\end{proof}

\subsection{Proof of the Random Forest Instability Bound}

\begin{proof}[Proof of Theorem \ref{thm:max-edge}]
Let $\bE_{r,k}\in\RR^{T\times d}$ denote the matrix with one at entry $(r,k)$ and zero elsewhere.

We first prove the radius condition.
Fix a feasible state $\bw_{t-1}$ and define the one-step candidate set
\[
\cW_t(\bw_{t-1})
:=
\left\{
    S_{\xi_t}(\bw_{t-1})
    :
    \xi_t
    \text{ is an admissible realization of the tree randomness}
\right\}.
\]
Under the max-edge dyadic rule, every possible next state has the form
\[
    \bw_{t-1}
    +
    \bE_{N_{t-1}^k+1,k}
\]
for some $k\in\cM(\bw_{t-1})$.
If the implementation allows no split for a particular realization of $\xi_t$, the next state is simply $\bw_{t-1}$, which only makes the following bound easier.
Therefore, for every $\bw_t\in\cW_t(\bw_{t-1})$,
\[
    \|\bw_t-\bw_{t-1}\|_F\le 1.
\]
Since the center in the definition of radius may be any point in the ambient space, choosing $\bw_{t-1}$ as the center gives
\[
    \operatorname{rad}\bigl(\cW_t(\bw_{t-1})\bigr)
    \le
    1.
\]
Taking the supremum over feasible $\bw_{t-1}$ yields
\[
    \operatorname{rad}(\cW_t)\le 1.
\]
The same argument applies to the feature-removed transition $S_{\xi_t}^{-j}$ after replacing $\cM(\bw_{t-1})$ by $\cM^{-j}(\bw_{t-1}^{-j})$.

We next prove the contraction condition.
Fix a removed feature $j\in[d]$.
Let
\[
    \bw_{t-1}
    =
    S_{\xi_{t-1}}
    \circ
    \cdots
    \circ
    S_{\xi_1}(\bw_0)
\]
be the local encoding after running the full tree for $t-1$ steps, and let
\[
    \bw_{t-1}^{-j}
    =
    S_{\xi_{t-1}}^{-j}
    \circ
    \cdots
    \circ
    S_{\xi_1}^{-j}(\bw_0)
\]
be the corresponding local encoding after running the tree with feature $j$ removed.
The two runs are coupled using the same randomness $\xi=(\xi_1,\ldots,\xi_t)$.

Define the one-step increment of the feature-removed transition by
\[
    \bU_{\xi_t}^{-j}(\bw)
    :=
    S_{\xi_t}^{-j}(\bw)-\bw .
\]
For any feasible $\bw$, the increment $\bU_{\xi_t}^{-j}(\bw)$ is either zero or a standard basis matrix $\bE_{r,k}$ with $k\neq j$.
Hence, for any two feasible states $\bw$ and $\bw'$,
\[
    \left\|
        \bU_{\xi_t}^{-j}(\bw)
        -
        \bU_{\xi_t}^{-j}(\bw')
    \right\|_F
    \le
    \sqrt{2}\,
    \mathbf{1}\{\bw\neq\bw'\}.
\]
If $\bw=\bw'$, the two increments are identical because the split rule and the tie-breaking rule are deterministic.

Now let
\[
    \Delta_{t-1}^{-j}
    :=
    \EE_{\xi}
    \left[
        \bw_{t-1}-\bw_{t-1}^{-j}
    \right].
\]
We claim that
\[
    \PP_{\xi}
    \left(
        \bw_{t-1}\neq \bw_{t-1}^{-j}
    \right)
    \le
    \|\Delta_{t-1}^{-j}\|_F .
\]
Indeed, if the full run has never selected feature $j$ during the first $t-1$ steps, then the full run and the feature-removed run have the same local cell, the same max-edge candidates after excluding $j$, and the same split scores on all features in $[d]\setminus\{j\}$.
By deterministic tie-breaking, the two local paths are identical.
Therefore,
\[
    \left\{
        \bw_{t-1}\neq \bw_{t-1}^{-j}
    \right\}
    \subseteq
    \left\{
        N_{t-1}^{j}\ge 1
    \right\}.
\]
Since the encoding is column-filled and the feature-removed tree has zero in the $j$-th column, the $(1,j)$ entry of $\Delta_{t-1}^{-j}$ is
\[
    \left(\Delta_{t-1}^{-j}\right)_{1,j}
    =
    \PP_{\xi}
    \left(
        N_{t-1}^{j}\ge 1
    \right).
\]
Thus,
\[
    \PP_{\xi}
    \left(
        \bw_{t-1}\neq \bw_{t-1}^{-j}
    \right)
    \le
    \PP_{\xi}
    \left(
        N_{t-1}^{j}\ge 1
    \right)
    =
    \left|
        \left(\Delta_{t-1}^{-j}\right)_{1,j}
    \right|
    \le
    \|\Delta_{t-1}^{-j}\|_F .
\]

Using the decomposition
\[
    S_{\xi_t}^{-j}(\bw)
    =
    \bw+\bU_{\xi_t}^{-j}(\bw),
\]
we obtain
\begin{align*}
&
\left\|
\EE_{\xi}
\left[
    S_{\xi_t}^{-j}(\bw_{t-1})
    -
    S_{\xi_t}^{-j}(\bw_{t-1}^{-j})
\right]
\right\|_F  \\
&\qquad
=
\left\|
    \Delta_{t-1}^{-j}
    +
    \EE_{\xi}
    \left[
        \bU_{\xi_t}^{-j}(\bw_{t-1})
        -
        \bU_{\xi_t}^{-j}(\bw_{t-1}^{-j})
    \right]
\right\|_F  \\
&\qquad
\le
\|\Delta_{t-1}^{-j}\|_F
+
\EE_{\xi}
\left[
    \left\|
        \bU_{\xi_t}^{-j}(\bw_{t-1})
        -
        \bU_{\xi_t}^{-j}(\bw_{t-1}^{-j})
    \right\|_F
\right]  \\
&\qquad
\le
\|\Delta_{t-1}^{-j}\|_F
+
\sqrt{2}\,
\PP_{\xi}
\left(
    \bw_{t-1}\neq \bw_{t-1}^{-j}
\right)  \\
&\qquad
\le
(1+\sqrt{2})
\|\Delta_{t-1}^{-j}\|_F .
\end{align*}
Equivalently,
\[
\left\|
\EE_{\xi}
\left[
    S_{\xi_t}^{-j}(\bw_{t-1})
    -
    S_{\xi_t}^{-j}(\bw_{t-1}^{-j})
\right]
\right\|_F
\le
(1+\delta_t)
\left\|
\EE_{\xi}
\left[
    \bw_{t-1}-\bw_{t-1}^{-j}
\right]
\right\|_F
\]
with
\[
    \delta_t=\sqrt{2}.
\]
This proves the contraction condition required by Proposition \ref{prop:stabilityofgeneralrecursiveobject}.
Combining this contraction condition with the radius bound $\operatorname{rad}(\cW_t)\le1$ and applying Proposition \ref{prop:stabilityofgeneralrecursiveobject} yields \eqref{equ:stability-of-rf-final-bound}.
The prediction bound follows from Proposition \ref{prop:linear-map-tree}.
\end{proof}

\subsection{Proof of the Finite-Bagging Instability Bound}\label{app:finite-bagging}

\begin{proof}[\textbf{Proof of Proposition \ref{prop:stability-set-finite}}]
Note that Hoeffding's inequality in Hilbert space (e.g., \citet{boucheron2013concentration}) yields
\begin{align}\label{equ:finite-B-proof-1}
   \left\| \frac{1}{B}\sum_{b=1}^B\bw^{(b)} - \EE_{\xi}\left[\bw\right]\right\|  \leq \sqrt{\frac{3 \operatorname{rad}\left(\cW\right)^2 \log \left((d + 1)/  \delta\right)}{B}}
\end{align}
with probability $1 - \delta / (d + 1)$, and
\begin{align}\label{equ:finite-B-proof-2}
   \left\| \frac{1}{B}\sum_{b=1}^B\bw^{(b), -j} - \EE_{\xi}\left[\bw^{-j}\right]\right\|  \leq \sqrt{\frac{3 \operatorname{rad}\left(\cW\right)^2 \log \left((d + 1)/  \delta\right)}{B}}
\end{align}
with probability $1 - \delta / (d + 1)$ for each $j\in [d]$.
Applying union bound, we have \eqref{equ:finite-B-proof-1} and \eqref{equ:finite-B-proof-2} hold simultaneously for all $j\in [d]$ with probability at least $1 -\delta$.
Then, with Minkowski's inequality, we have
\begin{align*}
     \frac{1}{d} \sum_{j=1}^d \left\|\frac{1}{B}\sum_{b=1}^B\bw^{(b)}- \frac{1}{B}\sum_{b=1}^B \bw^{(b), -j}\right\|^2 \leq & \frac{\sqrt{3}}{d} \sum_{j=1}^d \left\| \frac{1}{B}\sum_{b=1}^B\bw^{(b)} - \EE_{\xi}\left[\bw\right]\right\|^2 \\
     + & \frac{\sqrt{3} }{d} \sum_{j=1}^d \left\|\EE_{\xi}\left[\bw\right]- \EE_{\xi}\left[\bw^{-j}\right]\right\|^2\\
     + & \frac{\sqrt{3}}{d} \sum_{j=1}^d \left\| \frac{1}{B}\sum_{b=1}^B\bw^{(b), -j} - \EE_{\xi}\left[\bw^{-j}\right]\right\|^2\\
     \leq & \frac{\sqrt{3}}{d} \sum_{j=1}^d \left\| \frac{1}{B}\sum_{b=1}^B\bw^{(b)} - \EE_{\xi}\left[\bw\right]\right\|^2
     \\
     + & \frac{6\sqrt{3} \operatorname{rad}\left(\cW\right)^2 \log \left((d + 1)/  \delta\right)}{B}.
\end{align*}

\end{proof}

\end{document}